\documentclass{article}
\usepackage[utf8]{inputenc}
\usepackage{amsmath}
\usepackage{amsthm}
\usepackage{amssymb}
\usepackage{mathtools}
\usepackage{natbib}
\renewcommand{\cite}[1]{\citep{#1}}
\setcitestyle{authoryear,round,citesep={;},aysep={,},yysep={;}}

\usepackage{thmtools}
\usepackage{thm-restate}

\usepackage[table]{xcolor}
\definecolor{Gray}{gray}{0.9}
\definecolor{midgreen}{rgb}{0.1,0.5,0.1}
\definecolor{darkgray}{gray}{0.25}
\definecolor{lightblue}{rgb}{0.25,0.25,0.8}
\definecolor{mydarkblue}{rgb}{0,0.08,0.45}
\usepackage[colorlinks,
            linkcolor=mydarkblue,
            anchorcolor=blue,
            urlcolor=mydarkblue,
            citecolor=mydarkblue]{hyperref}
\usepackage{graphicx}
\usepackage{subcaption}
\usepackage{enumerate}
\usepackage{fullpage}
\usepackage{tikz}
\usetikzlibrary{angles,arrows,quotes,decorations.pathreplacing}
\usepackage{thmtools}
\usepackage{thm-restate}
\usepackage[capitalize,nameinlink,noabbrev]{cleverref}
\usepackage{apptools}
\usepackage{bm}

\usepackage{booktabs}       % professional-quality tables
\usepackage{algorithm}
\usepackage{algorithmic}
\usepackage{multicol}
\usepackage{multirow}
\newtheorem{theorem}{Theorem}
\newtheorem{lemma}{Lemma}

\theoremstyle{definition}

\newtheorem{defn}{Definition}

\newcommand{\NTK}{\mathrm{NTK}}
\newcommand{\NNGP}{\mathrm{NNGP}}
\newcommand{\CS}{{\sc CountSketch} }
\newcommand{\TS}{{\sc TensorSketch} }

\newcommand{\CIFARTEN}{$\mathtt{CIFAR}$-10}
\newcommand{\CIFARHUN}{$\mathtt{CIFAR}$-100}
\newcommand{\VOC}{$\mathtt{VOC07}$}
\newcommand{\CUB}{$\mathtt{CUB}$-200}
\newcommand{\CALTECH}{$\mathtt{Caltech}$-101}

\newcommand{\FLOWER}{$\mathtt{Flower}$-102}
\newcommand{\FOOD}{$\mathtt{Food}$-101}
\newcommand{\DOG}{$\mathtt{Dog}$-120}
\newcommand{\MSD}{$\mathtt{MillionSongs}$}

\newcommand{\Protein}{$\mathtt{Protein}$}
\newcommand{\SuperConduct}{$\mathtt{SuperConduct}$}
\newcommand{\WorkLoads}{$\mathtt{WorkLoads}$}
\newcommand{\norm}[1]{\ensuremath{\left\| #1 \right\|}}

\newcommand{\inner}[1]{\left \langle {#1} \right \rangle}
\newcommand{\bigo}{\mathcal{O}}
\newcommand{\abs}[1]{\left |#1\right|}
\newcommand{\FFT}{\mathrm{FFT}}
\newcommand{\iFFT}{\mathrm{FFT}^{-1}}

\def\tr{\mathtt{tr}}

\def\0{{\bm 0}}

\def\c{{\bm c}}

\def\s{{\bm s}}

\def\v{{\bm v}}
\def\w{{\bm w}}
\def\x{{\bm x}}
\def\y{{\bm y}}
\def\z{{\bm z}}
\def\A{{\bm A}}
\def\B{{\bm B}}
\def\C{{\bm C}}

\def\I{{\bm I}}

\def\K{{\bm K}}

\def\S{{\bm S}}

\def\W{{\bm W}}
\def\X{{\bm X}}
\def\Y{{\bm Y}}
\def\Z{{\bm Z}}

\def\BSigma{\boldsymbol{\Sigma}}
\def\BLambda{\boldsymbol{\Lambda}}
\def\BGamma{\boldsymbol{\Gamma}}
\def\BPhi{\boldsymbol{\Phi}}
\def\BPsi{\boldsymbol{\Psi}}

\def\Btheta{\boldsymbol{\theta}}

\def\E{{\mathbb{E}}}
\def\R{{\mathbb{R}}}

\title{Random Features for the Neural Tangent Kernel}

\newcommand{\email}[1]{Email: \href{mailto:#1}{\color{black} \texttt{#1}}}
\def\kaist{Korea Advanced Institute of Science and Technology}
\def\tauniv{Tel Aviv University}

\author{%
    Insu Han~\thanks{School of Electrical Engineering, \kaist. \email{insu.han@kaist.ac.kr}} \and
    Haim Avron~\thanks{Department of Applied Mathematics, \tauniv. \email{haimav@tauex.tau.ac.il}} \and
    Neta Shoham~\thanks{Department of Applied Mathematics, \tauniv. \email{shohamne@gmail.com}} \and
    Chaewon Kim~\thanks{Graduate School of AI, \kaist. \email{chaewonk@kaist.ac.kr}} \and
    Jinwoo Shin~\thanks{Graduate School of AI, \kaist. \email{jinwoos@kaist.ac.kr}}
}

\begin{document}

\maketitle

\begin{abstract}
 The Neural Tangent Kernel (NTK) has discovered connections between deep neural networks and kernel methods with insights of optimization and generalization. Motivated by this, recent works report that NTK can achieve better performances compared to training neural networks on small-scale datasets. However, results under large-scale settings are hardly studied due to the computational limitation of kernel methods. In this work, we propose an efficient feature map construction of the NTK of fully-connected ReLU network which enables us to apply it to large-scale datasets. We combine random features of the arc-cosine kernels with a sketching-based algorithm which can run in linear with respect to both the number of data points and input dimension. We show that dimension of the resulting features is much smaller than other baseline feature map constructions to achieve comparable error bounds both in theory and practice. We additionally utilize the leverage score based sampling for improved bounds of arc-cosine random features and prove a spectral approximation guarantee of the proposed feature map to the NTK matrix of two-layer neural network. We benchmark a variety of machine learning tasks to demonstrate the superiority of the proposed scheme. In particular, our algorithm can run tens of magnitude faster than the exact kernel methods for large-scale settings without performance loss.
\end{abstract}

\section{Introduction} \label{sec:intro}
Recent literature have shown that trained overparameterized Deep Neural Networks (DNNs), i.e., neural networks with substantially more parameters than training data points,  generalize surprisingly well. 
In an effort to understand this phenomena, recently researchers have studied the infinite width limit of DNNs (i.e., the number of neurons in each hidden layer goes to infinity) and has shown that in that limit, deep learning is equivalent to kernel regression where the kernel is the so-called Neural Tangent Kernel (NTK) of the network~\cite{arora2019exact, chizat2019lazy, jacot2018neural, lee2020generalized}. This connection has been used to shed light on the ability of DNNs to generalize~\cite{cao2019generalization,neyshabur2014search} and the ability to optimize (train) their parameters efficiently~\cite{allen2019convergence,arora2019fine,du2018gradient}.

Beyond the aforementioned theoretical purpose, several papers have explored the algorithmic use of the NTK. ~\citet{arora2019harnessing} and ~\citet{geifman2020similarity} showed that NTK based kernel models can sometimes perform better than trained neural networks. The NTK has also been used in experimental design for neural networks~\cite{shoham2020experimental} and predicting training time~\cite{zancato2020predicting}.

% A kernel approach so-called the neural tangent kernel~\cite{jacot2018neural}  (NTK) have been used to characterize such network linearization where it was originally proposed to analyze training dynamics of infinite-width neural networks. Observing this, \citet{chizat2019lazy} proposed the lazy training where parameter updates are done in neighborhood of the initial point.

Although the NTK of a given network can sometimes be expressed as closed-form formulas~\cite{arora2019exact, novak2020neural}, actually using it to learn kernel models encounters the computational bottlenecks of kernel learning, e.g., $\bigo(n^3)$ time and $\bigo(n^2)$ space complexity for kernel ridge regression with $n$ data points. With the NTK, the situation is much worse since the cost required to compute the kernel matrix can be huge~\cite{novak2020neural}. This makes exact kernel learning with NTK infeasible under large-scale setups.  

%The cost of computing the kernel matrix of a kernel $k$ on $n$ data points is $O(T_k n^2)$ where $T_k$ is the cost of evaluating the kernel function on a single pair of data points. For standard kernels such a the Gaussian kernel, $T_k$ is proportional to the data dimension, which is typically much smaller than $n$. For the NTK the constant $T_k$ can be very large.

%expression in of specific types of DNNs can be expressed as closed-form formulas~\cite{arora2019exact, novak2020neural}, there still exist computational limitations from kernel methods, i.e., they generally require $\bigo(n^3)$ time and $\bigo(n^2)$ space complexity for $n$ data points. These restrictions indeed make them impractical for large-scale settings as pointed out in~\cite{novak2020neural, arora2019exact, shankar2020neural, arora2019harnessing}.
% Several NTK researches indeed raised issues on reducing the complexity of kernel evaluations~\cite{shankar2020neural}. 

There is a rich literature on using kernel approximations in order to enable large-scale learning. One of the most popular approaches is the {\em random features} approach, originally due to~\citet{rahimi2007random}. Following the seminal work, % of~\citet{rahimi2007random}, 
many random feature constructions have been suggested for a variety of kernels, e.g., arc-cosine kernels~\cite{cho2009kernel}, polynomial kernels~\cite{pham2013fast,pennington2015spherical}, general dot product kernels~\cite{han2020polynomial}, just to name a few.  These low-dimensional features enable us to apply them to fast linear methods, saving time and space complexity drastically.
% For example, kernel ridge regression can be solved in $\bigo(nm^2)$ time using $m$-dimensional random features, which runs much faster for $m \ll n$.
% by utilizing an useful linear algebraic property so-called kernel trick. 
Furthermore, their performances are similar or sometimes better than the exact kernel methods due to implicit regularization effect~\cite{rahimi2007random,rudi2016generalization,jacot2020implicit}.
% is but also improve the performance  from randomness

In this paper, we propose an efficient random features construction for the NTK of fully-connected neural networks with ReLU activations. Our starting point is the explicit feature map for the NTK suggested by \citet{bietti2019inductive}. That feature map uses known explicit feature maps for the arc-cosine kernel. By replacing the explicit feature map of the arc-cosine kernel with a random feature map for the same kernel~\cite{cho2009kernel}, 
%in a recursive manner throughout the layers, 
we obtain a random feature map for the NTK. 
% However, the size of that feature map for that construction will be exponential in the number of layers.
However, the size of that feature map for that construction can be even larger than the number of input data points $n$.
The underlying reason is from the tensor products between features generated in consecutive layers. 
% The underlying reason for the exponential growth is the tensor products between features generated in consecutive layers. 
To avoid the issue, we utilize an efficient sketching algorithm known as \TS transform~\cite{pham2013fast,ahle2020oblivious,woodruff2020near} which can effectively approximate the tensor products of vectors while preserving their inner products. 
We provide a rigorous error analysis of the proposed scheme. 
The resulting random features have smaller dimension than the previous NTK feature map constructions. 
% The time complexity of our proposed sketching-based algorithm is linear in the input sparsity.
% the network depth. 

Furthermore, in order to approximate the NTK with less features, we improve the underlying existing random feature map of the arc-cosine kernel.  Our construction is based on a modified leverage score based sampling. Recent literature has shown that random features that use leverage score sampling entertain better convergence bounds~\cite{avron2017random, lee2020generalized}. However, computing the exact leverage scores requires the inversion of a $n$-by-$n$ matrix which is equivalent to the cost for solving the kernel method exactly. 
Luckily, \citet{avron2017random, lee2020generalized} showed that sampling from the upper bound of leverage score is enough to provide tight error bounds. Motivated by these results, we propose simple and closed-form upper bounds of leverage scores regarding to arc-cosine kernels. For further efficiency, we make use of Gibbs sampling to generate random features from the proposed modified distribution.

To theoretically justify our construction, we provide a spectral approximation guarantee for the proposed random features for two-layer neural network.
% when the network depth is two. 
Recent literature has advocated the use of such spectral bounds as a more general metric to measure kernel approximation quality as it pertains to many downstream tasks~\cite{avron2017random}. 
%We also discuss the challenges to extend our results to deeper networks. rec
% Under mild assumptions, the proposed features with nearly constant dimension suffices to spectrally approximate NTK.

Finally, we empirically benchmark the proposed random feature methods under machine learning tasks including classification/regression under UCI datasets, and active learning under MNIST datasets. We demonstrate that our random features method can perform similar to or better than the kernel method with NTK. We further show that the random features approach can run up to 17 times much faster with tested large-scale datasets, without loss on performance. 
%Our algorithm can also be expanded to generate optimal experiment design (OED) with improved running time.
%We also verify that random features without leverage score perform similar to that with the leverage score sampling.

%To the best of our knowledge, none of works have analyzed a finite-dimensional feature map of NTK. In this work, we propose a novel algorithm to approximate feature embeddings of NTK for given budget of the  dimension. In particular, we utilize a recursive relation of NTK for deep neural network and adopt Tensor sketch methods~\cite{pham2013fast,ahle2020oblivious,woodruff2020near} to prevent the feature dimension from  exponentially increasing. 

% \subsection{Related Works}
{\bf Related works. }
% Several research studied the NTK derivation of various deep networks including convolutional neural nets~\cite{arora2019exact}, ResNet~\cite{huang2020deep}, transformer~\cite{xx} and recurrent models~\cite{}.
% Recently, it have been reported that NTK is similar to Laplacian kernels when the inputs are restricted in sphere~\cite{geifman2020similarity}.
Many literature have studied a variety of NTK properties including optimization~\cite{allen2019convergence, du2018gradient}, generalization~\cite{cao2019generalization}, loss surface~\cite{mei2018mean} and so on. 
Recent works have more focused on NTK kernel itself.
\citet{geifman2020similarity, chen2020deep} have discovered that NTK are similar to the Laplace kernel in term of spectral information when data points are in hypersphere. 
\cite{fan2020spectra} studied eigenvalue distributions of NNGP and NTK and showed that they converge to deterministic distribution.
However, to the best of our knowledge, not many works focus on approximating the NTK.
\citet{arora2019exact} studied that gradient of randomly initialized network with finite widths can approximate the NTK in theory. However, they report that practical performances of random gradients are worse than that of the exact NTK by a large margin. 
Another line of work on NTK approximation is an explicit feature map construction via tensor product proposed by \citet{bietti2019inductive}. These explicit features can have infinite dimension in general. Hence, it is impossible to use their features in practice.
Even though one can use a finite-dimensional feature map, the computational gain of random features can be lost due to expensive tensor product operations.

\section{Preliminaries} \label{sec:preliminaries}
%%%%%%%%%%%%%%%%%%%%%%%%%%%%%%%%%%%%%%%%%%%%%%%%%%%%%%%%%%%%%%%%%%%%%%%%%%%%%%%%%%%%%%%%%%%%%%%%
%%%%%%%%%%%%%%%%%%%%%%%%%%%%%%%%%
%%%%%%%%%%%%%%%%%%%%%%%%%%%%%%%%%%%%%%%%%%%%%%%%%%%%%%%%%%%%%%%
\paragraph{Notations.} 
We use $[n]:=\{1,\dots,n\}$.  We denote $\otimes$ by the tensor (a.k.a. Kronecker) product and $\odot$ by the element-wise (a.k.a. Hadamard) product of two matrices. For square matrices $\A$ and $\B$, we write $\A \preceq \B$ if $\B - \A$ is positive semi-definite. We write $[\A]_{i,j}$ as an entry of $\A$ in $i$-th row and $j$-th column. Similarly, $[\v]_i$ is used for an $i$-th entry of vector $\v$.
% We also say $\A$ is an $(\varepsilon, \lambda)$-spectral approximation of $\B$ for $\varepsilon, \lambda > 0$ if 
% % Given $\varepsilon, \lambda > 0$ and a square matrix $\M \in \R^{n\times n}$, we say $Z\in \R^{n \times d}$ to be an 
% \begin{align*}
% \left(1-\varepsilon\right) \left( \B + \lambda \I\right)
% \preceq
% \A + \lambda \I
% \preceq
% \left(1+\varepsilon\right) \left( \B + \lambda \I\right).
% \end{align*}
We also denote $\mathrm{ReLU}(x) = \max(x,0)$ and consider this element-wise operation when the input is a matrix.
% \begin{defn}[Statistical dimension] 
% \label{def:sdim}
Given a positive semidefinite matrix $\K$ and $\lambda >0$, the statistical dimension of $\K$ with $\lambda$ is defined as $s_{\lambda}(\K) := \mathtt{tr}(\K (\K + \lambda \I)^{-1})$.
% \begin{align}\label{def:sdim}
    % s_{\lambda}(\K) := \mathtt{tr}(\K (\K + \lambda \I)^{-1}).
% \end{align}
% \end{defn}

%%%%%%%%%%%%%%%%%%%%%%%%%%%%%%%%%%%%%%%%%%%%%%%%%%%%%%%%%%%%%%%%%%%%%%%%%%%%%%%%%%%%%%%%%%%%%%%%
%%%%%%%%%%%%%%%%%%%%%%%%%%%%%%%%%%%%%%%%%%%%%%%%%%%%%%%%%%%%%%%%%%%%%%%%%%%%%%%%%%%%%%%%%%%%%%%%

\subsection{NTK of Fully-connected Deep Neural Networks}

Given an input $\x \in \R^d$,
% $\X = [\x_1, \dots, \x_n]^\top \in \R^{d \times n}$, 
consider a fully-connected ReLU network with input dimension $d$, 
hidden layer dimensions $d_1, \dots, d_L$ as
% $f$ maps to some scalar values as for $\ell \in [L]$
\begin{align} \label{eq:nn-definition}
f(\x, \Btheta) = {\boldsymbol h}_L^\top \w, \quad {\boldsymbol h}_0 = \x, \quad {\boldsymbol h}_\ell = \sqrt{\frac{2}{d_\ell}}~\mathrm{ReLU}\left( {\boldsymbol h}_{\ell-1}^\top \W_{\ell} \right), %\quad \text{for} \quad \ell=1,\dots, L
%f_{\theta}(\x) = \w^\top \frac1{\sqrt{d_L}} \sigma\left( W_L \frac1{\sqrt{d_{L-1}}} \sigma\left(
%\dots \frac1{\sqrt{d_2}}
%\sigma\left(W_2 \frac1{\sqrt{d_1}} \sigma(W_1 \x)
%\right)
%\right)
%\right)
\end{align}
where $\Btheta:=(\W_1,\dots,\W_L, \w)$ for $\W_\ell \in \R^{d_{\ell-1} \times d_{\ell}}$, $\w \in \R^{d_L}$, $\ell \in [L]$ represents the trainable parameters and $d_0=d$.
% \footnote{We conventionally denote $d_0 = d$.} 
The {\it Neural Tangent Kernel} (NTK) is defined as 
\begin{align} \label{def:ntk}
K_{\NTK}^{(L)}(\x, \x^\prime) = \E_{\Btheta}\left[ \inner{\frac{\partial f(\x,\Btheta)}{\partial \Btheta}, \frac{\partial f(\x^\prime, \Btheta)}{\partial \Btheta}}
\right]
\end{align}
where $\Btheta$ is from the standard Gaussian distribution. Given $n$ data points $\X = [\x_1, \dots, \x_n]^\top \in \R^{n \times d}$, we will write $f(\X, \Btheta):=[f(\x_1,\Btheta), \dots, f(\x_n, \Btheta)]^\top \in \R^n$ and the NTK matrix as $\K_{\NTK}^{(L)}\in \R^{n \times n}$ whose $(i,j)$-th entry is $K_{\NTK}^{(L)}(\x_i, \x_j)$ for $i,j \in [n]$.

The motivation for the definition of the NTK is as follows. Consider learning the network parameters $\Btheta$ by minimizing the squared loss $\frac{1}{2} \norm{f(\X, \Btheta) - \y}_2^2$ for the target $\y \in \R^n$ using gradient descent with infinitesimally
small learning rate. Regard the parameters as a time-evolving continuous variable $\Btheta_t$ for $t \geq 0$ that develops in the course of the optimization. Then, \citet{arora2019exact} showed that
\begin{align} \label{eq:gradient_flow}
\frac{d}{dt} f\left( \X, \theta_t \right) = - \K_t \cdot \left( f(\X, \theta_t) - \y \right)
% \footnotemark
\end{align}
% \footnotetext{See Lemma 3.1 in \cite{arora2019exact} for the detail.}
where $\K_t := \frac{\partial f(\X,\Btheta_t)}{\partial \Btheta} \left(\frac{\partial f(\X, \Btheta_t)}{\partial \Btheta}\right)^\top \in \R^{n \times n}$. In the infinite width limit, i.e., $d_1, \dots, d_L \rightarrow \infty$, recent works analyzed that $\Btheta_t$ remains constant during optimization, i.e., equals to $\Btheta_0$~\cite{chizat2019lazy,allen2019convergence,du2018gradient} and $\K_t = \K_0$.
Furthermore, under a certain random initialization, in the same infinite width limit, $\K_0$ converges in probability to $\K_{\NTK}^{(L)}$.  
%This also leads to that $\K_t$ converges to deterministic kernel $\K^{(L)}_{\NTK}$ when $t \rightarrow \infty$.
% the time-evolving dynamic of $\theta(t)$ can be captured by gradient flow~\cite{jacot2018neural}:
This implies equivalence between the prediction of neural network under the some initialization and kernel regression with NTK~\citep{arora2019exact}.

In addition, when the parameters of only last layer are updated, the network prediction corresponds to the kernel so-called the neural network Gaussian Process (NNGP):
\begin{align} \label{def:nngp}
% \K_{\NNGP}^{(L)} = \E_{\theta}\left[  f_{\theta}(\X) f_{\theta}(\X)^\top \right]
K_{\NNGP}^{(L)}(\x,\x^\prime) = \E_{\Btheta}\left[  f(\x, \Btheta) \cdot f(\x^\prime, \Btheta) \right]
\end{align}
where the expectation is same as the NTK and we denote $[\K_{\NNGP}^{(L)}]_{i,j} := K_{\NNGP}^{(L)}(\x_i, \x_j)$.

{\bf NTK computation.} 
The NTK matrix of a fully-connected ReLU network can be computed by the following recursive relation~\cite{jacot2018neural,chizat2019lazy,arora2019exact}:
\begin{align}
\begin{aligned}
% &K_{\NTK}^{(0)}(\x, \x^\prime) = K_{\NNGP}^{(0)}(\x, \x^\prime) = \frac{\x^\top {\x^\prime}}{d}, \nonumber \\
% &K_{\NNGP}^{(\ell)}(\x, \x^\prime) = A_1\left(\K_{\NNGP}^{(\ell-1)}\right) \\
% &K_{\NTK}^{(\ell)}(\x, \x^\prime)  = A_1\left(\K_{\NNGP}^{(\ell-1)}\right) + \K_{\NTK}^{(\ell-1)}\odot  A_0\left(\K_{\NNGP}^{(\ell-1)}\right) \nonumber
&\K_{\NTK}^{(0)} = \K_{\NNGP}^{(0)} = {\X\X^\top}, \\
&\K_{\NNGP}^{(\ell)} = F_1\left(\K_{\NNGP}^{(\ell-1)}\right), \\
&\K_{\NTK}^{(\ell)}  = \K_{\NNGP}^{(\ell)} + \K_{\NTK}^{(\ell-1)}\odot  F_0\left(\K_{\NNGP}^{(\ell-1)}\right),
\end{aligned}
\end{align}
where $F_0, F_1:\R^{n\times n} \rightarrow \R^{n \times n}$ are defined as
\begin{align*}
[F_0(\K)]_{i,j} &:= 1 - \frac{1}{\pi} \cos^{-1}\left( \frac{[\K]_{i,j}}{\sqrt{[\K]_{i,i} [\K]_{j,j}}}\right),\\
[F_1(\K)]_{i,j} &:= \sqrt{[\K]_{i,i}[\K]_{j,j}} \cdot f\left(\frac{[\K]_{i,j}}{\sqrt{[\K]_{i,i} [\K]_{j,j}}}\right)   
% [F_0(\K)]_{i,j} := 1 - \frac{1}{\pi} \cos^{-1}\left( \frac{[\K]_{i,j}}{\sqrt{[\K]_{i,i} [\K]_{j,j}}}\right), \quad 
% [F_1(\K)]_{i,j} := \sqrt{[\K]_{i,i}[\K]_{j,j}} \cdot f\left(\frac{[\K]_{i,j}}{\sqrt{[\K]_{i,i} [\K]_{j,j}}}\right)   
\end{align*}
where $f(x) = \frac1{\pi}\left( \sqrt{1-x^2} + (\pi- \cos^{-1}(x))x\right)$ for $x \in [-1,1]$ and $\K$ is an arbitrary positive semidefinite matrix. Note that these matrix functions are derived from arc-cosine kernels~\cite{cho2009kernel}:
\begin{align}\label{def:arccos_kernels}
\begin{aligned}
A_0(\x, \x^\prime) &:= 1 - \frac{1}{\pi} \cos^{-1}\left( \frac{\inner{\x, \x^\prime}}{\norm{\x}_2 \norm{\x^\prime}_2}\right), \\
A_1(\x, \x^\prime) &:= \norm{\x}_2 \norm{\x^\prime}_2 f\left( \frac{\inner{\x, \x^\prime}}{\norm{\x}_2 \norm{\x^\prime}_2} \right).
\end{aligned}
% A_0(\x, \x^\prime) := 1 - \frac{1}{\pi} \cos^{-1}\left( \frac{\inner{\x, \x^\prime}}{\norm{\x}_2 \norm{\x^\prime}_2}\right), \quad
% A_1(\x, \x^\prime) := \norm{\x}_2 \norm{\x^\prime}_2 f\left( \frac{\inner{\x, \x^\prime}}{\norm{\x}_2 \norm{\x^\prime}_2} \right).
% \begin{aligned} \label{def:arccos_kernels}
% &A_0(\x, \x^\prime) := 1 - \frac{1}{\pi} \cos^{-1}\left( \frac{\inner{\x, \x^\prime}}{\norm{\x}_2 \norm{\x^\prime}_2}\right), \\
% &A_1(\x, \x^\prime) := \norm{\x}_2 \norm{\x^\prime}_2 f\left( \frac{\inner{\x, \x^\prime}}{\norm{\x}_2 \norm{\x^\prime}_2} \right).
% \end{aligned}
\end{align}
Computing NTK of the network with $L$ layers takes $\bigo(n^2 (d + L))$ time and $\bigo(n(n+d))$ space complexity. 
% Even one estimates the entire entries in prior, most of kernel methods require much expensive computations than computing NTK. % However, 
% To overcome this limitation, one can consider the feature map of NTK which can reduce the quadratic dependency of the number of data points $n$. However, from the definition, the dimension of random features \cref{def:ntk} is equal to the number of parameters in network, in which is much larger than $n$.

%%%%%%%%%%%%%%%%%%%%%%%%%%%%%%%%%%%%%%%%%%%%%%%%%%%%%%%%%%%%%%%%%%%%%%%%%%%%%%%%%%%%%%%%%%%%%%%%
%%%%%%%%%%%%%%%%%%%%%%%%%%%%%%%%%%%%%%%%%%%%%%%%%%%%%%%%%%%%%%%%%%%%%%%%%%%%%%%%%%%%%%%%%%%%%%%%

\subsection{Random Features and Spectral Approximation} \label{sec:modified_random_features}

Random features~\cite{rahimi2007random} is a methodology for scaling kernel methods that saves both time and storage. In most general terms, the random features model targets kernels $K: \R^d \times \R^d \rightarrow \R$ that can be written as $K(\x, \x^\prime) = \E_{\v \sim {p}}\left[ {\BPhi(\x, \v) \cdot \BPhi(\x^\prime, \v)}\right]$ for some distribution $p$ and a function $\BPhi: \R^{d} \times \R^d \rightarrow \R$. The random features approximations works as follows. First, we generate $m$ vectors $\v_1, \dots, \v_m \in \R^d$  sampled from $p$. We then define the feature map as \
\begin{align*}
\BPhi_m(\x) := \frac{1}{\sqrt{m}} \left[\BPhi(\x, \v_1), \dots, \BPhi(\x, \v_m)\right]^\top \in \R^{m}
\end{align*}
and the approximate kernel is $K^\prime(\x,\x^\prime)=\inner{\BPhi_m(\x), \BPhi_m(\x^\prime)}$.

The main utility of the random feature map is due to the fact that the kernel matrix $\K^\prime$ associated with $K^\prime$ is a low-rank matrix with a known factorization. In particular, let 
$\BPhi := \left[\BPhi_m(\x_1), \dots, \BPhi_m(\x_n) \right]^\top \in \R^{n \times m}$, then $\K^\prime = \BPhi \BPhi^\top \approx \K$. 
The rank of the approximate kernel matrix is $m$, which allows faster computation and less storage. The parameter $m$  trades between computational complexity and approximation quality. A small $m$ results in faster speedup but less accurate kernel approximation.

Although the random features can approximate well the kernel function itself, it is still questionable how it affects the performance of downstream tasks. Several works on kernel approximation adopt {\it spectral approximation} bound with regularization $\lambda > 0$, that is, 
\begin{align*}
(1-\varepsilon) \left(\K + \lambda \I\right)
% \frac{{\BPhi} {\BPhi}^\top + \lambda \I}{1 + \varepsilon}
\preceq
\BPhi \BPhi^\top + \lambda \I
\preceq
% \frac{{\BPhi} {\BPhi}^\top + \lambda \I}{1 - \varepsilon}
(1+\varepsilon)\left( \K + \lambda \I \right)
\end{align*}
for $\varepsilon > 0$ and show that it can provide rigorous guarantees of downstream applications including kernel ridge regression~\cite{avron2017random}, clustering and PCA~\cite{musco2017recursive}.

\vspace{-0.1in}
\subsection{\CS and \TS Transforms} \label{sec:sketch}

The \CS transform is a norm-preserving dimensionality reduction technique~\cite{charikar2002finding}. Formally, let $h:[d] \rightarrow [m]$ be a pairwise independent hash function whose bins are chosen uniformly at random and $s: [d] \rightarrow \{+1, -1\}$ be a pairwise independent sign function where signs are chosen uniformly at random. Given $\x \in \R^d$ and $m \in \mathbb{N}$, we define $\mathcal{C}:\R^d \rightarrow \R^m$ such that for $i \in [m]$ 
\begin{align}
[\mathcal{C}(\x)]_i = \sum_{j: h(j)=i} s(j) [\x]_j 
% \qquad \text{for } \ i \in [m]
\end{align}
and it is well-known that $\E[\inner{\mathcal{C}(\x), \mathcal{C}(\y)}] = \inner{\x, \y}$.
Observe that it requires a single pass over the input, hence the running time becomes $\bigo(d)$.

\citet{pham2013fast} proposed an efficient algorithm to apply \CS to the tensor product of vectors and referred to this as {\sc TensorSketch}.
% \citet{pham2013fast} showed that it can be more effective when 
% When \CS applies to the tensor product of vectors (so-called {\sc TensorSketch}), 
% that a \CS with specific structures can be computed more efficiently 
% inspired by the fast convolution operations. More precisely, 
% They observed that specific sketches on tensor products can be considered as the convolution of {\sc CountSketch}es. 
Let $h_1: [d_1] \rightarrow [m]$, $h_2: [d_2] \rightarrow [m]$ be pairwise independent random hash functions and $s_1:[d_1] \rightarrow \{-1,1\}$, $s_2:[d_2] \rightarrow \{-1,1\}$ be pairwise independent sign functions. We denote the corresponding {\sc CountSketch}es by $\mathcal{C}_1$ and $\mathcal{C}_2$, respectively. Now consider a new transform $\mathcal{C} : \R^{d_1 d_2} \rightarrow \R^m$ whose hash and sign functions are defined as
% \begin{align*}
$H(j_1, j_2) \equiv h_1(j_1) + h_2(j_2) \pmod{m}$ and $S(j_1, j_2) = s_1(j_1) \cdot s_2(j_2)$ for $j_1 \in [d_1], j_2 \in [d_2]$.
Given $\x \in \R^{d_1}, \y \in \R^{d_2}$, \citet{pham2013fast} showed that $\mathcal{C}\left( \x \otimes \y \right)$ equals the convolution between $\mathcal{C}_1(\x)$ and $\mathcal{C}_2(\y)$ and its computation can be amortized as
% Then, the transform $\mathcal{C}$ is amortized as
% For any $\x \in \R^{d_1}, \y \in \R^{d_2}$, consider
% and sign functions
% $\mathcal{C}:\R^d \rightarrow \R^m$ is defined as for $i \in [m]$ 
% \begin{align}
% [\mathcal{C}(\x)]_i = \sum_{j: h(j)=i} s(j) \x_j.
% \end{align}
% \begin{align*}
% &\left[ \mathcal{C}\left( \x \otimes \y \right) \right]_i 
% = \sum_{(j_1, j_2) : H(j_1, j_2) = i} S(j_1, j_2) \left[ \x \y^\top \right]_{j_1, j_2} \\
% &= \sum_{(j_1, j_2) : H(j_1, j_2) = i} \left(s_1(j_1) \x_{j_1} \right) \left(s_2(j_2) \x_{j_2} \right) \\
% &= \sum_{(j_1, j_2) : h_1(j_1) + h_2(j_2) \equiv i} \left(s_1(j_1) \x_{j_1} \right) \left(s_2(j_2) \x_{j_2} \right) \\
% \end{align*}
\begin{align} \label{eq:tensorsketch}
\mathcal{C}\left( \x \otimes \y \right) = \iFFT\left( \FFT(\mathcal{C}_1(\x)) \odot \FFT(\mathcal{C}_2(\y)) \right)
    \end{align}
where $\FFT$ and $\FFT^{-1}$ are the Fast Fourier Transform and its inverse.
% where $\mathcal{C}_1, \mathcal{C}_2$ are 2-wise independent {\sc CountSketch}es. 
% Observe that \cref{eq:tensorsketch} requires $\bigo(d + m \log m)$ operations instead of $\bigo(d^2)$. 
By the inner product preserving property, \TS also can be used as a low-rank approximation of element-wise product between two Gramian matrices. More specifically, given $\X \in \R^{n \times d_1}, \Y \in \R^{n \times d_2}$, it holds
\begin{align} \label{eq:inner_product_approximating_property}
\begin{aligned}
&\left(\X \X^\top \right) \odot  \left( \Y \Y^\top\right) 
= \left( \X \otimes \Y \right) \left( \X \otimes \Y \right)^\top
= \E\left[\mathcal{C} \left( \X \otimes \Y \right)  \mathcal{C}  \left( \X \otimes \Y \right)^\top \right] 
\end{aligned}
\end{align}
where $\otimes$ and $\mathcal{C}$ are performed in a row-wise manner.
Note that %both time and space complexities with \TS are
$\mathcal{C} \left( \X \otimes \Y \right) \in \R^{n \times m}$ can be computed in $\bigo(n(d_1 + d_2 + m \log m))$ time using  \cref{eq:tensorsketch}. This is much cheaper than that of computing $n$-by-$n$ dense matrix when $n \gg d_1, d_2, m$. 
A larger $m$ guarantees better approximation quality but also increases its running time. 
% One can choose $m \ll \min(d_1 d_2, n)$ so that it runs much faster. 
\citet{avron2017faster, ahle2020oblivious, woodruff2020near} analyzed a spectral approximation guarantee of \TS transform.
% which can be used in various downstream tasks.
% analyzed the error of \TS and we introduce these results in \cref{sec:mainresults}.

For simplicity, we presented \TS for vectors that are the tensor product of only two vectors. This is enough for our needs, though we mention that the \TS transform can be defined for an arbitrary number of tensor products.

\section{Efficient NTK Random Features via Sketching Method} \label{sec:mainresults}

Our goal is to design efficient such random features for the NTK. Seemingly, one can obtain such random features from definition of NTK (i.e., \cref{def:ntk}) by using gradients of the randomly initialized networks as features.
However, the NTK is the infinite-width limit, while in practice we need to fix some finite width, which will introduce a bias. Moreover, \citet{arora2019exact} showed that the gradient features from a network with finite widths degrade the practical performance by a huge gap. Instead, we focus on the closed-form expression of NTK with ReLU activations.

\begin{algorithm}[h]
\caption{Random Features for NTK of ReLU network via \CS} \label{alg:ntk_random_features}
\begin{algorithmic}[1]
	\STATE {\bf Input}: $\x \in \R^{d}$, network depth $L$, feature dimensions $m_0$, $m_1$ and $m_{\mathtt{cs}}$
	\STATE ${\BPhi}^{(0)}(\x) \leftarrow \x, \BPsi^{(0)}(\x) \leftarrow \x$, and $m \leftarrow d$
	\FOR{ $\ell=1$ to $L$}
		\STATE Draw i.i.d. $\w_i \sim \mathcal{N}(\boldsymbol{0}, \I_{m})$ for $i\in [m_0]$ and
		$$\BLambda^{(\ell)}(\x) \leftarrow \sqrt{\frac{2}{m_0}} \ \mathrm{Step}\left(\begin{bmatrix}\w_1, \dots, \w_{m_0}\end{bmatrix}^\top \BPsi^{(\ell-1)}(\x)\right)$$
		\STATE Draw i.i.d. ${\w}^{\prime}_j \sim \mathcal{N}(\boldsymbol{0}, \I_{m})$ for $j \in [m_1]$ and
		$$\BPsi^{(\ell)}(\x) \leftarrow \sqrt{\frac{2}{m_1}} \ \mathrm{ReLU}\left( \begin{bmatrix} \w^{\prime}_1, \dots, \w^{\prime}_{m_1} \end{bmatrix}^\top \BPsi^{(\ell-1)}(\x) \right)$$ %\frac{1}{\sqrt{m_1}} 
% 		\STATE Draw two independent \CS transforms $\mathcal{C}_0^{(\ell)}: \R^{m_0} \rightarrow \R^{m_{\mathtt{cs}}}$, $\mathcal{C}_1^{(\ell)}: \R^{d_1} \rightarrow \R^{m_{\mathtt{cs}}}$ and
        \STATE Draw two independent \CS transforms $\mathcal{C}_0^{(\ell)}$ and $\mathcal{C}_1^{(\ell)}$ that map to $\R^{m_{\mathtt{cs}}}$ and 
        $$\BGamma^{(\ell)}(\x) \leftarrow \iFFT\left(\FFT(\mathcal{C}^{(\ell)}_0(\BLambda^{(\ell)}(\x))) \odot \FFT(\mathcal{C}^{(\ell)}_1({\BPhi}^{(\ell-1)}(\x)))\right)$$
		\STATE ${\BPhi}^{(\ell)}(\x) \leftarrow \begin{bmatrix} \BPsi^{(\ell)}(\x), \ \BGamma^{(\ell)}(\x) \end{bmatrix}$, $m \leftarrow m_1$
% 		$d_0 \leftarrow m_1 + m_{\mathtt{cs}}$ and $d_1 \leftarrow m_1$
	\ENDFOR
	\STATE {\bf return} ${\BPhi}^{(L)}(\x)$
\end{algorithmic}
\end{algorithm}

% \subsection{Feature Construction via Sketching Method}

We begin by introducing random features of arc-cosine kernels $A_0$ and $A_1$ originally due to~\citet{cho2009kernel}:
\begin{align}  
a_0(\x) &= \sqrt{\frac{2}{m_0}} \ \mathrm{Step} \left( \left[ \w_1, \dots, \w_{m_0} \right]^\top \x\right) \label{eq:a0_arccosine_random_features}, \\
a_1(\x) &= \sqrt{\frac{2}{m_1}} \ \mathrm{ReLU} \left( \left[ \w^\prime_1, \dots, \w_{m_1}^\prime \right]^\top \x\right) \label{eq:a1_arccosine_random_features}
\end{align}
where $\w_1, \dots, \w_{m_0}, \w^\prime_{1}, \dots, \w^\prime_{m_1}\in \R^d$ are sampled from 
%i.i.d. random samples from 
$\mathcal{N}(\boldsymbol{0}, \I_d)$. It is known that $\E[\inner{a_0(\x), a_0(\x^\prime)}] = A_0(\x, \x^\prime)$ and $\E[\inner{a_1(\x), a_1(\x^\prime)}] = A_1(\x, \x^\prime)$ for $\x, \x^\prime \in \R^d$.
% that map to $\R^{m_0}$ and $\R^{m_1}$, respectively.
% such that $\E\left[\phi_0(\X) \phi_0(\X)^\top\right] = \A_0$ and $\E\left[ \phi_1(\X) \phi_1(\X)^\top\right] = \A_1$, respectively. 

Recently,~\citet{bietti2019inductive} presented an explicit infinite-dimensional feature map for the NTK of ReLU networks by using recursive tensoring of explicit feature maps for the arc-cosine kernel. Replacing each explicit feature map with a random feature map for the corresponding kernel we can obtain a random feature map for the NTK. The resulting construction is:
% \begin{align*}
% \BPhi_{\NTK}^{(1)} &= \BPhi_{\NNGP}^{(1)} = \frac{\X }{\sqrt{d}}, \ \ \BPhi_{\NNGP}^{(\ell)} = \phi_1 \left(\BPhi_{\NNGP}^{(\ell-1)}\right), \\
% \BPhi_{\NTK}^{(\ell)} &= \left[ \phi_1 \left(\BPhi_{\NNGP}^{(\ell-1)}\right), \quad \phi_0 \left(\BPhi_{\NNGP}^{(\ell-1)}\right) \otimes \BPhi_{\NTK}^{(\ell-1)} \right],
% \end{align*}
\begin{align} \label{eq:ntk_feature_construction_tensor_product} 
\begin{aligned}
\BPsi^{(\ell+1)}(\x) &= a_1 \left(\BPsi^{(\ell)}(\x)\right), \quad \BPhi^{(0)}(\x) = \BPsi^{(0)}(\x) = \x,  \\
\BPhi^{(\ell+1)}(\x) &= \left[ \BPsi^{(\ell+1)}(\x), \quad a_0 \left(\BPsi^{(\ell)}(\x)\right) \otimes \BPhi^{(\ell)}(\x) \right], 
\end{aligned}
\end{align}
for $\ell = 0, \dots, L-1$. These features can be used for approximating both NTK and NNGP as 
% $\K_{\NTK}^{(L)} \approx \BPhi^{(L)} {\BPhi^{(L)}}^\top$ and $\K_{\NNGP}^{(L)} \approx \BPsi^{(L)} {\BPsi^{(L)}}^\top$.
$K_{\NTK}^{(\ell)}(\x, \x^\prime) \approx \inner{ \BPhi^{(\ell)}(\x), {\BPhi^{(\ell)}}(\x^\prime)}$ and $K_{\NNGP}^{(\ell)}(\x, \x^\prime) \approx \inner{\BPsi^{(\ell)}(\x), {\BPsi^{(\ell)}}(\x^\prime)}$.
% where $\otimes$ is the tensor product (a.k.a. Kronecker product) in a row-wise manner.
% Assume that $\phi_0, \phi_1$ map to finite dimensions $m_0, m_1$, respectively. Then, t

However, one major drawback of the last construction is that the number of features is exponential in the depth. Indeed, the dimension of output features $\BPhi^{(L)}(\x)$ is $\left( \sum_{k=0}^{L-1} m_0^k\right) m_1 + m_0^{L} d = \bigo(m_0^L (m_1 + d))$. This also leads to $\bigo(m_0^L (m_1 + d) + m_1^2)$ time complexity. The exponential growth in depth $L$ is due to the tensor product $\otimes$ in \cref{eq:ntk_feature_construction_tensor_product}. For a large $L$, the number of features can easily be larger than the number of data points $n$ and any computational saving is hopeless.

In order to make the feature map more compact, we utilize a \TS  to reduce the dimension of
% $a_0 \left(\BPsi^{(\ell-1)}\right) \otimes \BPhi^{(\ell-1)}$. 
$a_0 \left(\BPsi^{(\ell)}(\x)\right) \otimes \BPhi^{(\ell)}(\x)$. 
% of features from tensor products and 
% reduce the dimension of $\phi_0 \left(\BPsi^{(\ell-1)}\right) \otimes \BPhi^{(\ell-1)}$
% introduced in \cref{eq:inner_product_approximating_property} and 
We do so by replacing it with 
\begin{align*}
% \BGamma :=\FFT^{-1} \left( \FFT\left(\mathcal{C}_1(\BPsi^{(\ell-1)})\right) \odot \FFT\left(\mathcal{C}_2(\BPhi^{(\ell-1)})\right) \right)
\BGamma^{(\ell)}(\x) :=\FFT^{-1} \left( \FFT\left(\mathcal{C}_1(a_0 (\BPsi^{(\ell)}(\x)))\right) \odot \FFT\left(\mathcal{C}_2(\BPhi^{(\ell)}(\x))\right) \right)
\end{align*}
% introduced in \cref{sec:sketch}. Suppose that $\phi_0 \left(\BPsi^{(\ell-1)}\right)$ and $\BPhi^{(\ell-1)}$ have $d_1$ and $d_2$ columns, respectively. Let $\mathcal{C}_1$ and $\mathcal{C}_2$ be two \CS transforms from $\R^{d_1}, \R^{d_2}$ to $\R^{m}$. Then, we can approximate as follows:
% \begin{align}
% &\phi_0 \left(\BPsi^{(\ell-1)}\right) \otimes \BPhi^{(\ell-1)} 
% \approx \FFT^{-1} \left( \FFT(\C_1) \odot \FFT(\C_2) \right)
% % \FFT^{-1} \left( 
% % \FFT\left( \mathcal{C}_1 \left( \phi_0 \left(\BPsi^{(\ell-1)}\right) \right) \right) \odot
% % \FFT\left( \mathcal{C}_2 \left( \BPhi^{(\ell-1)} \right)\right)
% % \right)
% \end{align}
where $\mathcal{C}_1$ and $\mathcal{C}_2$ are independent\footnotemark{}  {\sc CountSketch} transforms that map to $\R^{m_{\mathtt{cs}}}$. 
Denote $\widehat{\BPhi}^{(\ell)}(\x) := \begin{bmatrix}\BPsi^{(\ell)}(\x), \ \BGamma^{(\ell)}(\x) \end{bmatrix}$ and one can expect that 
% \begin{align}
% \BPhi^{(\ell+1)} \left( \BPhi^{(\ell+1)} \right)^\top 
% \approx
% \widehat{\BPhi}^{(\ell+1)} \left( \widehat{\BPhi}^{(\ell+1)} \right)^\top 
% \end{align}
\begin{align}
\inner{ \BPhi^{(\ell)}(\x) , \BPhi^{(\ell)}(\x^\prime) }
\approx
\inner{ \widehat{\BPhi}^{(\ell)}(\x), \widehat{\BPhi}^{(\ell)}(\x^\prime) }
\end{align}
from the property in \cref{eq:inner_product_approximating_property}.  The process is repeated for every layer. A pseudo-code for the proposed feature construction is described in~\cref{alg:ntk_random_features}.
% Note that the dimension of $\widehat{\BPhi}^{(\ell)}(\x)$ equals $m_1 + m_{\mathtt{cs}}$ regardless of the network depth $L$.

\iffalse
\textcolor{red}{
After our submission, we found that \cref{thm:random_features} in the main draft should be revised.
To be specific, the {\it unbiased} property of random features only holds for $L=2$ and does not generalize for $L\geq 3$. 
Two-layer fully-connected neural networks are widely used in deep learning applications and we indeed observe that $L=2$ empirically suffices to perform better than the baseline method (see \cref{sec:experiments}). Even $L=2$, the dimension of the proposed feature can be smaller than that of  \cref{eq:ntk_feature_construction_tensor_product}, i.e., $(m_1 + m_0 d)$.
Here, we provide the revised statement.
% We now show that the generated features from \cref{alg:ntk_random_features} are indeed random features of NTK and NNGP, and that the running time is linear in both input data and network depth $L$ in below.
}
\fi
We now provide that the approximation error bound of generated features from \cref{alg:ntk_random_features}.
% and the running time is linear in the size of input data in below.

\begin{restatable}{theorem}{randomfeatureserror} \label{thm:ntk-random-features-error}
Given $\x, \y \in \mathbb{R}^d$ such that $\norm{\x}_2 =\norm{\x^\prime}_2 = 1$ and $L \geq 1$, let $K_{\NTK}^{(L)}$ the NTK of $L$-layer fully-connected ReLU network. Given $\delta, \in (0,1)$,  $\varepsilon \in (0 ,1/L)$, there exist constants $C_0, C_1, C_2 > 0$ such that
$$m_0 \geq C_0 \frac{L^2}{\varepsilon^2}  \log\left( \frac{L}{\delta} \right), \ \ 
m_1 \geq C_1 \frac{L^6}{\varepsilon^4} \log\left( \frac{L}{\delta}\right), \ \ 
m_{\mathtt{cs}}\geq C_2 \frac{L^3}{\varepsilon^2 \delta}$$
% \begin{itemize}
% 	\item $m_0 \geq C_0 \frac{L^2}{\varepsilon^2}  \log\left( \frac{L}{\delta} \right)$,
% 	\item $m_1 \geq C_1 \frac{L^6}{\varepsilon^4} \log\left( \frac{L}{\delta}\right)$,
% 	\item $m_{\mathtt{cs}}\geq C_2 \frac{L^3}{\varepsilon^2 \delta}$,
% \end{itemize}
and 
\begin{align*}
\Pr\left( \abs{\inner{{\BPhi}^{(L)}(\x), \BPhi^{(L)}(\x^\prime)} - K_{\NTK}^{(L)}(\x,\x^\prime)} \leq L \varepsilon \left(1 + \frac{\varepsilon}{2} \right)^2 + \varepsilon \right) \geq 1 - \delta
\end{align*}
where $\BPhi^{(L)}(\x), \BPhi^{(L)}(\x^\prime) \in \R^{m_1 + m_{\mathtt{cs}}}$ be the output of \cref{alg:ntk_random_features} of $\x, \x^\prime$, respectively, using the same \CS transforms. %Moreover, the running time to compute $\BPhi^{(L)}(\x)$ is \bigo().
\end{restatable}

The proof of \cref{thm:ntk-random-features-error} is provided in~\cref{sec:proof-ntk-random-features-error}. 
We note that the restriction of inputs to the hypersphere (i.e., $\norm{\x_i}_2=1$) is a common assumption used in the NTK analysis~\cite{bietti2019inductive, geifman2020similarity}. This can be easily achieved by normalizing input data points.
From \cref{thm:ntk-random-features-error}, the dimension of the proposed random features can be $\bigo\left(\frac{L^6}{\varepsilon^4} \log\left(\frac{L}{\delta}\right) + \frac{L^3}{\varepsilon^2 \delta} \right)$, which gets rid of the exponential dependency on $L$, to guarantee the above error bound.  Furthermore, \citet{arora2019exact} studied that the gradient of randomly initialized ReLU network with finite width can approximate the NTK, but their feature dimension should be $\Omega\left( \frac{L^{13}}{\varepsilon^8} \log^2 \left( \frac{L}{\delta}\right) +  \frac{L^{6}}{\varepsilon^4} \log \left( \frac{L}{\delta}\right) d \right)$ to guarantee an approximation error of $(L+1)\varepsilon$ with probability at least $1-\delta$.
This error bound is smaller than that in \cref{thm:ntk-random-features-error} by a factor of $\left(1+\frac{\varepsilon}{2}\right)^2$, but their feature dimension is much larger by a factor of $\bigo\left( \frac{L^7}{\varepsilon^4} \log \left(\frac{L}{\delta}\right)\right)$.
% We remark that the dimension of random features from \cref{alg:ntk_random_features} is $\bigo\left(\frac{L^6}{\varepsilon^4} \log\left(\frac{L}{\delta}\right) + \frac{L^3}{\varepsilon^2 \delta} \right)$ to achieve the above approximation error.
% % $L \varepsilon \left(1 + \frac{\varepsilon}{2} \right)^2 + \varepsilon$.
% From Theorem 3.1 in \cite{arora2019exact}, the gradient features of randomly initialized ReLU network with finite width 
% can also approximate the NTK, but their feature dimension should be $\Omega\left( \frac{L^{13}}{\varepsilon^8} \log^2 \left( \frac{L}{\delta}\right) +  \frac{L^{6}}{\varepsilon^4} \log \left( \frac{L}{\delta}\right) d \right)$ to guarantee the approximation error $(L+1)\varepsilon$.
We empirically observe that \cref{alg:ntk_random_features} requires much fewer dimension than both the random gradient and the na\"ive feature map construction in \cref{eq:ntk_feature_construction_tensor_product} to achieve the same error and provide these results in \cref{sec:kernel-approx}.

\footnotetext{i.e., hash and sign functions of $\mathcal{C}_1$ and $\mathcal{C}_2$ are independent.}

% With \TS, the random features has 
% for $\mathcal{C}: \R$
% that can efficiently reduce the dimension of tensor products while their inner products are approximately preserved. Formally, given $X \in \R^{n \times d_1}$, $Y \in \R^{n \times d_2}$ and \CS $\mathcal{C} : \R^{d_1 d_2} \rightarrow \R^m$ 
% %for $m < d_1 d_2$ 
% it holds
% \[
% %\inner{\x \otimes \x^\prime, \x \otimes \x^\prime}
% (X \otimes Y) (X \otimes Y)^\top
% = \E_{\mathcal{C}}\left[ \mathcal{C}\left( X \otimes Y \right) \ \mathcal{C}\left( X \otimes Y \right)^\top \right]
% \]
% %where larger $m < d_1 d_2$ guarantees better approximation quality. Furthermore, 
% while the running time of Count Sketch is $\bigo(n(d_1 + d_2 + m \log m))$. A larger $m$ guarantees better approximation quality but also increases its running time\footnote{We provide formal approximation guarantee of \CS in \cref{sec:proof_cs_spectral}.}. One can choose $m \ll \min(d_1 d_2, n)$ so that it runs much faster than both the exact kernel computation and \cref{eq:ntk_feature_construction_tensor_product}.
%Hence, the total running time can relies on feature construction of $f_0$ and $f_1$.

\section{Spectral Approximation for the NTK} \label{sec:ntk_spectral}

Our ultimate goal is to provide lower bounds on the parameter $m_0, m_1, m_\mathtt{cs}$ to achieve tight error bound in terms of spectral approximation of the NTK, i.e.,
$$
(1-\varepsilon) \left( \K^{(L)}_{\NTK} + \lambda \I\right)
\preceq
\BPhi^{(L)} (\BPhi^{(L)})^\top + \lambda \I
\preceq
(1+\varepsilon) \left( \K^{(L)}_{\NTK} + \lambda \I\right),
$$
where $\BPhi^{(L)} := \left[ \BPhi^{(L)}(\x_1), \dots, \BPhi^{(L)}(\x_n) \right]^\top$. 
\iffalse
Before diving into the spectral approximation, %of~\cref{alg:ntk_random_features}, 
in this subsection we prove spectral bounds on the arc-cosine kernels in~\cref{def:arccos_kernels}, 
which is a necessary prerequisite of our analysis on the NTK random features.
\fi
We first provide spectral bounds of the arc-cosine kernels in~\cref{def:arccos_kernels}, 
which are necessary prerequisites of our analysis on the NTK random features. 
Based on these results, we present spectral bound of a two-layer ReLU network (i.e., $L=1$ in \cref{eq:nn-definition}) and discuss hardness on generalizing this result to networks with deeper layer.
% spectral approximation bound of NTK random features is 
% that of arc-cosine random features $a_0, a_1$.
% need to be spectral approximations of kernel $A_0, A_1$. 
% As proposed by \citet{cho2009kernel}, one can use the random features for arc-cosine kernels as
% % na\"ive choices for arc-cosine feature maps can be 
% \begin{align*}
% \phi_0(\X) = \frac{\mathrm{Step}(\X \W_0)}{\sqrt{m_0}}, \quad \phi_1(\X) = \frac{\mathrm{ReLU}(\X \W_1)}{\sqrt{m_1}}, 
% \end{align*}
% where entries in $\W_0 \in \R^{d \times m_0},\W_1 \in \R^{d \times m_1}$ are i.i.d. samples from $\mathcal{N}(0,1)$.
To the best of our knowledge, this has not been studied in previous literature.
% explored whether $a_0$ or $a_1$ can spectrally approximate.
% that such random features are the spectral approximations.
% the error analysis of these feature maps were not studied. In this work, we propose a novel analysis with their spectral approximations.
% proposed random features with the leverage score sampling~\cite{avron2017random} and analyzed the spectral approximation guarantee to the target kernel matrix which can be directly applied to the downstream guarantees. 
% \textcolor{red}{TODO: add more intuition on leverage score sampling.}
% Hence, we use to 
% Instead, we make use of the results in~\citet{lee2020generalized} introduced in \cref{sec:}
% Recently, proposed that random features scaled by proper weights can provide spectral approximation bounds.
% Motivated by this work, we aim to design modified arc-cosine random features that can guarantee spectral approximation bounds. 

\subsection{Spectral Approximation for Arc-cosine Kernels}

Recently, \citet{avron2017random, lee2020generalized} proposed that random features with sampling from a modified distribution can give better approximation guarantee. More precisely, 
suppose $\BPhi:\R^d \times \R^d \rightarrow \R$ is a function for random features of kernel $K$ with distribution $p$.
Consider random vectors $\z_1, \dots, \z_m \in \R^d$ sampled from some distribution $q$. Denote that
\begin{align*}
\overline{\BPhi}_m(\x) := \frac{1}{\sqrt{m}} \left[
\sqrt{\frac{p(\z_1)}{q(\z_1)}} \BPhi(\x, \z_1), \dots\sqrt{\frac{p(\z_m)}{q(\z_m)}} \BPhi(\x, \z_m)
\right]^\top \in \R^m
\end{align*}
then one can verify that $\E_{\z\sim q}\left[\inner{\overline{\BPhi}_m(\x), \overline{\BPhi}_m(\x^{\prime})} \right] = K(\x, \x^\prime)$
% \begin{align*}
% &\E_{\z\sim q}\left[\overline{\phi}_m(\x)^\top \overline{\phi}_m(\x^{\prime}) \right] 
% = \E_{\z\sim q} \left[ \frac{p(\z)}{q(\z)} \phi(\x, \z)^\top \phi(\x^{\prime}, \z)\right] \\
% &= \E_{\v\sim p}\left[{\phi}(\x,\v)^\top {\phi}(\x^{\prime},\v) \right] = K(\x, \x^\prime)
% \end{align*}
for all $\x, \x^\prime \in \R^d$. 
Now assume that the distribution $q$ is defined by a measurable function $q_{\lambda}:\R^d \rightarrow \R$ satisfies that $q(\v) = q_{\lambda}(\v) / \int_{\R^d} q_{\lambda}(\v) d\v$ and 
\begin{align} \label{eq:leverage_score_function}
q_{\lambda}(\v) \geq p(\v) \cdot \BPhi(\X, \v)^\top \left( \K + \lambda \I \right)^{-1} \BPhi(\X, \v)
\end{align} 
for $\lambda > 0$ where $\BPhi(\X, \v) := [\BPhi(\x_1, \v), \dots, \BPhi(\x_n, \v)]^\top \in \R^n$. 
Then, they proved that with high probability it holds
\begin{align*}
(1-\varepsilon)\left( \K + \lambda \I \right)
\preceq
\overline{\BPhi}_m \overline{\BPhi}_m^\top + \lambda \I
\preceq
(1+\varepsilon)\left( \K + \lambda \I \right)
% \frac{\overline{\BPhi}_m \overline{\BPhi}_m^\top + \lambda \I}{1 + \varepsilon}
% \preceq
% \K + \lambda \I 
% \preceq
% \frac{\overline{\BPhi}_m \overline{\BPhi}_m^\top + \lambda \I}{1 - \varepsilon}
\end{align*}
where $\overline{\BPhi}_m = \left[ \overline{\BPhi}_m(\x_1), \dots, \overline{\BPhi}_m(\x_n) \right]^\top \in \R^{n \times m}$ and $\varepsilon \in (0,1)$ is a given parameter.

Observe that the lower bound in \cref{eq:leverage_score_function} requires $\bigo(n^3)$ operations to compute due to the matrix inverse. This can hurt the computational advantage of random features.
Hence, it is important to find such distribution $q$ that is easy to sample while holding the \cref{eq:leverage_score_function}.

% \citet{avron2017random} studied that several distributions for Random Fourier Features (RFF) that approximates Gaussian RBF kernel. However, it is still unclear how to construct a good random features for general kernels. In \cref{sec:mainresults}, we

In what follows, we provide that the original arc-cosine random features of $0$-th order in~\cref{eq:a0_arccosine_random_features} can indeed guarantee a spectral approximation bound. 
% introduce the modified distributions for the arc-cosine random features that can guarantee spectral approximation bounds.

\begin{restatable}{theorem}{azerospectral} \label{thm:a0_spectral}
Given $\X \in \R^{n \times d}$, let $\A_0 \in \R^{n \times n}$ be the arc-cosine kernel matrix of $0$-th order with $\X$ and denote $\BPhi_0 := \sqrt{\frac{2}{m}} \mathrm{Step}(\X \W) \in \R^{n \times m}$ where each entry in $\W\in \R^{d \times m}$ is an i.i.d. sample from $\mathcal{N}(0,1)$. 
Let $s_{\lambda}$ be the statistical dimension of $\A_0$.  Given $\lambda \in (0, \norm{\A_0}_2)$, $\varepsilon \in (0, 1/2)$ and $\delta \in (0,1)$, if $m\geq \frac{8}{3} \frac{n}{\lambda \varepsilon^{2} } \log\left( \frac{16 s_{\lambda}}{\delta} \right)$, then
% then $\BPhi_0 \BPhi_0^\top$ is a $(\varepsilon, \lambda)$-spectral approximation of $\A_0$
it holds
\begin{align*}
(1 - \varepsilon) (\A_0 + \lambda \I) 
\preceq
\BPhi_0 \BPhi_0^\top + \lambda \I 
\preceq 
(1 + \varepsilon) (\A_0 + \lambda \I)
% (1 - \varepsilon) (\A_0 + \lambda \I) 
% \preceq
% \frac{\BPhi_0 \BPhi_0^\top + \lambda \I}{1+\varepsilon} 
% \preceq 
% \A_0 + \lambda \I
% \preceq
% \frac{\BPhi_0 \BPhi_0^\top + \lambda \I}{1-\varepsilon}
\end{align*}
with probability at least $1-\delta$.
\end{restatable}

The proof of \cref{thm:a0_spectral} is provided in \cref{sec:proof-a0-spectral}. The analysis is similar to that studied by~\citet{avron2017faster}, i.e., $q_{\lambda}(\v) = ({n}/{\lambda})p(\v)$, which implies that the modified distribution is identical to the original one.

Next, we present our result on spectral approximation for arc-cosine random features of $1$-st order. Unlike the previous case, sampling vectors from the modified distribution in the form of the Gaussian scaled by squared $\ell_2$-norm is required. The formal statement is provided in \cref{thm:a1_spectral}.

\begin{algorithm}[t]
\caption{Modified Random Arc-cosine Features of $1$-st order~\cref{eq:reweighted_phi1} via Gibbs Sampling
% Gibbs Sampling for \cref{eq:pdf_weighted_normal} via Inverse Transformation Method
} \label{alg:gibbs}
\begin{algorithmic}[1]
\STATE {\bf Input}: $\X \in \R^{n \times d}$, feature dimension $m_1$, Gibbs iterations $T$
\STATE Draw i.i.d. $\v_i \sim \mathcal{N}({\bf 0}, \I_d)$ for $i \in [m_1]$
\FOR{ $i = 1$ to $m_1$}
    \STATE $q(x, z) \leftarrow $ inverse of $\frac{\mathrm{erf}\left( {x}/{\sqrt{2}}\right)+1}{2} - \frac{ x \exp\left( -x^2/2\right)}{\sqrt{2 \pi}(z+1)}$ 
    (corresponds to the CDF of $\Pr([\v_i]_j | [\v_i]_{\setminus \{j\}})$)
    \FOR{ $t = 1$ to $T$}
        \FOR{ $j = 1$ to $d$}
            \STATE $u \leftarrow$ sample from $[0,1]$ at uniformly random
            \STATE $[\v_i]_j \leftarrow q\left(u, \sum_{k \in [d]\setminus\{j\}} [\v_i]_k^2\right)$
        \ENDFOR
    \ENDFOR
\ENDFOR
\STATE {\bf return } $\sqrt{\frac{2d}{m}} \left[
\frac{\mathrm{ReLU}(\X\v_1)}{\norm{\v_1}_2} , \ \dots \ , \frac{\mathrm{ReLU}(\X\v_m)}{\norm{\v_m}_2} \right]$
\end{algorithmic}
\end{algorithm}

\begin{restatable}{theorem}{aonespectral} \label{thm:a1_spectral}
Given $\X \in \R^{n \times d}$, let $\A_1 \in \R^{n \times n}$ be the arc-cosine kernel matrix of $1$-th order with $\X$ and $\v_1, \dots, \v_m \in \R^d$ be i.i.d. random vectors from probability distribution 
\begin{align} \label{eq:pdf_weighted_normal}
q(\v) = \frac{1}{(2\pi)^{d/2} d} \norm{\v}_2^2 \exp\left(-\frac{1}{2}\norm{\v}_2^2\right).
\end{align}
Denote
\begin{align} \label{eq:reweighted_phi1}
\BPhi_1 := \sqrt{\frac{2d}{m}} \left[
\frac{\mathrm{ReLU}(\X\v_1)}{\norm{\v_1}_2} , \ \dots \ , \frac{\mathrm{ReLU}(\X\v_m)}{\norm{\v_m}_2} \right]
\end{align}
and let $s_{\lambda}$ be the statistical dimension of $\A_1$. Given $\lambda \in (0, \norm{\A_1}_2)$, $\varepsilon \in (0, 1/2)$ and $\delta \in (0,1)$, if $m\geq \frac{8}{3} \frac{d \norm{\X}_2^2}{\lambda \varepsilon^{2} } \log\left( \frac{16 s_{\lambda}}{\delta} \right)$, then it holds that
% $\BPhi_1 \BPhi_1^\top$ is a $(\varepsilon, \lambda)$-spectral approximation of $\A_1$ 
% then it holds that
\begin{align*}
(1 - \varepsilon) (\A_1 + \lambda I) 
\preceq
\BPhi_1 \BPhi_1^\top + \lambda \I 
\preceq 
(1 + \varepsilon) (\A_1 + \lambda I)
% \frac{\BPhi_1 \BPhi_1^\top + \lambda \I}{1+\varepsilon}
% \preceq
% \A_1 + \lambda \I
% \preceq
% \frac{\BPhi_1 \BPhi_1^\top + \lambda \I}{1-\varepsilon}
\end{align*}
with probability at least $1-\delta$.
\end{restatable}

The proof of \cref{thm:a1_spectral} is provided in \cref{sec:proof-a1-spectral}. %in \cref{sec:a0_a1_spectral}.
We note that the modified distribution can be expressed as a closed-form formula as in \cref{eq:pdf_weighted_normal}. Once random vectors are sampled from this distribution, the modified random features in \cref{eq:reweighted_phi1} can be computed at the same cost of the original features in \cref{eq:a1_arccosine_random_features}. In addition, the lower bound on feature dimension depends on the square of the spectral norm of input.

{\bf Approximate sampling.}
It is not trivial to sample a vector $\v\in \R^d$ from the distribution $q(\cdot)$ defined in \cref{eq:pdf_weighted_normal}. Thus, we suggest to perform an approximate sampling via Gibbs sampling. The algorithm starts with a random initialized vector $\v$ and then iteratively replaces $[\v]_i$ with a sample from $q([\v]_i | [\v]_{\setminus i})$ for $i\in [d]$ and repeat this process for $T$ iterations. 
One can derive the conditional distribution
\begin{align} \label{eq:conditional_probability}
    q([\v]_i | [\v]_{\setminus \{i\}}) \propto \frac{ \norm{\v}_2^2}{1 + \norm{\v}_2^2 - [\v]_i^2}  \exp\left(- \frac{[\v]_i^2}{2}\right)
\end{align} and sampling a single random variable from \cref{eq:conditional_probability} can be done via the inverse transformation method.\footnote{It requires the CDF of $q([\v]_i | [\v]_{\setminus \{i\}})$ which is equivalent to $\frac{\mathrm{erf}\left( {[\v]_i }/{\sqrt{2}}\right)+1}{2} - \frac{ [\v]_i  \exp\left( -[\v]_i^2/2\right)}{\sqrt{2 \pi}(1 + \norm{\v}_2^2 - [\v]_i^2)}$.} We empirically verify that $T=1$ is enough for promising performances. The running time of Gibbs sampling becomes $\bigo(m_1 d  T)$ where $m_1$ corresponds to the number of independent samples from $q(\v)$. This is negligible compared to the feature map construction of \CS for $T=\bigo(1)$. 
The pseudo-code for the modified random features of $A_1$ using Gibbs sampling is outlined in \cref{alg:gibbs}.
% \paragraph{Sampling from \cref{eq:pdf_weighted_normal}.}

\captionsetup[subfigure]{aboveskip=1pt}
\begin{figure}[t]
\begin{center}
    \begin{subfigure}{0.3\textwidth}
        \includegraphics[width=\columnwidth]{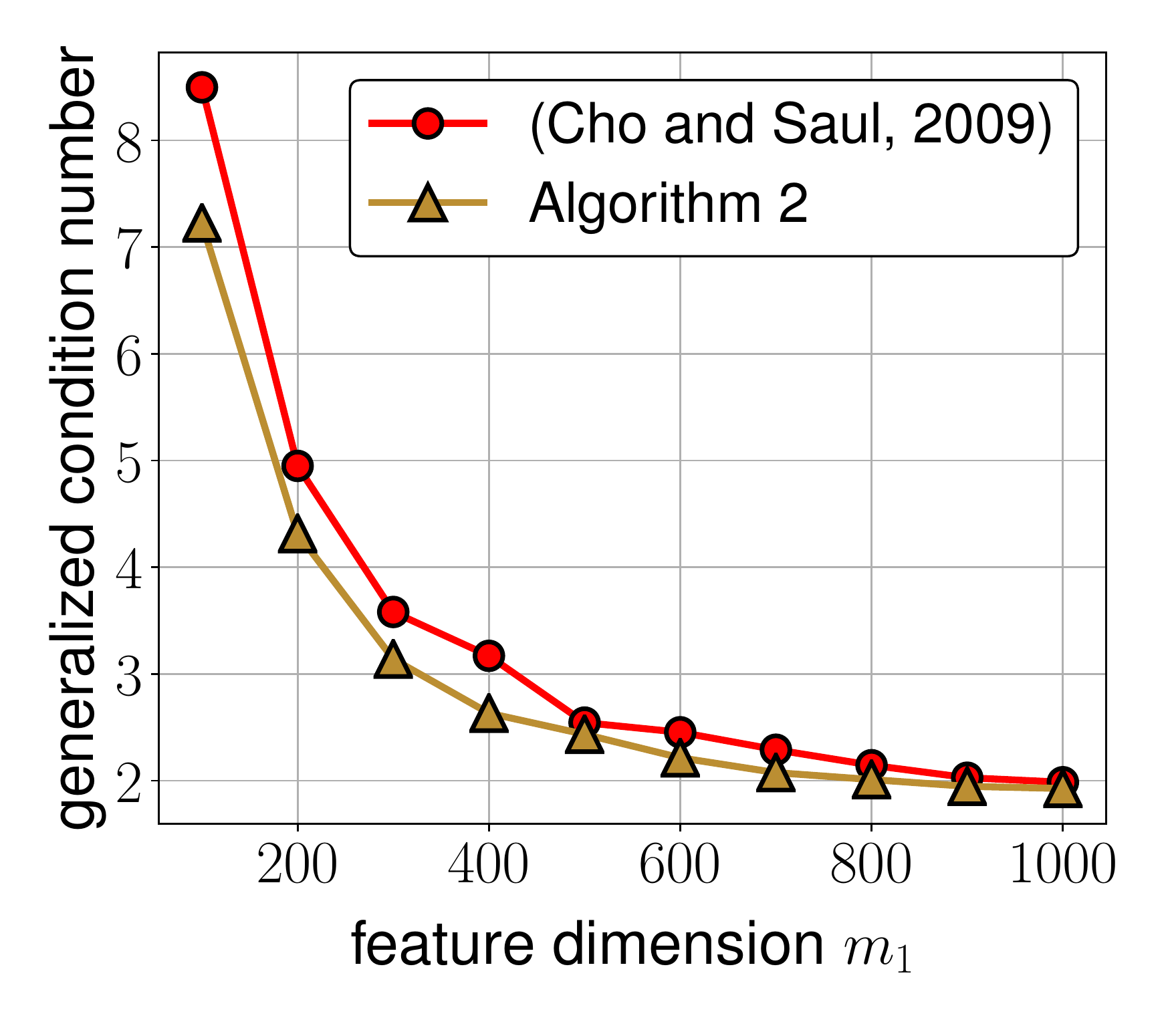}
        \vspace{-0.25in}
        \caption{$\mathtt{abalone}$}
    \end{subfigure}
    \hspace{-0.05in}
    \begin{subfigure}{0.3\textwidth}
        \includegraphics[width=\columnwidth]{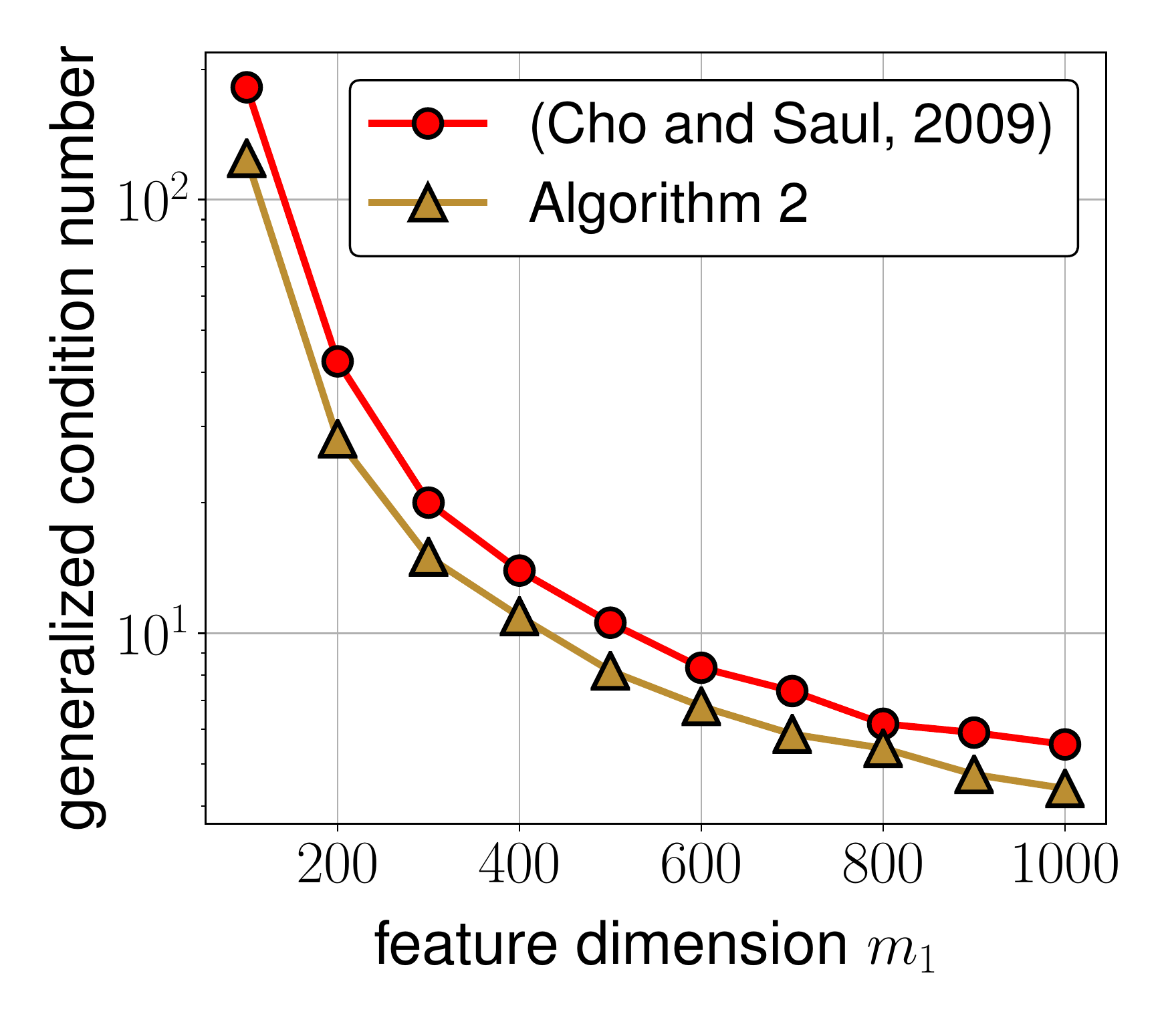}
        \vspace{-0.25in}
        \caption{$\mathtt{ecoli}$}
    \end{subfigure}
    \hspace{-0.05in}
    \begin{subfigure}{0.3\textwidth}
        \includegraphics[width=\columnwidth]{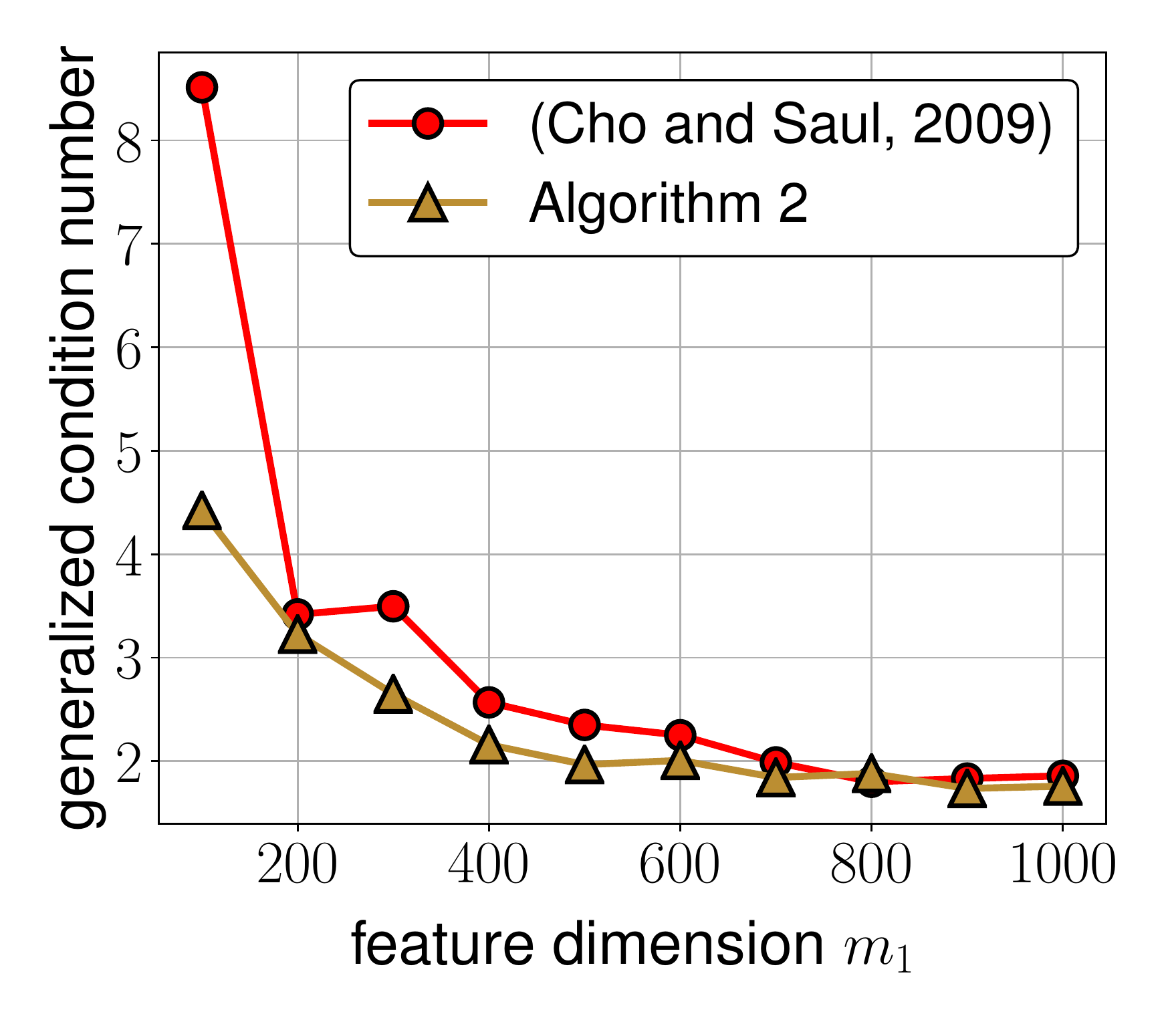}
        \vspace{-0.25in}
        \caption{$\mathtt{titanic}$}
    \end{subfigure}
    \vspace{-0.1in}
% \begin{center}
% \includegraphics[width=0.7\columnwidth]{figures/fig_cond_num_abalone_yeast_a1.pdf} \\
% \end{center}
% \begin{center}
% \begin{subfigure}{0.35\textwidth}
% \vspace{-0.45in}
% \centering
% \caption{$\mathtt{abalone}$}
% \end{subfigure}
% \begin{subfigure}{0.33\textwidth}
% \vspace{-0.45in}
% \centering
% \caption{$\mathtt{yeast}$}
% \end{subfigure}
% \vspace{-0.25in}
\caption{Generalized condition numbers of arc-cosine random features of \cite{cho2009kernel} and the proposed approach (\cref{alg:gibbs}) under real-world UCI datasets.} \label{fig:a1_cond_num}
\vspace{-0.2in}
\end{center}
\end{figure}

% To validate , 
We evaluate approximation quality of the proposed method (\cref{alg:gibbs}) to that of the random features~\cite{cho2009kernel} in \cref{fig:a1_cond_num}. In particular, we compute the condition number (i.e., ratio of the largest and smallest eigenvalues) of 
%\begin{align}
$
(\A_1  + \lambda \I)^{-1/2} \left( \BPhi_1 \BPhi_1^\top + \lambda \I\right) (\A_1  + \lambda \I)^{-1/2}.
$
%\end{align}
If $(\BPhi_1 \BPhi_1^\top + \lambda\I)$ is spectrally close to $(\A_1  + \lambda \I)$, then the corresponding condition number will be close to $1$. We evaluate the condition numbers of those random features using 3 UCI datasets and set $\lambda = 10^{-4} \cdot n$ when $n$ data points are given. For each dataset, we increase $m_1$ from $100$ to $1{,}000$. Observe that the proposed random features for arc-cosine features have smaller condition numbers than the previous method for all datasets. We provide more experimental results that the modified random features can improve performance on downstream tasks in \cref{sec:experiments}.

\subsection{Spectral Approximation for the NTK of Two-layer ReLU Network}

We are now ready to state a spectral approximation bound for our NTK random features of a two-layer ReLU network, i.e., $L=1$.

\begin{restatable}{theorem}{ntkspectral} \label{thm:ntk_spectral}
Given $\X = [\x_1, \dots, \x_n]^\top \in \R^{n \times d}$, assume that $\norm{\x_i}_2 = 1$ for $i\in [n]$.
Let $\K_{\NTK}$ be the NTK of two-layer ReLU network, i.e., $L=1$ in \cref{eq:nn-definition}, and $\A_0, \A_1$ denote the arc-cosine kernels of $0$-th, $1$-st order with $\X$, respectively, as in \cref{def:arccos_kernels}. For any $\lambda \in (0, 2\min(\norm{\A_0}_2, \norm{\A_1}_2)]$, suppose $s_{\lambda}$ is an upper bound of statistical dimensions of both $\A_0, \A_1$. Given $\varepsilon \in (0, 1/2)$, $\delta \in (0,1)$, 
let $\BPhi \in \R^{n \times (m_1 + m_{\mathtt{cs}})}$ be the first output of \cref{alg:ntk_random_features} with $L=1$ and
$$m_0 \geq \frac{48 n}{\varepsilon^{2} \lambda} \log\left( \frac{48 s_{\lambda}}{\delta} \right), \ \ 
m_1 \geq \frac{16}{3} \frac{d \norm{\X}_2^2}{\varepsilon^{2}\lambda} \log\left( \frac{48 s_{\lambda}}{\delta} \right), \ \ 
m_{\mathtt{cs}}\geq \frac{297}{ \varepsilon^2 \delta} \left( \frac{n}{\lambda + 1}\right)^2$$
% \begin{itemize}
% \item $m_0 \geq \frac{48 n}{\varepsilon^{2} \lambda} \log\left( \frac{48 s_{\lambda}}{\delta} \right)$,
% \item $m_1 \geq \frac{16}{3} \frac{d \norm{\X}_2^2}{\varepsilon^{2}\lambda} \log\left( \frac{48 s_{\lambda}}{\delta} \right)$,
% \item $m_{\mathtt{cs}} \geq \frac{297}{ \varepsilon^2 \delta} \left( \frac{n}{\lambda + 1}\right)^2$.
% \end{itemize}
% Then, $\BPhi \BPhi^\top$ is a $(\varepsilon, \lambda)$-spectral approximation of $\K^{\NTK}$
Then, with probability at least $1-\delta$, it holds that
% it holds that
\begin{align} \label{eq:spectral_approximation_ntk_features}
(1-\varepsilon) \left( \K_{\NTK} + \lambda \I\right) \preceq
\BPhi \BPhi^\top + \lambda \I \preceq (1+\varepsilon) \left( \K_{\NTK} + \lambda \I\right).
% \frac{\BPhi \BPhi^\top + \lambda \I}{1+\varepsilon}
% \preceq
% \K^{\NTK} + \lambda \I
% \preceq
% \frac{\BPhi \BPhi^\top + \lambda \I}{1-\varepsilon}.
% (1-\varepsilon) \left( \K^{\NTK} + \lambda \I\right) \preceq
\end{align}
\end{restatable}

The proof of \cref{thm:ntk_spectral} is provided in \cref{sec:proof_ntk_spectral}.
%\cref{sec:proof_ntk_spectral}. 
We note that the ridge regularization parameter typically set to $\lambda = \lambda^\prime n$ where $\lambda^\prime$ is a small constant, e.g., $10^{-4}$~\cite{rudi2016generalization, avron2017random, geifman2020similarity}. 
Combining this setting with the fact that $\norm{\x_i}_2=1$ yields that $\s_{\lambda} = \bigo(1)$. Hence, it is enough to choose $m_0 = {\bigo}\left(\frac{1}{\varepsilon^{2}}\log\left(\frac{1}{\delta}\right)\right), m_{\mathtt{cs}} = \bigo\left(\frac{1}{\varepsilon^{2}\delta}\right)$ and $m_1 = \bigo\left( \frac{d}{\varepsilon^2} \log\left(\frac{1}{\delta}\right) \right)$ to achieve the spectral approximation in \cref{eq:spectral_approximation_ntk_features} since $\norm{\X}_2^2 \leq \norm{\X}_F^2 = n$.  This leads us to $\bigo\left( \frac{d}{\varepsilon^2} \log\left(\frac{1}{\delta}\right) + \frac1{\varepsilon^2 \delta}\right)$ feature dimension which is nearly linear in the input dimension $d$.
% As discussed in \cref{thm:ntk-random-feature-entry-error}, the dimension of $\BPhi$ is 
% can be generated in $\widetilde{\bigo}(n d \varepsilon^{-2})$ time by assuming that $\norm{\X}_2 = \bigo(\sqrt{n/d})$ for $X \in \R^{n\times d}$.
% , smaller number of feature to approximate $\A_1$ suffices to achieve similar approximation quality of $\A_0$..

% \textcolor{red}{TODO: add more remarks on Theorem 3.}
% {\bf Hardness for Deeper Networks.} 
The current proof technique cannot be used to generalize the result in \cref{thm:ntk_spectral} to deeper networks (i.e., $L\geq 2$). For the proof to work, one needs a monotone property of arc-cosine kernels, i.e., $F_1(\X) \preceq F_1(\Y)$ for $\X \preceq \Y$. However, this property does not hold in general. Thus, we leave the extension to deeper networks to future work.

\section{Experiments} \label{sec:experiments}

\begin{figure}[t]
	\centering
	\begin{subfigure}{0.3\textwidth}
		\includegraphics[width=\textwidth]{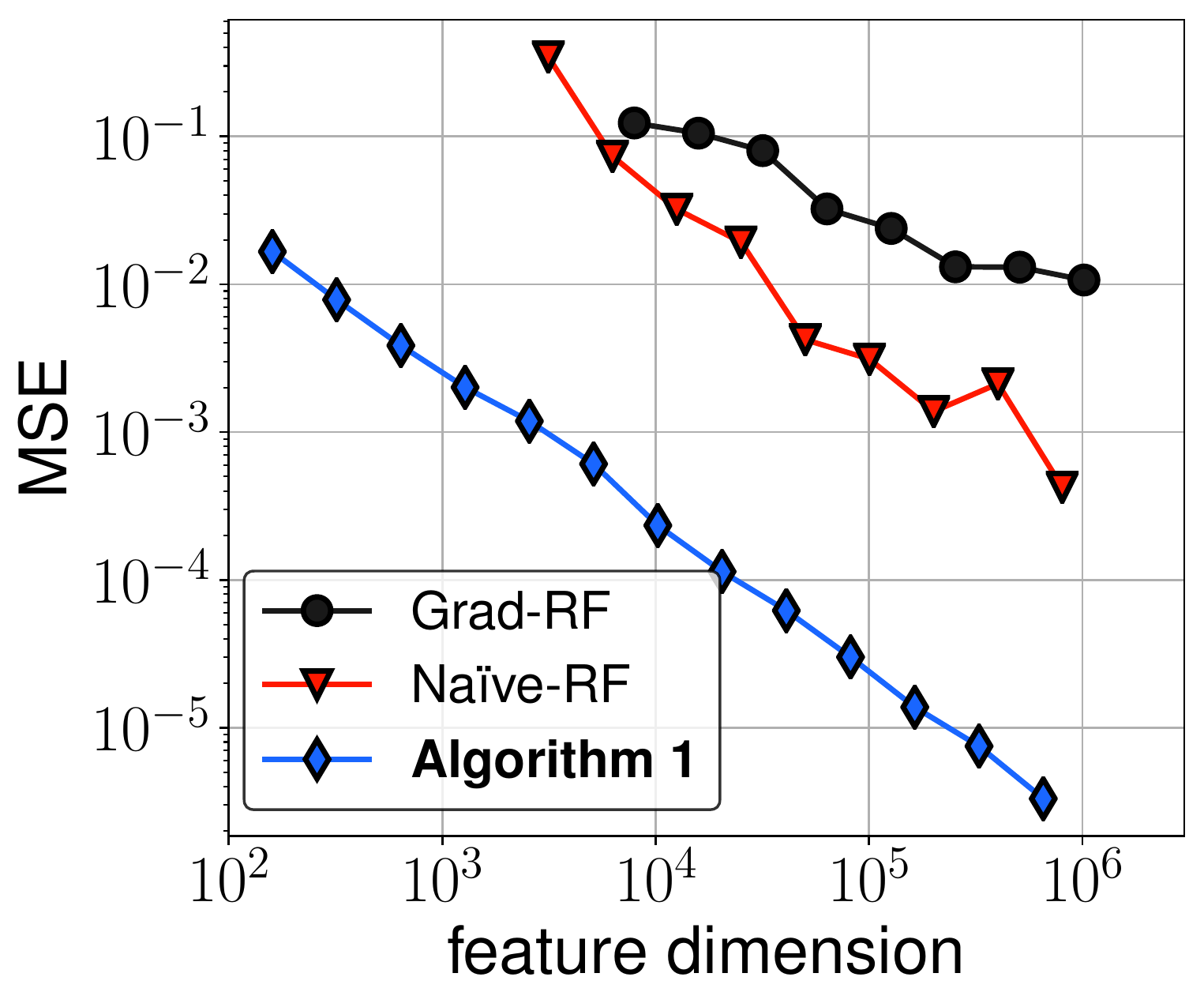}
    % 	\vskip -0.08in
		\caption{$L=1$}
	\end{subfigure}
%	\hspace{-0.1in}
	\begin{subfigure}{0.3\textwidth}
		\includegraphics[width=\textwidth]{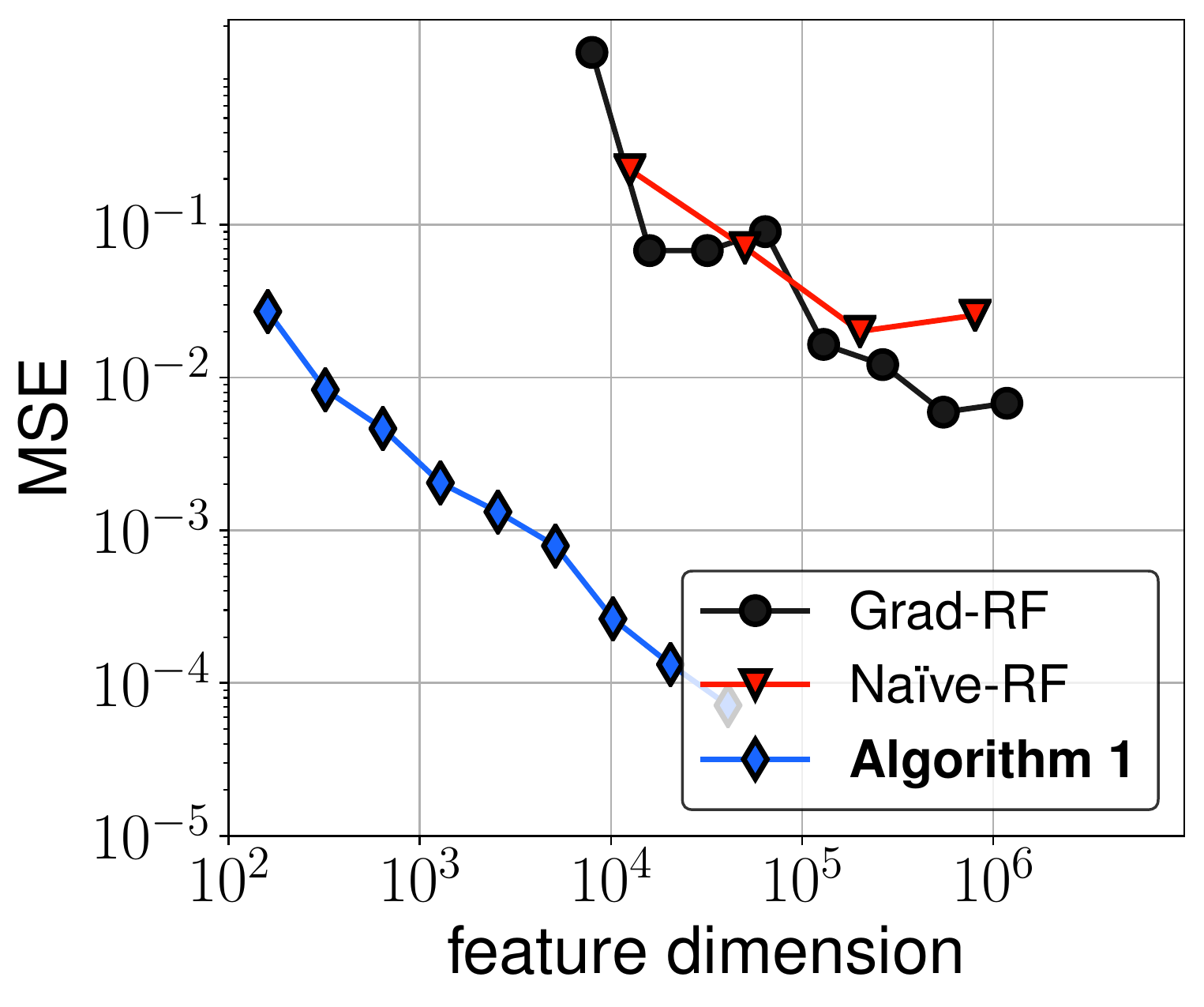}
    % 	\vskip -0.08in
		\caption{$L=2$}
	\end{subfigure}
	\begin{subfigure}{0.3\textwidth}
		\includegraphics[width=\textwidth]{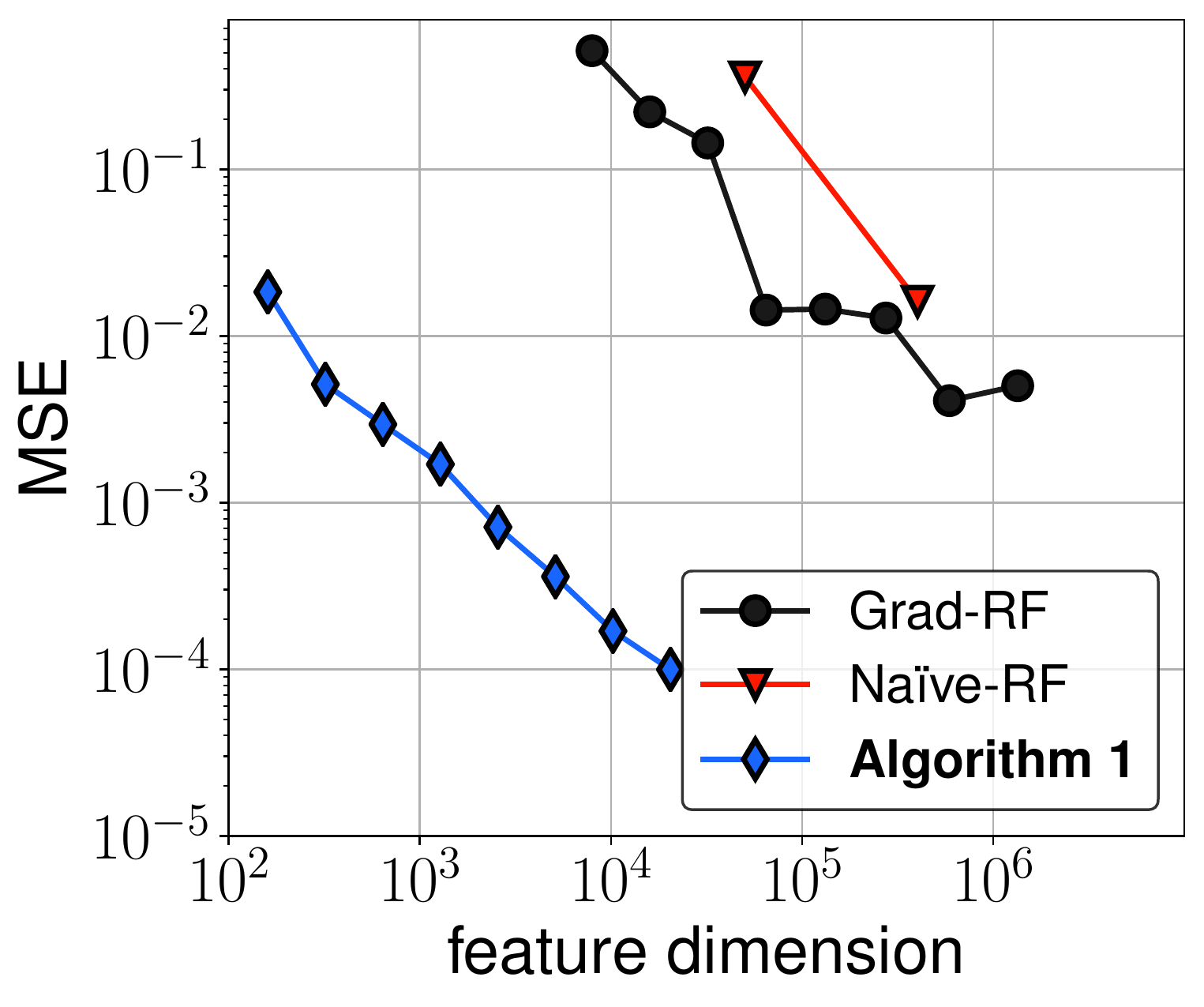}
    % 	\vskip -0.08in
		\caption{$L=3$}
	\end{subfigure}
	\vspace{-0.1in}
	\caption{Mean squared error of entries in NTK matrix approximated by (1) gradient of randomly initialized networks (Grad-RF), (2) the naive feature map construction in~\cite{bietti2019inductive} (Na\"ive-RF) and our algorithm based on a sketching method under subsampled MNIST dataset.} \label{fig:mse_vs_featdim}
	\vspace{-0.15in}
\end{figure}

In this section, we provide experimental results of our method on kernel approximation and various kernel learning tasks including classification, regression and active learning. 
% We refer to the proposed algorithm with modified random features (\cref{alg:ntk_random_features}) as the NTK Random Features (LS).
% and that without Gibbs sampling is said to be simply NTK Random Features.
% on both small-scale and large-scale datasets to test the ability of generalization on unseen examples. 
% addressed the effectiveness of our algorithm on 
% We benchmark the proposed algorithm on both classification and regression tasks.
% on small-scale UCI datasets, (2) transfer learning on image datasets and (3) computational speed on large-scale UCI regression tasks. 

\subsection{Kernel Approximation on MNIST Dataset} \label{sec:kernel-approx}

We first explore \cref{alg:ntk_random_features} for approximating the NTK matrices. We compare to gradient-based NTK random features~\cite{arora2019exact} (Grad-RF) and the na\"ive random features without sketching~\cite{bietti2019inductive} (Na\"ive-RF) as baseline methods. To compute the exact NTK, we randomly choose $n=10{,}000$ data samples from MNIST dataset and evaluate the mean squared error (MSE) of all approximate entries in NTK. We use the ReLU network with depths $L = 1, 2, 4$. 
For Grad-RF, 
we use an implementation proposed by~\citet{novak2020neural}.\footnote{\url{https://github.com/google/neural-tangents}}
In particular, it returns an approximate NTK matrix rather than random features because the dimension of features can be larger than $n$ which loses the computational gain of random features.
For example, gradient of a two-layer, 16-width ReLU network for MNIST has $127{,}040$ dimension. We vary the network width in $\{2^1 \dots, 2^7\}$. For fair dimension comparisons, we report the expected dimension of Grad-RF.  
For Na\"ive-RF, we set $m_0 = m_1 $ and $m_0 \in \{ 2^1, \dots, 2^7\}$.
For our method, we set $m_0 = m_1 = m - m_{\mathtt{cs}}$ and $m_\mathtt{cs} \in \{ \frac{m}{10}, \dots, \frac{9m}{10}\}$ for each $m \in \{ 10\cdot 2^4, \dots, 10 \cdot 2^{16}\}$ and report the average MSE of $9$ different values of $m_\mathtt{cs}$. We omit to report the result when memory overflow causes.

In \cref{fig:mse_vs_featdim}, we observe that our random features achieves the lowest MSE for the same dimension compared to other competitors. The Grad-RF is the worst method and this corresponds to observations reported in \citet{arora2019exact}, i.e., gradient features from a finite width network can degrade practical performances.
As the number of layers $L$ increases, the performance gaps between Na\"ive-RF and other methods become large because its dimension grows exponential in $L$.

\subsection{Classification on Small-scale UCI Datasets}

%{\bf Small-scale.} 
Next, we run our algorithm under 90 small-scale UCI classification datasets. The number of data points $n$ varies from $10$ to $5{,}000$.  We choose hyperparameters using validation data and evaluate the test accuracy using 4-fold cross-validation provided in~\citet{fernandez2014we}.\footnote{\url{http://persoal.citius.usc.es/manuel.fernandez.delgado/papers/jmlr/data.tar.gz}} We also consider the following additional metrics used in~\citet{arora2019harnessing, geifman2020similarity}; P90 and P95 are the ratios of datasets where a classifier reaches at least 90\% and 95\% of the maximum accuracy, and PMA (percentage of the maximum accuracy) is the average ratio its accuracy to the maximum among 90 datasets.

We run \cref{alg:ntk_random_features} with and without Gibbs sampling (GS) (i.e., \cref{alg:gibbs}) where the number of Gibbs iteration is set to $T=1$ throughout all experiments. We also execute various classifiers including AdaBoost, random forest, $k$-nearest neighbors and support vector classifier (SVC). For methods running with SVC, we search the cost value $C$ in $\{2^{-19},\dots, 2^{-3}, 10^{-7},  \dots, 10^3\}$ and choose the best one that achieves the best validation accuracy.
%The results of $k$-nearest neighbor, polynomial kernel, AdaBoost and random forest are from benchmark results in~\citet{fernandez2014we}. 
We use the support vector classifier (SVC) for random features methods (ours, RFF).
For methods using SVC, the cost value $C$ is chosen by searching in $\{2^{-19}, \dots, 2^{20} \}$ that achieves the best validation accuracy.  
For our algorithm and RFF, we consider the output dimension $m$ as a hyperparameter. We search $m$ in $\{10, 20, \dots, 100\}$ for datasets with $n \leq 600$ and explore $m$ in $\{20, 40, \dots, 200\}$ for datasets with $n > 600$ that achieves the best validation accuracy. For NTK, the network depth $L$ changes from $1$ to $5$ which is the same setup in \cite{arora2019harnessing, shankar2020neural}.  We also compare test accuracy of fully-connected ReLU network. We explore the network depth in $\{1,2,3,4,5\}$ and width in $\{2^6, \dots, 2^{11}\}$. The ReLU network is trained by Adam optimizer for $100$ epochs with an initial learning rate $0.1$ and cosine annealing is used to schedule learning rate. 
%The test accuracy with the setting that achieves the best validation performance is reported.

In \cref{tab:uci_classification_test_acc}, the average test accuracy with 95\% confidence interval, P90/95 and PMA are reported. Observe that the NTK achieves the best results while the NTK Random Features with GS is the second best. The NTK Random Features performs better than the Random Fourier Features because the NTK is more appropriate choice compared to the RBF kernel. Finally, our method with GS shows higher accuracy than that without GS. 
  
% We verify that the proposed approach with modified random features can achieve the best performance in terms of test accuracy and P90. 
% Interestingly, the random features approaches show better results than the corresponding exact kernel method. This result comes from that random features can yield implicit regularization effect that leads to performance improvement. Among kernel methods, the NTK performs with higher accuracy than others.
% \addtolength{\tabcolsep}{-2.5pt}  
\begin{table}[t]
% \vspace{-0.2in}
\caption{Results of average test accuracy, P90, P95 and PMA (percentage of the maximum accuracy) on 90 UCI classification datasets.  {\bf Bold entries} indicate the best results.} \label{tab:uci_classification_test_acc}
\vspace{-0.05in}
\centering
\scalebox{0.8}{
\begin{tabular}{@{}lccccccc@{}}
\toprule
Method & Test Accuracy (\%) & P90 & P95 & PMA \\ 
\midrule
AdaBoost                     & 76.32 $\pm$ 3.56  &  66.67 & 37.78 & 89.44 \\
Random Forest                & 77.46 $\pm$ 3.75  &  79.31 & 57.47 & 90.84 \\
$k$-Nearest Neighbors        & 76.95 $\pm$ 3.42  &  72.22 & 43.33 & 90.31 \\
Fully-connected ReLU Network & 81.10 $\pm$ 3.11  &  85.56 & 78.89 & 95.33 \\
Polynomial Kernel            & 79.54 $\pm$ 3.41  &  79.31 & 66.67 & 93.31 \\
RBF Kernel                   & 81.79 $\pm$ 2.95  &  91.11 & 75.56 & 95.97 \\
Random Fourier Features      & 81.61 $\pm$ 2.98  &  88.89 & 71.11 & 95.74 \\
\midrule
NTK                          & {\bf 82.24 $\pm$ 2.94}  &  {\bf 92.22} & {\bf 80.00} & {\bf 96.53} \\
\rowcolor{Gray}NTK Random Features          & 81.84 $\pm$ 2.89  &  {\bf 92.22} & 75.56 & 96.16 \\
\rowcolor{Gray}NTK Random Features with GS  & 81.85 $\pm$ 2.98  &  {\bf 92.22} & 75.56 & 96.05 \\
\bottomrule
\end{tabular}
}
% \vspace{-0.1in}
\end{table}
% \addtolength{\tabcolsep}{2.5pt}  

% \addtolength{\tabcolsep}{4pt}  
\begin{table*}[t]
\caption{Results of the mean squared errors (MSE) and wall-clock time (sec) on large-scale UCI regression datasets. We measure the entire time for solving the kernel ridge regression. {\bf Bold entries} indicate the best MSE or time for each dataset. (-) means the Out-of-Memory error.} \label{tab:large-scale-uci}
\vspace{-0.05in}
\centering
\scalebox{0.8}{
\begin{tabular}{@{}lcccccccc@{}}
\toprule
 & \multicolumn{2}{c}{\MSD} & \multicolumn{2}{c}{\WorkLoads} & \multicolumn{2}{c}{\Protein} & \multicolumn{2}{c}{\SuperConduct} \\ 
\midrule
\# of Training Data $n$ & \multicolumn{2}{c}{467{,}315} & \multicolumn{2}{c}{179{,}585}  & \multicolumn{2}{c}{39{,}617} & \multicolumn{2}{c}{19{,}077}\\
% \#Features $d$ & \multicolumn{2}{c}{90} & \multicolumn{2}{c}{7} & \multicolumn{2}{c}{9} & \multicolumn{2}{c}{81} \\
\midrule
                                         & MSE   & Time (s) & MSE & Time (s) & MSE & Time (s) & MSE & Time (s)\\
\midrule
RBF Kernel                               & (-) & (-) & (-) & (-) & 112.82 & 110.2 & 2239.83 & 19.5 \\
Random Fourier Features                  & 108.50 & 159 & 7.05$\times 10^{4}$ & 63.7 & {\bf 81.98} & 14.6 & 1175.13 & {\bf 7.1}\\
\midrule
NTK                                      & (-) & (-) & (-) & (-) & 90.03 & 243 & 513.25 & 51.9 \\
\rowcolor{Gray}{NTK Random Features} & {\bf 80.77} & {\bf 149.7} & 2.43$\times 10^{4}$ & {\bf 50.9} & 90.28 & {\bf 13.9} & 527.58 & 7.9 \\ 
\rowcolor{Gray} NTK Random Features with GS & 80.94 & 168.5 & $\mathbf{2.38\times 10^{4}}$ & 53.0 & 85.99 & 16.1 & {\bf 492.78}  & 12.7 \\
\bottomrule
\end{tabular}
}
\end{table*}
% \addtolength{\tabcolsep}{-4pt}   

\subsection{Regression on Large-scale UCI Datasets} \label{sec:speed}

% {\bf Large-scale.} 
We also demonstrate the computational efficiency of our method using $4$ large-scale UCI regression datasets. In particular, we consider kernel ridge regression (KRR) problem.
For a kernel function $K:\R^d \times \R^d \rightarrow \R$, KRR problem can be formulated as
\begin{align}
    % \left(\K + \lambda n \I\right) \alpha = \y, \quad y_{\mathrm{te}} = K(\x_{\mathrm{te}}, \X) \alpha.
    y_{\mathrm{test}} = K(\x_{\mathrm{test}}, \X) \left(\K + \lambda n \I\right)^{-1}\y.
\end{align}
where $\X = [\x_1, \dots, \x_n]^\top\in \R^{n \times d}$ is training data, $\y \in \R^n$ is training label, $\x_{\mathrm{test}}\in \R^d$ is a test data, $y_{\mathrm{test}}$ is a predicted label, $[\K]_{ij} = K(\x_i,\x_j)$ and $K(\x_{\mathrm{test}},\X) = [K(\x_{\mathrm{test}}, \x_1), \dots, K(\x_{\mathrm{test}}, \x_n)] \in \R^{1 \times n}$. Note that solving the problem can require $\bigo(n^3)$ time complexity due to the matrix inversion in general.  Consider a feature map $\BPhi:\R^d \rightarrow \R^m$ can approximate the kernel $K$ such that $K(\x,\x) \approx \hat{K}(\x,\x) = \inner{\BPhi(\x), \BPhi\x)}$. Then, the computation can be amortized as
% \begin{align*}
% y_{\mathrm{test}} 
%     &= \BPhi(\x_{\mathrm{test}})^\top \BPhi(\X)^\top \left(\BPhi(\X) \BPhi(\X)^\top + \lambda n \I\right)^{-1}\y \\ 
%     &= \BPhi(\x_{\mathrm{test}})^\top \left( \BPhi(\X)^\top \BPhi(\X) + \lambda n \I_m\right)^{-1} \BPhi(\X)^\top \y
% \end{align*}
$ y_{\mathrm{test}}  = \BPhi(\x_{\mathrm{test}})^\top \left( \BPhi(\X)^\top \BPhi(\X) + \lambda n \I_m\right)^{-1} \BPhi(\X)^\top \y$
which requires $\bigo(nm^2 + m^3)$ time to run. When $n \gg m$, this is much efficient than solving the problem with the exact kernel.

% in which random features can reduce the computational cost from $\bigo(n^3)$ to $\bigo(n m)$ where $m$ is the target features dimension.
We compare our methods to NTK, RBF and RFF. For ours and RFF, we choose the output dimension to $m=10{,}000$ for all datasets, which is much smaller than the number of data samples $n$.
% and partition each data into 90\%/10\% for training and test split. 
%We randomly choose 1000 data points for 
%ith small portion of the training data used for searching hyper-parameters. 
In \cref{tab:large-scale-uci}, we report the wall-clock times and mean squared errors (MSE) of test prediction. We face Out-of-Memory errors when running kernel methods using \MSD{} and \WorkLoads{} datasets. Observe that our random features are significantly faster than NTK, e.g., up to $\times$17 speedup for \Protein{} dataset, without performance loss. We also verify that the NTK features achieve lower MSE than RFF only for \Protein{} but it outperforms with a huge gap for the rest of the datasets.

\subsection{Active Learning on MNIST Dataset}

We finally apply the proposed method to {\em active learning} using {\sc MNIST} dataset. 
% In active learning, a set of labeled data 
% with fixed size are given with labels. 
The goal is to select training data of fixed size $k$ that maximizes the performance. % by training newly acquired data. 
Recently, \citet{shoham2020experimental} suggested an active learning strategy based on the NTK. They propose a novel criteria that can be an upper bound of statistical risk for general kernel learning and present an algorithm that greedily minimizes this criteria with the NTK. Their greedy process begins with an empty set and iteratively appends singleton that minimizes the proposed risk bound. It takes $\bigo(n^2 k^2)$ time to obtain $k$ data points which equals to the budget size to acquire labels which can be prohibitive if $n$ is large.

% statistical criteria for selecting unlabeled data points for general kernel learning and present an algorithm that greedily minimizes this criteria with the NTK. The greedy algorithm begins with an empty set and iteratively appends single data point that minimizes the proposed criteria which takes $\bigo(n^2 k^2)$ time to run for the target design size $k$. This can be prohibitive if $n$ is large.

\begin{figure}[h]
\centering
\includegraphics[width=0.4\textwidth]{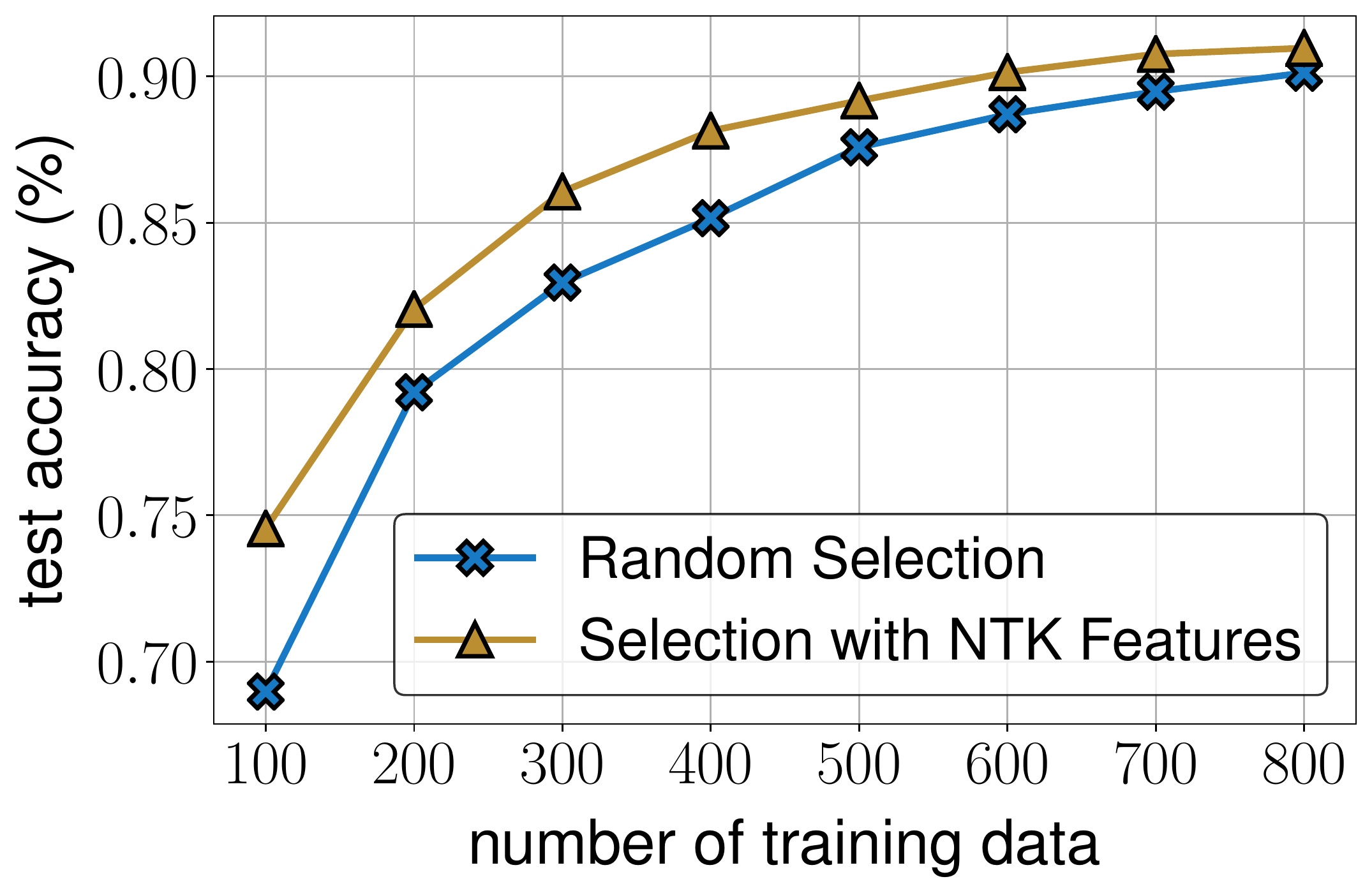}
\vspace{-0.1in}
\caption{Test accuracy (\%) of random selection versus greedy minimization of statistical bound using NTK random features. 
%Learning is conducted using the full network.}
\label{fig:random_features_oed}} 
% \vspace{-0.18in}
\end{figure}

Motivated by this, we apply the proposed NTK random features to their greedy algorithm that can improve the running time. 
Recall that our random features builds an approximation $\K^\prime = \BPhi \BPhi^\top \approx \K$ where $\BPhi \in\R^{n \times m}$. Under certain parameter regimes, the low rank structure of $\K^\prime$ can be used to implement a faster version of the greedy algorithm. Specifically, after $\bigo(nm^2)$ preprocessing,  the cost per iteration of the greedy algorithm can be reduced to $\bigo(nmj^2)$, and the cost of finding the design of size $k$ to $\bigo(n m (k^3 + m))$. We provide more details in the supplementary material.
% \cref{sec:math-active}. 

\cref{fig:random_features_oed} illustrates performance of greedy minimization using our NTK random features compared to randomly generated designs under {\sc MNIST} dataset. We use a $4$-layer fully-connected ReLU network with width $1{,}000$ and the dimension of NTK random features is $m=10{,}000$. 
% To save time, when generating the NTK designs, instead of scanning the entire dataset in each iteration, in each iteration we first sample 2,000 data-points, and then scan only these data-point as possible candidates for extending ${\cal S}$. 
% We measure the test accuracy of five different network initializations (but one design). 
We clearly see that using the random NTK features we can generate much better designs than randomly chosen data points.  It justifies that our random features plays a crucial role for active learning tasks. We expect that the proposed method can be applied to various machine learning applications with remarkable performance and computation gains.

\section{Conclusion}
In this work, we propose an efficient algorithm for generating random features of the Neural Tangent Kernel (NTK). 
We utilize \TS transform combined with the arc-cosine random features with an importance sampling. We also provide spectral approximation bound to the NTK with layer $2$. Our experiments validate the effectiveness of the proposed methods. We believe that our method would be a broad interest both in theoretical and practical domains.

\bibliographystyle{icml2021}
% \setcitestyle{numbers}
\bibliography{references}

\appendix

\newcommand{\KNTK}{\boldsymbol \Theta}
\newcommand{\KNNGP}{\boldsymbol \Sigma}

\section{Proof of Theorems}

\subsection{Proof of \cref{thm:ntk-random-features-error}} \label{sec:proof-ntk-random-features-error}
\randomfeatureserror*

%%%%%%%%%%%%%%%%%%%%%%%%%%%%%%%%%%%%%%%%%%%%%%%%%%%%%%%%%%%%%%%%%%%%%%%%%%%%%%%%%%%%%%%%%%%%%%%%%%%
%%% ////////////////////////////////////////////////////////////////////////////////////////////
{\it Proof of \cref{thm:ntk-random-features-error}.}
For fixed $\x, \x^\prime \in \mathbb{S}^{d-1}$ and $\ell = 0, \dots, L$, we denote the estimate error as
\begin{align*}
\Delta_{\ell} := 
\max_{(\x_1, \x_2) \in \{ (\x,\x^\prime), (\x,\x), (\x^\prime, \x^\prime)\}}
\abs{\inner{\BPhi^{(\ell)}(\x_1), \BPhi^{(\ell)}(\x_2)} - K_{\NTK}^{(\ell)}(\x_1,\x_2)}
\end{align*}
and note that $\Delta_0 = 0$. Recall that
\begin{align}
    K_{\NTK}^{(\ell)}(\x,\x^\prime) = K_{\NNGP}^{(\ell)}(\x,\x^\prime) + \dot{K}_{\NNGP}^{(\ell)}(\x,\x^\prime)  \cdot {K}_{\NTK}^{(\ell-1)}(\x,\x^\prime)
\end{align}
where
\begin{align}
\dot{K}_{\NTK}^{(\ell)}(\x,\x^\prime) &:= 1 - \frac{1}{\pi} \cos^{-1} \left(\frac{K_{\NNGP}^{(\ell-1)}(\x,\x^\prime)}{\sqrt{K_{\NNGP}^{(\ell-1)}(\x,\x) K_{\NNGP}^{(\ell-1)}(\x^\prime,\x^\prime)}} \right) \\
\dot{K}_{\NNGP}^{(\ell)}(\x,\x^\prime) &:= f\left( \frac{K_{\NNGP}^{(\ell-1)}(\x,\x^\prime)}{\sqrt{K_{\NNGP}^{(\ell-1)}(\x,\x^\prime) K_{\NNGP}^{(\ell-1)}(\x^\prime,\x^\prime)}} \right)
\end{align}
and $f(x) = \frac1{\pi}\left( \sqrt{1-x^2} + (\pi- \cos^{-1}(x))x\right)$ for $x \in [-1,1]$. 

We use the recursive relation to approximate:
\begin{align*}
\inner{ \BPhi^{(\ell)}(\x), \BPhi^{(\ell)}(\x^\prime)} &= \inner{ \BPsi^{(\ell)}(\x), \BPsi^{(\ell)}(\x^\prime)} + \inner{ \BGamma^{(\ell)}(\x), \BGamma^{(\ell)}(\x^\prime)}\\
&\approx \inner{ \BPsi^{(\ell)}(\x), \BPsi^{(\ell)}(\x^\prime)} + 
\inner{ \BLambda^{(\ell)}(\x) \otimes \BPhi^{(\ell-1)}(\x), 
\BLambda^{(\ell)}(\x^\prime) \otimes \BPhi^{(\ell-1)}(\x^\prime)
} \\
% \inner{\BLambda^{(\ell)}(\x), \BLambda^{(\ell)}(\x^\prime)} \cdot 
% \inner{\BPhi^{(\ell-1)}(\x), \BPhi^{(\ell-1)}(\x^\prime)} \\ 
&\approx  K_\NNGP^{(\ell)}(\x,\x^\prime) +  \dot{K}_\NNGP^{(\ell)}(\x,\x^\prime) \cdot  K_\NTK^{(\ell-1)}(\x,\x^\prime) = K^{(\ell)}_{\NTK}(\x,\x^\prime).
\end{align*}

For notational simplicity, we define the following events:
\begin{align}
\mathcal{E}_{\BPsi}^{(\ell)} (\x, \x^\prime, \varepsilon) &:= 
	\left\{
		\abs{ \inner{\BPsi^{(\ell)}(\x),\BPsi^{(\ell)}(\x^\prime)} - K_\NNGP^{(\ell)}(\x,\x^\prime) } \leq \varepsilon	
	\right\}, \\
\mathcal{E}_{\BLambda}^{(\ell)} (\x, \x^\prime, \varepsilon) &:= 
    \left\{
        \abs{ \inner{\BLambda^{(\ell)}(\x),\BLambda^{(\ell)}(\x^\prime)} - \dot{K}_{\NNGP}^{(\ell)}(\x,\x^\prime) } \leq \varepsilon	
    \right\}, \\
\mathcal{E}_{\BGamma}^{(\ell)} (\x, \x^\prime, \varepsilon) &:= 
    \left\{
        \abs{\inner{\BGamma^{(\ell)}(\x), \BGamma^{(\ell)}(\x^\prime)} - \inner{\BLambda^{(\ell)}(\x) \otimes \BPhi^{(\ell-1)}(\x), \BLambda^{(\ell)}(\x^\prime) \otimes \BPhi^{(\ell-1)}(\x^\prime)}} \leq \varepsilon
    \right\}
\end{align}
and $\mathcal{E}_{\boldsymbol \Omega}^{(\ell)} (\varepsilon) := \mathcal{E}_{\boldsymbol \Omega}^{(\ell)} (\x, \x, \varepsilon) \bigcap \mathcal{E}_{\boldsymbol \Omega}^{(\ell)} (\x, \x^\prime, \varepsilon) \bigcap \mathcal{E}_{\boldsymbol \Omega}^{(\ell)} (\x^\prime, \x^\prime, \varepsilon)$ for ${\boldsymbol \Omega} = \{ \BPsi, \BLambda, \BGamma\}$.  Our proof is based on the following claims:

\begin{restatable}{claim}{claimone} \label{claim1}
There exists a constant $C_1 > 0$ such that if $m_1 \geq C_1 \frac{L^6}{\varepsilon^4} \log\left( \frac{L}{\delta}\right)$ then 
\begin{align*}
	\Pr\left( 
	\mathcal{E}_{\BPsi}^{(\ell)} \left( \frac{\varepsilon^2}{32L^2} \right)
	\right) \geq 1 - \frac{\delta}{3L}.
\end{align*}
\end{restatable}
The proof of \cref{claim1} is provided in \cref{sec:proof-claim1}.

\begin{restatable}{claim}{claimtwo} \label{claim2}
There exists a constant $C_0 > 0$ such that if $m_0 \geq C_0 \frac{L^2}{\varepsilon^2}  \log\left( \frac{L}{\delta} \right)$ then
\begin{align*}
	\Pr \left(
	\mathcal{E}_{\BLambda}^{(\ell)} \left( \frac{3\varepsilon}{8L} \right)
	\bigg|
	\mathcal{E}_{\BPsi}^{(\ell)} \left( \frac{\varepsilon^2}{32L^2} \right)
	\right) \geq 1 - \frac{\delta}{3L}.
\end{align*}
\end{restatable}
The proof of \cref{claim2} is provided in \cref{sec:proof-claim2}.

\begin{restatable}{claim}{claimthree} \label{claim3}
There exists a constant $C_2 > 0$ such that if $m_{\mathtt{cs}}\geq C_2 \frac{L^3}{\varepsilon^2 \delta}$ then 
\begin{align*}
\Pr \left( \mathcal{E}_{\BGamma}^{(\ell)} \left( \frac{\varepsilon}{8L} \left( \ell + \Delta_{\ell-1} \right)  \right) \bigg| 
\mathcal{E}_{\BLambda}^{(\ell)} \left( \frac{3\varepsilon}{8L} \right) \bigcap
\mathcal{E}_{\BPsi}^{(\ell)} \left( \frac{\varepsilon^2}{32L^2} \right)
\right)\geq 1 - \frac{\delta}{3L}.
\end{align*}
\end{restatable}
The proof of \cref{claim3} is provided in \cref{sec:proof-claim3}.

Combining \cref{claim1}, \cref{claim2} and \cref{claim3}, we have
% with probabilty at least $1 - \frac{\delta}{L}$ it holds that
\begin{align*}
\Pr \left(
	\mathcal{E}_{\BGamma}^{(\ell)} \left(  \frac{\varepsilon}{8L} \left( \ell + \Delta_\ell \right)  \right)
	~\bigcap~
	\mathcal{E}_{\BLambda}^{(\ell)} \left( \frac{3\varepsilon}{8L} \right)
	~\bigcap~
	\mathcal{E}_{\BPsi}^{(\ell)} \left( \frac{\varepsilon^2}{32L^2} \right)
\right) \geq 1 - \frac{\delta}{L}.
\end{align*}

Next, we claim that the above event can provide the recurrence relation of $\Delta_\ell$ as described in below.

\begin{restatable}{claim}{claimfour} \label{claim4}
For $\varepsilon \in (0, 1)$, $L \geq 1$, if event 
\begin{align*}
	\mathcal{E}_{\BGamma}^{(\ell)} \left(  \frac{\varepsilon}{8L} \left( \ell + \Delta_\ell \right)  \right)
	~\bigcap~
	\mathcal{E}_{\BLambda}^{(\ell)} \left( \frac{3\varepsilon}{8L} \right)
	~\bigcap~
	\mathcal{E}_{\BPsi}^{(\ell)} \left( \frac{\varepsilon^2}{32L^2} \right)
\end{align*}
holds for $\ell \in [L]$ then 
\begin{align} \label{eq:delta-recur}
	\Delta_\ell \leq \left(1 + \frac{\varepsilon}{2L} \right) \Delta_{\ell-1} + \frac{\varepsilon}{2L} \ell + \frac{\varepsilon^2}{32L^2}.
\end{align}
\end{restatable}
The proof of \cref{claim4} is provided in \cref{sec:proof-claim4}.

Applying union bound on \cref{eq:delta-recur} for all $\ell \in [L]$ and solving the recurrence, we obtain that with probabilty at least $1 - \delta$
\begin{align}
\Delta_L \leq \left( 1 + \frac{\varepsilon}{2L} + \frac{\varepsilon^2}{32L^2}\right) \frac{\left(1+\frac{\varepsilon}{2L}\right)^L - 1}{\frac{\varepsilon}{2L}} - L.
\end{align}

When $L=1$, the statement in \cref{thm:ntk-random-features-error} holds since $\Delta_1 \leq \frac{\varepsilon}{2} + \frac{\varepsilon^2}{32} \leq \varepsilon(1 + \frac{\varepsilon}{2})^2 + \varepsilon$ for $\varepsilon \in (0,1)$. Assume that $L \geq 2$, we obtain 
\begin{align}
\Delta_L 
&\leq \left( 1 + \frac{\varepsilon}{2L} + \frac{\varepsilon^2}{32L^2}\right) L \left( 1 + \frac{\varepsilon}{2L} \right)^{L-1} - L\\
&\leq L \left(  \left( 1 + \frac{\varepsilon}{L} \right) \left(1 + \frac{\varepsilon}{2L} \right)^{L-1} - 1 \right) \\
&\leq L \left( \left(1 + \frac{\varepsilon}{L} \right) \left( \frac{\varepsilon}{2L} (L-1) \left( 1 + \frac{\varepsilon}{2L}\right)^{L-2} + 1 \right) - 1 \right) \\
&= L  \left(1 + \frac{\varepsilon}{L} \right) \frac{\varepsilon}{2L} (L-1) \left( 1 + \frac{\varepsilon}{2L}\right)^{L-2} + \varepsilon \\
&\leq L \left( 1 + \frac{\varepsilon}{2} \right)  \frac{\varepsilon}{2} \left( 1 + \varepsilon\right) + \varepsilon \\
&\leq L \varepsilon \left( 1 + \frac{\varepsilon}{2} \right)^2 + \varepsilon
\end{align}
where inequalities in the first and third line are from the fact that $\frac{(1+x)^k - 1}{x} \leq k x (1+x)^{k-1}$ for $x > 0, k \geq 1$, the fifth line follows from that 
\begin{align*}
\left( 1 + \frac{\varepsilon}{2L}\right)^{L-2} \leq \exp\left( \frac{\varepsilon }{2L}(L-2) \right) \leq \exp\left( \frac{\varepsilon}{2} \right) \leq 1 + \varepsilon.
\end{align*}
%and the last one follows from $\varepsilon \in (0,1)$. 
Hence, we conclude that 
\begin{align}
	\Pr \left( 
	\Delta_L \leq L \varepsilon \left( 1 + \frac{\varepsilon}{2} \right)^2 + \varepsilon
	\right) \geq 1 - \delta.
\end{align}
This completes the proof of \cref{thm:ntk-random-features-error}.
\qed
%%% \\\\\\\\\\\\\\\\\\\\\\\\\\\\\\\\\\\\\\\\\\\\\\\\\\\\\\\\\\\\\\\\\\\\\\\\\\\\\\\\\\\\\\\\\\\\\

\subsubsection{Proof of \cref{claim1}} \label{sec:proof-claim1}
\claimone*

{\it Proof of \cref{claim1}.}
Recall that $\mathcal{E}_{\BPsi}^{(\ell)}\left(\frac{\varepsilon^2}{32 L^2}\right)$ is equivalent to 
\begin{align*}
	\abs{ \inner{\BPsi^{(\ell)}(\x_1),\BPsi^{(\ell)}(\x_2)} - K_{\NNGP}^{(\ell)}(\x_1,\x_2) } \leq \frac{\varepsilon^2}{32 L^2}
\end{align*}
for $(\x_1, \x_2) \in \{ (\x,\x^\prime), (\x,\x), (\x^\prime, \x^\prime)\}$. The proof is directly followed by below lemma.
\begin{lemma}[Corollary 16 in \citep{daniely2016toward}] \label{lmm:a1-entry-diff}
	% Let S be a skeleton with ReLU activations, and w a random initialization of N(S, r)
	% with r ≥ c1 depth2(S) log( 8|S| δ )�2  
	% For all x, y and � ≤ min(c2, 1depth(S) ), w.p. ≥ 1 − δ,
	% |κw(x, y) − κS(x, y)| ≤ �
	% Here, c1, c2 > 0 are universal constants.
	% Given $\X = [\x_1, \dots, \x_n]^\top \in \R^{n \times d}$, 
	Given $\x, \x^\prime \in \mathbb{R}^{d}$ such that $\norm{\x}=\norm{\x^\prime}=1$, consider a ReLU network with $L$ layers. For $\delta, \varepsilon \in (0,1)$, there exist constants $C_1, C_2 > 0$ such that $m_1 \geq C_1 \frac{L^2}{\varepsilon_1^2} \log\left( \frac{ L}{\delta}\right)$, $\varepsilon_1 \leq \min(c_2, \frac1L)$ and 
	for $(\x_1, \x_2) \in \{ (\x,\x^\prime), (\x,\x), (\x^\prime, \x^\prime)\}$ and $\ell \in [L]$ 
	it holds that with probability at least $1-\delta$ 
	\begin{align}
		\abs{
			% \BPsi^{(\ell)} (\BPsi^{(\ell)})^\top - \KNNGP^{(\ell)}
			\inner{\BPsi^{(\ell)}(\x_1), \BPsi^{(\ell)}(\x_2)} - K_{\NNGP}^{(\ell)}(\x_1, \x_2)
		} \leq \varepsilon_1.
	\end{align}
\end{lemma}
In \cref{lmm:a1-entry-diff}, setting $\varepsilon_1 = \frac{\varepsilon^2}{32L^2}$ for $\varepsilon \in (0,1)$ provides the result. This completes the proof of \cref{claim1}. \qed

\subsubsection{Proof of \cref{claim2}} \label{sec:proof-claim2}
\claimtwo*

{\it Proof of \cref{claim2}.} Recall that 
% $\mathcal{E}_{\BLambda}^{(\ell)}\left(\frac{3\varepsilon}{8 L}\right)$ is equivalent to 
\begin{align*}
&\mathcal{E}_{\BLambda}^{(\ell)}\left(\frac{3\varepsilon}{8 L}\right) = \left\{
	\abs{ \inner{\BLambda^{(\ell)}(\x_1),\BLambda^{(\ell)}(\x_2)} - \dot{K}_{\NNGP}^{(\ell)}(\x_1,\x_2) } \leq \frac{3\varepsilon}{8 L}
\right\} \\
&\mathcal{E}_{\BPsi}^{(\ell)}\left(\frac{\varepsilon^2}{32 L^2}\right) = \left\{
	\abs{ \inner{\BPsi^{(\ell)}(\x_1),\BPsi^{(\ell)}(\x_2)} - K_{\NNGP}^{(\ell)}(\x_1,\x_2) } \leq \frac{\varepsilon^2}{32 L^2}
\right\}
\end{align*}
for $(\x_1, \x_2) \in \{ (\x,\x^\prime), (\x,\x), (\x^\prime, \x^\prime)\}$.  The proof is a direct consequence of the following lemma.
\begin{lemma}[Lemma E.5 in \cite{arora2019exact}] \label{lmm:a0-entry-diff}
	Given $\x, \x^\prime \in \R^d$, $\ell \in [L]$ and $\varepsilon_2 \in (0,1)$, assume that
	\begin{align}
		\abs{ \inner{ \BPsi^{(\ell)}(\x), \BPsi^{(\ell)}(\x^\prime)} - K_{\NNGP}^{(\ell)}(\x,\x^\prime)} \leq \frac{\varepsilon_2^2}{2}.
	\end{align}
	Then, it holds that with probability at least $1-\delta_2$
	\begin{align}
		\abs{ \inner{\BLambda^{(\ell)}(\x), \BLambda^{(\ell)}(\x^\prime)} - \dot{K}_{\NNGP}^{(\ell)}(\x,\x^\prime) } \leq \varepsilon_2 + \sqrt{\frac{2}{m_0} \log\left( \frac{6}{\delta_2}\right)}.
	\end{align}
\end{lemma}
In \cref{lmm:a0-entry-diff}, we choose $\varepsilon_2 = \frac{\varepsilon}{4L}, \delta_2 = \frac{\delta}{3L}$ and $m_0 \geq \frac{128L^2}{\varepsilon^2} \log \left(\frac{18L}{\delta} \right)$ for $\varepsilon, \delta \in(0,1)$ to obtain that
\begin{align*}
\Pr \left( 
\abs{ \inner{\BLambda^{(\ell)}(\x), \BLambda^{(\ell)}(\x^\prime)} - \dot{K}_{\NNGP}^{(\ell)}(\x,\x^\prime) } \leq \frac{3 \varepsilon}{8L}
~\bigg|~
\abs{ \inner{ \BPsi^{(\ell)}(\x), \BPsi^{(\ell)}(\x^\prime)} - K_{\NNGP}^{(\ell)}(\x,\x^\prime)} \leq \frac{\varepsilon^2}{32L^2}
\right) \geq 1- \frac{\delta}{3L}.
\end{align*}
This completes the proof of \cref{claim2}.
\qed

\subsubsection{Proof of \cref{claim3}} \label{sec:proof-claim3}
\claimthree*
{\it Proof of \cref{claim3}.}
Recall that $\mathcal{E}_{\BGamma}^{(\ell)}\left( \frac{\varepsilon}{8L}(\ell + \Delta_{\ell-1}) \right)$ is equivalent to 
\begin{align*}
\abs{\inner{\BGamma^{(\ell)}(\x_1), \BGamma^{(\ell)}(\x_2)} - \inner{\BLambda^{(\ell)}(\x_1) \otimes \BPhi^{(\ell-1)}(\x_1), \BLambda^{(\ell)}(\x_2) \otimes \BPhi^{(\ell-1)}(\x_2)}} \leq  \frac{\varepsilon}{8L}(\ell + \Delta_{\ell-1})
\end{align*}
for $(\x_1, \x_2) \in \{ (\x,\x^\prime), (\x,\x), (\x^\prime, \x^\prime)\}$.  The proof is based on the following lemma that provides an upper bound on variance of the \textsc{CountSketch} transform.

\begin{restatable}{lemma}{cserrorbound} \label{lmm:ts-entry-diff}
	Given $\x, \x^\prime \in \R^{m}$ and $\y,\y^\prime \in \R^{m^\prime}$, let $\mathcal{C}_1: \R^{m} \rightarrow \R^{m_{\mathtt{cs}}}, \mathcal{C}_2: \R^{m^\prime} \rightarrow \R^{m_{\mathtt{cs}}}$ be two independent \CS transforms for some $m_{\mathtt{cs}}>0$. Denote
	\begin{align}
		\BGamma := \FFT^{-1} \left( \FFT\left(\mathcal{C}_1(\x)\right) \odot \FFT\left(\mathcal{C}_2(\y)\right) \right), \quad 
		\BGamma^\prime := \FFT^{-1} \left( \FFT\left(\mathcal{C}_1(\x^\prime)\right) \odot \FFT\left(\mathcal{C}_2(\y^\prime)\right) \right).
	\end{align}
	Then, it holds that with probability at least $1-\delta_3$
	\begin{align}
		\abs{\inner{\BGamma, \BGamma^\prime} - \inner{ \x \otimes \y,  \x^\prime \otimes \y^\prime}} \leq \sqrt{\frac{11}{\delta_3 m_{\mathtt{cs}}}} \norm{\x}_2 \norm{\x^\prime}_2 \norm{\y}_2 \norm{\y^\prime}_2.
	\end{align}
\end{restatable}
The proof of \cref{lmm:ts-entry-diff} is provided in \cref{sec:proof-ts-entry-diff}.  In \cref{lmm:ts-entry-diff}, we choose $\delta_3 = \frac{\delta}{3L}$, $ m_{\mathtt{cs}} \geq \frac{33L(8L + 3\varepsilon)^2}{\varepsilon^2 \delta}$ for $\varepsilon, \delta \in (0,1)$ to satisfies that
\begin{align}
\sqrt{\frac{11}{\delta_3 m_{\mathtt{cs}}}} \leq \frac{\frac{\varepsilon}{8L} }{1 + \frac{3\varepsilon}{8L}}. 
\end{align}
Then, with probability at least $1- \frac{\delta}{3L}$ we have
\begin{align}
	&\abs{\inner{\BGamma^{(\ell)}(\x_1), \BGamma^{(\ell)}(\x_2)} - \inner{\BLambda^{(\ell)}(\x_1) \otimes \BPhi^{(\ell-1)}(\x_1), \BLambda^{(\ell)}(\x_2) \otimes \BPhi^{(\ell-1)}(\x_2)}} \\
	&\leq \frac{\frac{\varepsilon}{8L} }{1 + \frac{3\varepsilon}{8L}} \norm{\BLambda^{(\ell)}(\x_1)}_2 \norm{\BLambda^{(\ell)}(\x_2)}_2 \norm{\BPhi^{(\ell-1)}(\x_1)}_2 \norm{\BPhi^{(\ell-1)}(\x_2)}_2 \\
	&\leq \frac{\frac{\varepsilon}{8L} }{1 + \frac{3\varepsilon}{8L}} \left( 1 + \frac{3 \varepsilon}{8L} \right) \norm{\BPhi^{(\ell-1)}(\x_1)}_2 \norm{\BPhi^{(\ell-1)}(\x_2)}_2 \\
	&\leq \frac{\frac{\varepsilon}{8L} }{1 + \frac{3\varepsilon}{8L}} \left( 1 + \frac{3 \varepsilon}{8L} \right) \left( \ell + \Delta_{\ell-1} \right)  = \frac{\varepsilon}{8L} \left(\ell + \Delta_{\ell - 1}\right)
\end{align}
where the second inequlity holds from the fact that $\mathcal{E}_{\BLambda}^{(\ell)}\left( \frac{3 \varepsilon}{8L} \right)$ implies that for $\x^\prime \in \{ \x_1, \x_2\}$
\begin{align}
	\norm{\BLambda^{(\ell)}(\x^\prime)}_2^2 \leq \dot{K}_{\NNGP}^{(\ell-1)}(\x^\prime,\x^\prime) + \frac{3 \varepsilon}{8 L} = 1 + \frac{3 \varepsilon}{8 L}
\end{align}
since $\dot{K}_{\NNGP}^{(\ell-1)}(\x^\prime,\x^\prime)=1$ and the third one follows from that 
$K_{\NTK}^{(\ell-1)}(\x,\y) \leq \ell$ for $\x,\y \in \mathbb{S}^{d-1}, \ell \geq 1$ and
\begin{align} \label{eq:norm-psi}
	\norm{\BPhi^{(\ell-1)}(\x^\prime)}_2^2 \leq K_{\NTK}^{(\ell-1)}(\x^\prime,\x^\prime) + \Delta_{\ell-1} \leq \ell + \Delta_{\ell-1}.
\end{align}
This completes the proof of \cref{claim3}.
\qed

\subsubsection{Proof of \cref{claim4}} \label{sec:proof-claim4}
\claimfour*
{\it Proof of \cref{claim4}.}
Recall that 
\begin{align*}
    \Delta_{\ell} := 
    \max_{(\x_1, \x_2) \in \{ (\x,\x^\prime), (\x,\x), (\x^\prime, \x^\prime)\}}
    \abs{\inner{\BPhi^{(\ell)}(\x_1), \BPhi^{(\ell)}(\x_2)} - K_{\NTK}^{(\ell)}(\x_1,\x_2)}.
\end{align*}
Observe that the estimate error $\Delta_{\ell}$ can be decomposed into three parts:
\begin{align} \label{eq:three-part}
&\abs{ \inner{\BPhi^{(\ell)}(\x_1),\BPhi^{(\ell)}(\x_2)} - K_{\NTK}^{(\ell)}(\x_1,\x_2) } 
 \leq \abs{\inner{\BPsi^{(\ell)}(\x_1), \BPsi^{(\ell)}(\x_2)} - K_{\NNGP}^{(\ell)}(\x_1,\x_2)} \nonumber \\
&\qquad \qquad \qquad \qquad \qquad + \abs{\inner{\BGamma^{(\ell)}(\x_1), \BGamma^{(\ell)}(\x_2)} - \inner{\BLambda^{(\ell)}(\x_1) \otimes \BPhi^{(\ell-1)}(\x_1), \BLambda^{(\ell)}(\x_2) \otimes \BPhi^{(\ell-1)}(\x_2)}} \\
&\qquad \qquad \qquad \qquad \qquad + \abs{\inner{\BLambda^{(\ell)}(\x_1), \BLambda^{(\ell)}(\x_2)} \inner{\BPhi^{(\ell-1)}(\x_1), \BPhi^{(\ell-1)}(\x_2)} - \dot{K}_{\NNGP}^{(\ell)}(\x_1,\x_2) K_{\NTK}^{(\ell-1)}(\x_1,\x_2)}. \nonumber
\end{align} 
for $(\x_1, \x_2) \in \{ (\x,\x^\prime), (\x,\x), (\x^\prime, \x^\prime)\}$. By definition, the event $\mathcal{E}_{\BPsi}^{(\ell)} \left( \frac{\varepsilon^2}{32L^2} \right)$ implies that
\begin{align} \label{eq:part1}
	\abs{ \inner{\BPhi^{(\ell)}(\x_1),\BPhi^{(\ell)}(\x_1)} - \KNTK^{(\ell)}(\x_1,\x_2) }  \leq  \frac{\varepsilon^2}{32L^2}
\end{align}
and the event $\mathcal{E}_{\BGamma}^{(\ell)} \left(  \frac{\varepsilon}{8L} \left( \ell + \Delta_\ell \right)  \right)$ implies that
\begin{align} \label{eq:part2}
	\abs{\inner{\BGamma^{(\ell)}(\x_1), \BGamma^{(\ell)}(\x_2)} - \inner{\BLambda^{(\ell)}(\x_1) \otimes \BPhi^{(\ell-1)}(\x_1), \BLambda^{(\ell)}(\x_2) \otimes \BPhi^{(\ell-1)}(\x_2)}} \leq  \frac{\varepsilon}{8L} \left( \ell + \Delta_\ell \right).
\end{align}
For the third part in \cref{eq:three-part}, we observe that 
\begin{align}
	&\abs{\inner{\BLambda^{(\ell)}(\x_1), \BLambda^{(\ell)}(\x_2)} \inner{\BPhi^{(\ell-1)}(\x_1), \BPhi^{(\ell-1)}(\x_2)} - \dot{K}_{\NNGP}^{(\ell)}(\x_1,\x_2) K_{\NTK}^{(\ell-1)}(\x_1,\x_2)} \nonumber\\
	&\leq 
	\abs{\inner{\BPhi^{(\ell-1)}(\x_1), \BPhi^{(\ell-1)}(\x_2)}} 
	\abs{\inner{\BLambda^{(\ell)}(\x_1), \BLambda^{(\ell)}(\x_2)} - \dot{K}_{\NNGP}^{(\ell)}(\x_1,\x_2)} \nonumber \\
	&\quad \qquad + \abs{\dot{K}_{\NNGP}^{(\ell)}(\x_1,\x_2)} \abs{\inner{\BPhi^{(\ell-1)}(\x_1), \BPhi^{(\ell-1)}(\x_2)} - K_{\NTK}^{(\ell-1)}(\x_1,\x_2)} \nonumber \\
	&\leq \norm{\BPhi^{(\ell-1)}(\x_1)}_2 \norm{\BPhi^{(\ell-1)}(\x_2)}_2 \frac{3 \varepsilon}{8L} + \Delta_{\ell-1} \nonumber \\
	&\leq \left(\ell +  \Delta_{\ell-1} \right) \frac{3 \varepsilon}{8L} + \Delta_{\ell-1} \label{eq:part3}
\end{align}
where the second inequality comes from that the event $\mathcal{E}_{\BLambda}^{(\ell)}\left( \frac{3\varepsilon}{8L} \right)$ and $\abs{\dot{K}_{\NNGP}^{(\ell)}(\x_1,\x_2)} \leq 1$ and the last one follows from \cref{eq:norm-psi}. 
Putting \cref{eq:part1}, \cref{eq:part2} and \cref{eq:part3} into \cref{eq:three-part}, we have 
\begin{align}
	\Delta_{\ell} \leq \frac{\varepsilon^2}{32L^2} + \frac{\varepsilon}{8L} \left( \ell + \Delta_\ell \right) + \left(\ell +  \Delta_{\ell-1} \right) \frac{3 \varepsilon}{8L} + \Delta_{\ell-1}
	= 
	\left(1 + \frac{\varepsilon}{2L}\right) \Delta_{\ell-1} + \frac{\varepsilon}{2L} \ell + \frac{\varepsilon^2}{32L^2}.
\end{align}
This completes the proof of \cref{claim4}.
\qed

%% \\\\\\\\\\\\\\\\\\\\\\\\\\\\\\\\\\\\\\\\\\\\\\\\\\\\\\\\\\\\\\\\\\\\\\\\\\\\\\\\\\\\\\\\\\\\\\\\\\\\\\

\subsection{Proof of \cref{thm:a0_spectral}} \label{sec:proof-a0-spectral}
% and \cref{thm:a1_spectral}: Spectral Approximation of Arc-cosine Kernels} \label{sec:a0_a1_spectral}

The proofs here rely on Theorem 3.3 in \cite{lee2020generalized} which states spectral approximation bounds of random features for general kernels equipped with the leverage score sampling. This result is a generalization of~\cite{avron2017random} working on the Random Fourier Features.

\begin{theorem}[Theorem 3.3 in \citep{lee2020generalized}] \label{thm:generalized_leverage_score}
Suppose $\K\in \R^{n \times n}$ is a kernel matrix with statistical dimension $s_{\lambda}$ for some $\lambda\in (0, \norm{\K}_2)$. Let $\BPhi(\w) \in \R^n$ be a feature map with a random vector $\w \sim p(\w)$ satisfying that
$\K = \E_{\w}\left[ \BPhi(\w) \BPhi(\w)^\top \right]$.
%Define the ridge leverage function as
%\begin{align*}
%\tau_{\lambda}(\w) := p(\w) \cdot \phi(\w)^\top (K + \lambda I)^{-1} \phi(\w).
%\end{align*}
Define $\tau_{\lambda}(\w) := p(\w) \cdot \BPhi(\w)^\top (\K + \lambda \I)^{-1} \BPhi(\w)$. Let $\widetilde{\tau}(\w)$ be any measurable function such that $\widetilde{\tau}(\w) \geq \tau_{\lambda}(\w)$ for all $\w$. Assume that $s_{\widetilde{\tau}} := \int \widetilde{\tau}(\w) d \w$ is finite.
Consider random vectors $\w_1, \dots, \w_m$ sampled from $q(\w):= \widetilde{\tau}(\w) / s_{\widetilde{\tau}}$ and define that
\begin{align}
\overline{\boldsymbol \Phi} := \frac1{\sqrt{m}} \left[ \sqrt{\frac{p(\w_1)}{q(\w_1)}} \BPhi(\w_1), \ \dots \ , \sqrt{\frac{p(\w_m)}{q(\w_m)}} \BPhi(\w_m) \right].
\end{align}
If $m \geq \frac{8}{3} \varepsilon^{-2} s_{\widetilde{\tau}} \log \left( 16 s_{\lambda} / \delta\right)$ then  
\begin{align}
\left(1 - \varepsilon \right) \left( \K + \lambda \I\right)
\preceq
\overline{\boldsymbol \Phi} \overline{\boldsymbol \Phi}^\top + \lambda \I
\preceq
\left(1 + \varepsilon \right) \left( \K + \lambda \I\right)
\end{align}
holds with probability at least $1-\delta$.
\end{theorem}

We now ready to provide proofs of \cref{thm:a0_spectral}.

\azerospectral*

{\it Proof of \cref{thm:a0_spectral}.}
% $
% \BPhi_0 := \frac{1}{\sqrt{m}}
% \begin{bmatrix}
% \phi_0(\w_1), \ \dots \ , \phi_0(\w_m)
% \end{bmatrix}
% $ where 
Let $\BPhi_0(\w) := \sqrt{2} \ \mathrm{Step}(\X \w) \in \R^n$ for $\w \in \R^d$ and $p(\w)$ be the 
probability density function of the standard normal distribution. As studied in~\cite{cho2009kernel}, $\BPhi_0(\w)$ is a random feature of $\A_0$ such that
\begin{align}
\A_0 = \E_{\w \sim p(\w)} \left[ \BPhi_0(\w) \BPhi_0(\w)^\top\right].
\end{align}
In order to utilize \cref{thm:generalized_leverage_score}, we need an upper bound of $\tau_{\lambda}(\w)$ as below:
\begin{align}
\tau_{\lambda}(\w) &:= p(\w) \cdot \BPhi_0(\w)^\top \left( \A_0 + \lambda \I \right)^{-1} \BPhi_0(\w) \\
&\leq p(\w) \norm{(\A_0 + \lambda \I)^{-1}}_{2} \norm{\BPhi_0(\w)}_2^2 \\
&\leq p(\w) \frac{\norm{\BPhi_0(\w)}_2^2 }{\lambda} \\
&\leq p(\w) \frac{2n}{\lambda}
\end{align}
where the inequality in second line holds from the definition of matrix operator norm and the inequality in third line follows from the fact that smallest eigenvalue of $\A_0 + \lambda \I$ is equal to or greater than $\lambda$. The last inequality is from that $\norm{\mathrm{Step}(\x)}_2^2 \leq n$ for any $\x \in \R^n$. Note that $\int_{\R^d} p(\w) \frac{2n}{\lambda} d\w= \frac{2n}{\lambda}$ and since it is constant the modified random features correspond to the original ones. Putting all together into \cref{thm:generalized_leverage_score}, we can obtain the result.
This completes the proof of \cref{thm:a0_spectral}. \qed

\subsection{Proof of \cref{thm:a1_spectral}} \label{sec:proof-a1-spectral}
\aonespectral*

{\it Proof of \cref{thm:a1_spectral}.}
% Let $\BPhi_1 := \frac{1}{\sqrt{m}} \begin{bmatrix} \phi_1(\w_1), \ \dots \ , \phi_1(\w_m) \end{bmatrix}$ where 
Let $\BPhi_1(\w) := \sqrt{2}~\mathrm{ReLU}(\X \w) \in \R^n$ for $\w \in \R^d$ and $p(\w)$ be the probability density function of standard normal distribution.
\citet{cho2009kernel} also showed that $\BPhi_1(\w)$ is a random feature of $\A_1$ such that
\begin{align}
\A_1 = \E_{\w \sim p(\w)} \left[ \BPhi_1(\w) \BPhi_1(\w)^\top\right].
\end{align}
Again, we use the below upper bound as follow:
\begin{align}
\tau_{\lambda}(\w) &:= p(\w) \cdot \BPhi_1(\w)^\top (\A_1 + \lambda \I)^{-1} \BPhi_1(\w) \\
&\leq p(\w) \norm{(\A_1 + \lambda \I)^{-1}}_2 \norm{\BPhi_1(\w)}_2^2 \\
&= 2~p(\w) \frac{\norm{\mathrm{ReLU}(\X\w)}_2^2}{\lambda} \\
&\leq 2~p(\w) \frac{\norm{\X\w}_2^2}{\lambda} \\
&\leq 2~p(\w) \norm{\w}_2^2 \frac{\norm{\X}_2^2}{\lambda}
\end{align}
%\begin{align}
%\phi_1(\w)^\top (A_1 + \lambda I)^{-1} \phi_1(\w) 
%&\leq \norm{(A_1 + \lambda I)^{-1}}_2 \norm{\phi_1(\w)}_2^2 \\
%&\leq \frac{2\norm{\mathrm{ReLU}(X\w)}_2^2}{\lambda + \lambda_{\min}} \\
%&\leq \frac{2\norm{X\w}_2^2}{\lambda + \lambda_{\min}} \\
%&\leq \frac{2\norm{X}_2^2 \norm{\w}_2^2}{\lambda + \lambda_{\min}}
%\end{align}
where the inequality in fourth line holds from that $\norm{\mathrm{ReLU}(\x)}_2^2 \leq \norm{\x}_2^2$ for any vector $\x$. 
Denote $\widetilde{\tau}(\w) := 2~p(\w) \norm{\w}_2^2 \frac{\norm{\X}_2^2}{\lambda}$ and it holds that
\begin{align}
\int_{\R^d} \widetilde{\tau}(\w) d\w 
=
\int_{\R^d} 2~p(\w) \norm{\w}_2^2 \frac{\norm{\X}_2^2}{\lambda} d\w = 2 d \frac{ \norm{\X}_2^2}{\lambda }
\end{align}
since $\int_{\R^d}  p(\w) \norm{\w}_2^2 = \mathtt{tr}(\I_d) = d$ for $\w \sim \mathcal{N}(\boldsymbol{0}, \I_d)$.
We define the modified distribution as
\begin{align}
q(\w):= \frac{\widetilde{\tau}(\w)}{\int_{\R^d} \widetilde{\tau}(\w) d\w} = p(\w) \frac{\norm{\w}_2^2}{d} = 
\frac{1}{(2\pi)^{d/2} d} \norm{\w}_2^2 \exp\left(-\frac{1}{2}\norm{\w}_2^2\right)
\end{align}
and recall the modified random features as
\begin{align}
\BPhi_1 &= \frac1{\sqrt{m}}
\left[
\sqrt{\frac{p(\w_1)}{q(\w_1)}} \ \BPhi_1(\w_1), \ \dots \ , \sqrt{\frac{p(\w_m)}{q(\w_m)}} \ \BPhi_1(\w_m)
\right] \\
&= \sqrt{\frac{2d}{m}}
\left[
\frac{\mathrm{ReLU}(\X \w_1)}{\norm{\w_1}_2} , \ \dots \ ,  \frac{\mathrm{ReLU}(\X \w_m)}{\norm{\w_m}_2}
\right].
\end{align}
Putting all together into \cref{thm:generalized_leverage_score}, we derive the result. 
%such that $\tau_{\lambda}(\w) \leq \widetilde{\tau}(\w)$ and it holds that 
%\begin{align}
%\int_{\R^d} \widetilde{\tau}(\w) d \w  
%= 
%\int_{\R^d}  p(\w) \norm{\w}_2^2 \frac{2 \norm{X}_2^2}{\lambda + \lambda_{\min}} d \w = \frac{2 d \norm{X}_2^2}{\lambda + \lambda_{\min}}
%\end{align}
This completes the proof of \cref{thm:a1_spectral}. \qed

%%%%%%%%%%%%%%%%%%%%%%%%%%%%%%%%%%%%%%%%%%%%%%%%%%%%%%%%%%%%%%%%%%%%%%%%%%%%%%%%%%%%%%%%%%%%%%%%%%%%%%%%%%%%%%%%%%%%%%%%%%%%%
%%%%%%%%%%%%%%%%%%%%%%%%%%%%%%%%%%%%%%%%%%%%%%%%%%%%%%%%%%%%%%%%%%%%%%%%%%%%%%%%%%%%%%%%%%%%%%%%%%%%%%%%%%%%%%%%%%%%%%%%%%%%%
%%%%%%%%%%%%%%%%%%%%%%%%%%%%%%%%%%%%%%%%%%%%%%%%%%%%%%%%%%%%%%%%%%%%%%%%%%%%%%%%%%%%%%%%%%%%%%%%%%%%%%%%%%%%%%%%%%%%%%%%%%%%%

\subsection{Proof of \cref{thm:ntk_spectral}} \label{sec:proof_ntk_spectral}

Before diving into detailed algorithmic analysis, we introduce spectral approximation bounds of \CS when it applies to Hadamard product of two PSD matrices. Recall that the \CS plays a key role for reducing the feature map dimensionality and below theorem is used in the proof of \cref{thm:ntk_spectral}.

\begin{restatable}{lemma}{csspectral} \label{lmm:cs_spectral}
Given $\X \in \R^{n \times d_1}$ and $\Y \in \R^{n \times d_2}$, let $\mathcal{C}_1, \mathcal{C}_2$ be the two independent \CS transforms from $\R^{d_1}, \R^{d_2}$ to $\R^m$, respectively. Denote that
\begin{align}
\BGamma := \iFFT( \FFT(\mathcal{C}_1(\X)) \odot \FFT(\mathcal{C}_2(\Y))).
\end{align}
%and let $s_\lambda$ be the statistical dimension of $(XX^\top) \odot (YY^\top)$ for $\lambda>0$. 
Given $\varepsilon, \delta \in (0,1)$, $\lambda \geq 0$ and $m \geq \frac{11}{\varepsilon^{2} \delta} \left( \frac{\tr(\X\X^\top \odot \Y\Y^\top)}{\tr(\X\X^\top \odot \Y\Y^\top)/n + \lambda}\right)^2$, then it holds
\begin{align} \label{eq:cs_spectral}
(1-\varepsilon)
\left(\X\X^\top \odot \Y \Y^\top + \lambda \I \right)
\preceq
\BGamma \BGamma^\top + \lambda \I
\preceq 
(1+\varepsilon)
\left(\X\X^\top \odot \Y\Y^\top + \lambda \I \right)
\end{align}
with probability at least $1-\delta$.
\end{restatable}

The proof of \cref{lmm:cs_spectral} is provided in \cref{sec:proof_cs_spectral}.

\ntkspectral*

{\it Proof of \cref{thm:ntk_spectral}.}
Note that the NTK of two-layer ReLU network can be formulated as
\begin{align}
\K_{\NTK} = \A_1 + \A_0 \odot(\X\X^\top)
\end{align}
where $\A_0$ and $\A_1$ are the arc-cosine kernel matrices of order $0$ and $1$ with $\X$, respectively. Let $\BPhi_0$ and $\BPhi_1$ be the random features of $\A_0$ and $\A_1$, respectively, satisfying that $\A_0 = \E \left[\BPhi_0 \BPhi_0^\top\right]$ and $\A_1 = \E\left[\BPhi_1 \BPhi_1^\top\right]$. Based on the property of \CS, one can check that $\BPhi_0 \BPhi_0^\top \odot \X\X^\top = \E[\BGamma \BGamma^\top]$ where we recall that
\begin{align}
\BGamma := \iFFT( \FFT(\mathcal{C}_1(\BPhi_0)) \odot \FFT(\mathcal{C}_2(\X))).
\end{align}
%Then, $\Phi = [\Phi_1 \ \BGamma]$ and $\Phi \Phi^\top = \Phi_1 \Phi_1^\top + \BGamma \BGamma^\top$.
Our proof is a combination of spectral analysis of $\BPhi_0 \BPhi_0^\top, \BPhi_1 \BPhi_1^\top$ and $\BGamma\BGamma^\top$ which are stated in \cref{thm:a0_spectral}, \cref{thm:a1_spectral} and  \cref{lmm:cs_spectral}, respectively.

From \cref{thm:a1_spectral}, if $m_1 \geq \frac{16}{3} \frac{d \norm{\X}_2^2}{\lambda \varepsilon^{2} } \log\left( \frac{48 s_{\lambda}}{\delta} \right)$ then with probability at least $1 - \frac{\delta}{3}$ it holds
\begin{align} \label{eq:a1_spectral}
(1 - \varepsilon) \left(\A_1 + \frac{\lambda}{2} \I\right) 
\preceq
\BPhi_1 \BPhi_1^\top + \frac{\lambda}{2} \I 
\preceq 
(1 + \varepsilon) \left( \A_1 + \frac{\lambda}{2} \I\right).
\end{align}

From \cref{thm:a0_spectral}, if $m_0 \geq 48\frac{n}{\lambda \varepsilon^{2} } \log\left( \frac{48 s_{\lambda}}{\delta} \right)$ then with probability at least $1 - \frac{\delta}{3}$ it holds
\begin{align} \label{eq:a0_spectral}
\left(1 - \frac{\varepsilon}{3}\right) \left(\A_0 + \frac{\lambda}{2} \I\right)
\preceq
\BPhi_0 \BPhi_0^\top + \frac{\lambda}{2} \I 
\preceq 
\left(1 + \frac{\varepsilon}{3}\right) \left(\A_0 + \frac{\lambda}{2} \I\right)
\end{align}
Rearranging \cref{eq:a0_spectral}, we get 
\begin{align} \label{eq:phi0_upperbound}
\BPhi_0 \BPhi_0^\top \preceq  \left(1 + \frac{\varepsilon}{3}\right) \A_0 + \frac{\varepsilon}{6} \lambda \I.
\end{align}

To guarantee spectral approximation of $\BGamma\BGamma^\top$, we will use the result of \cref{lmm:cs_spectral}.
Before applying it, we provide an upper bound of the trace of $\BPhi_0 \BPhi_0^\top \odot \X\X^\top$. 
Consider $\BPhi_0 = \sqrt{\frac{2}{m_0}} \left[ \mathrm{Step}(\Z \w_1), \dots, \mathrm{Step}(\Z \w_{m_0})\right]$ for some $\Z = \left[ \z_1, \dots, \z_n \right] \in \R^{n \times d}$. Then, we have
\begin{align}
\tr\left( \BPhi_0 \BPhi_0^\top \odot \X\X^\top \right)
&= \sum_{j=1}^n \left[\BPhi_0 \BPhi_0^\top \right]_{jj} \cdot \left[\X\X^\top\right]_{jj} \\
&= \sum_{j=1}^n \left(\frac{2}{m_0} \sum_{i=1}^{m_0} 
\mathrm{Step}\left(\inner{\z_j, \w_i} \right)^2  \right) \cdot \norm{\x_j}_2^2\\
&\leq \sum_{j=1}^n 1 \cdot \norm{\x_j}_2^2 = n
\end{align}
where the inequality in third line holds from that $\mathrm{Step}(x) \leq 1$ for any $x \in \R$ and the last equality follows from the assumption that $\norm{\x_j}_2 = 1$ for all $j \in [n]$.

Hence, using \cref{lmm:cs_spectral} with the fact that
$m_{\mathtt{cs}} \geq \frac{297}{\varepsilon^2 \delta} \left(\frac{n}{1+\lambda}\right)^2\geq \frac{297}{\varepsilon^2 \delta} \left( \frac{\tr(\BPhi_0 \BPhi_0^\top \odot \X\X^\top)}{\tr(\BPhi_0 \BPhi_0^\top \odot \X\X^\top)/n + \lambda}\right)^2$, we have
\begin{align} \label{eq:cs_spectral_in_proof}
\left(1 + \frac{\varepsilon}{3} \right) \left( \BPhi_0 \BPhi_0^\top \odot \X\X^\top  + \frac{\lambda}{2} \I\right)
\preceq
\BGamma \BGamma^\top + \frac{\lambda}{2} \I
\preceq
\left(1 + \frac{\varepsilon}{3} \right) \left( \BPhi_0 \BPhi_0^\top \odot \X\X^\top  + \frac{\lambda}{2} \I \right)
\end{align}
with probability at least $1 - \frac{\delta}{3}$. Combining \cref{eq:cs_spectral_in_proof} with \cref{eq:phi0_upperbound}, with probability at least $1 - \frac{2}{3} \delta$, we get that
\begin{align}
\BGamma \BGamma^\top + \frac{\lambda}{2} \I
&\preceq \left(1 + \frac{\varepsilon}{3}\right) \left( \BPhi_0 \BPhi_0^\top \odot \X\X^\top  + \frac{\lambda}{2} I\right) \\
&\preceq \left(1 + \frac{\varepsilon}{3}\right) \left( \left[\left(1 + \frac{\varepsilon}{3}\right) \A_0 + \frac{\varepsilon}{6} \lambda \I\right] \odot \X\X^\top  + \frac{\lambda}{2} \I \right) \\
&= \left(1 + \frac{\varepsilon}{3}\right)  \left(\left(1 + \frac{\varepsilon}{3}\right) \left(\A_0 \odot \X\X^\top\right) + \frac{\varepsilon}{6} \lambda (\I \odot \X\X^\top ) + \frac{\lambda}{2} \I \right) \\
&=  \left(1 + \frac{\varepsilon}{3}\right) \left(1 + \frac{\varepsilon}{3}\right) \left(\A_0 \odot \X\X^\top +  \frac{\lambda}{2} \I  \right) \\
&\preceq  \left(1 + \varepsilon \right) \left(\A_0 \odot \X\X^\top +  \frac{\lambda}{2} \I  \right) \label{eq:cs_a0_spectral}
\end{align}
where the equality in second line follows from the fact that $\A\odot \C \preceq \B \odot \C$ holds if $\A \preceq \B$ for positive semidefinite matrices $\A, \B$ and $\C$~\footnote{It is enough to show that $(\B - \A) \odot \C \succeq 0$. Since $\B-\A, \C \succeq 0$, this holds from the Schur product theorem.}
% (see \cref{lmm:hadamard} for details) 
and the fourth equality is from the assumption $\norm{\x_i}_2 = 1$ for all $i \in [n]$ which leads that $\I \odot (\X\X^\top) = \I$. The last inequality holds since $\varepsilon \in (0, 1/2)$.

Similarly, we can obtain the following lower bound:
\begin{align}
\BGamma \BGamma^\top + \frac{\lambda}{2} \I
\succeq \left(1 - \varepsilon\right) \left(\A_0 \odot \X\X^\top + \frac{\lambda}{2} \I \right) \label{eq:cs_a0_spectral2}.
\end{align}

Combining \cref{eq:a1_spectral}, \cref{eq:cs_a0_spectral} with \cref{eq:cs_a0_spectral2} gives
\begin{align} \label{eq:ntk_spectral}
(1 - \varepsilon) \left( \A_1 + \A_0 \odot \X\X^\top + \lambda \I\right)
\preceq
\BPhi \BPhi^\top + \lambda \I 
\preceq
(1 + \varepsilon) \left( \A_1 + \A_0 \odot \X\X^\top + \lambda \I\right).
\end{align}

Furthermore, by taking a union bound over all events, 
\cref{eq:ntk_spectral} holds with probability at least $1- \delta$.
This completes the proof of \cref{thm:ntk_spectral}. \qed

\section{Proof of Lemmas}

The proofs of \cref{lmm:ts-entry-diff} and \cref{lmm:cs_spectral} are obtained from {Lemma 2} in~\cite{avron2014subspace} that provides an upper bound on variance of \TS transform of order $q \geq 2$. 
% Note that \cref{lmm:ts-entry-diff} is a corollary of \cref{lmm:ts-fro-norm} for $q=2$.

\begin{lemma}[Lemma 2 in \cite{avron2014subspace}] \label{lmm:cs_var_bound}
For $q \geq 2$, consider $q$ of 3-wise independent hash functions $h_1, \dots h_q : [d] \rightarrow [m]$ and $q$ of 4-wise independent sign functions $s_1, \dots, s_q: [d] \rightarrow \{+1,-1\}$. Define the hash function $H : [d^q] \rightarrow [m]$ and the sign function $S: [d^q] \rightarrow \{-1, +1\}$ such that
\begin{align*}
H(j) &\equiv h_1(i_1) + h_2(i_2) + \dots + h_q(i_q) \pmod{m}, \\ 
S(j) &= s_1(i_1) \cdot s_2(i_2) \cdot \dots \cdot s_q (i_q).
\end{align*}
where $j \in [d^q]$ and $i_1, \dots, i_q \in [d]$ such that $j = i_q d^{q-1} + \cdots + i_{2} d + i_1$. Denote sketch matrix $\S \in \{-1,0,+1\}^{d^q \times m}$ satisfying that $\S_{j, H(j)} = S(j)$ for $j \in [d^q]$ and other entries are set to zero. For any  $\A, \B \in \R^{n \times d^q}$, it holds
\begin{align}
\E_{\S} \left[
\norm{
\A\S\S^\top \B^\top - \A\B^\top
}_F^2
\right] \leq \frac{(2 + 3^q)}{m} \norm{\A}_F^2 \norm{\B}_F^2.
\end{align}
\end{lemma}

\subsection{Proof of \cref{lmm:ts-entry-diff}} \label{sec:proof-ts-entry-diff}
\cserrorbound*

{\it Proof of \cref{lmm:ts-entry-diff}.} 
By the Markov's inequality, we have
\begin{align}
\Pr \left( \abs{\inner{\BGamma, \BGamma^\prime} - \inner{ \x \otimes \y,  \x^\prime \otimes \y^\prime}} \geq \varepsilon \right) 
\leq \frac1{\varepsilon^2} \E \left[ \abs{\inner{ \x \otimes \y,  \x^\prime \otimes \y^\prime} - \inner{\BGamma, \BGamma^\prime}}^2 \right] 
\leq \frac{11 \norm{\x}_2^2 \norm{\x^\prime}_2^2 \norm{\y}_2^2 \norm{\y^\prime}_2^2}{\varepsilon^2 m_{\mathtt{cs}}} 
\end{align}
where the last inequality follows from \cref{lmm:cs_var_bound} with $q=2$.
% and the last equality comes from that $\max_i \norm{\Lambda_i}_2^2 = \norm{\BLambda \BLambda^\top}_\infty$. 
This completes the proof of \cref{lmm:ts-entry-diff}.
\qed

\subsection{Proof of \cref{lmm:cs_spectral}} \label{sec:proof_cs_spectral}

\csspectral*

{\it Proof of \cref{lmm:cs_spectral}.}
%We first note that $\Gamma$ is the output of a Count Sketch transform since the convolution of Count Sketch transform is also Count Sketch. Formally, 
%There exists a Count Sketch $\mathcal{C}: \R^{d_1 d_2} \rightarrow \R^m$ such that 
%
%
%then $\Gamma$ can be considered as the output of Count Sketch transform on $Z$. We can write $\Gamma = Z S$ for $S \in \R^{d_1 d_2 \times m}$ such that
Let $s_1:[d_1] \rightarrow \{-1,+1\}$ and $h_1:[d_1] \rightarrow [m]$ be the random sign and hash function of $\mathcal{C}_1$, respectively. Similarly, denote $s_2$ and $h_2$ by that of $\mathcal{C}_2$, respectively. Then, $\BGamma$ is the output of Count Sketch $\mathcal{C}:\R^{d_1d_2} \rightarrow \R^m$ applying to $\X \otimes \Y$ whose sign and hash functions are defined as
\begin{align}
s(i,j) &= s_1(i) \cdot s_2(j) \\
h(i,j) &\equiv h_1(i) + h_2(j)  \pmod{m}
\end{align}
for $i \in [d_1], j \in [d_2]$. Here, index $(i,j)$ can be considered as some $k \in [d_1 d_2]$ by transforming $i = \lfloor k / d_2 \rfloor$ and $j \equiv k \pmod{d_2}$. 

Let $\S \in \{-1,0,+1\}^{d_1 d_2 \times m}$ be the sketch matrix of $\mathcal{C}$ and we write $\Z:=\X \otimes \Y$ for notational simplicity. As shown in \cite{avron2014subspace}, it is easy to check that $\BGamma = \Z \S$ and we have
\begin{align}
\E[\BGamma \BGamma^\top ] 
&= \E[(\X \otimes \Y) \S \S^\top (\X \otimes \Y)^\top] \\
&= (\X \otimes \Y) \E[\S \S^\top] (\X \otimes \Y)^\top \\
&= (\X \otimes \Y) (\X \otimes \Y)^\top \\
&= \X \X^\top \odot \Y \Y^\top = \Z\Z^\top.
\end{align}
Rearranging \cref{eq:cs_spectral}, we have
\begin{align} \label{eq:cs_spectral2}
-\varepsilon \left( \Z\Z^\top + \lambda \I \right) \preceq
\Z\S\S^\top \Z^\top - \Z\Z^\top \preceq \varepsilon \left( \Z\Z^\top + \lambda \I \right).
\end{align}
By multiplying $(\Z\Z^\top + \lambda \I)^{-{1}/{2}}$ by both left and right sides in \eqref{eq:cs_spectral2}, it is enough to show that
\begin{align*}
\norm{(\Z\Z^\top + \lambda \I)^{-{1}/{2}} \Z \S \S^\top \Z^\top (\Z\Z^\top + \lambda \I)^{-{1}/{2}}
- (\Z\Z^\top + \lambda \I)^{-{1}/{2}} \Z\Z^\top (\Z\Z^\top + \lambda \I)^{-{1}/{2}}
}_2 \leq \varepsilon
\end{align*}
By denoting $\A := (\Z\Z^\top + \lambda \I)^{-{1}/{2}} \Z$, it is equivalent to prove that 
\begin{align} \label{eq:cs_spectral_error}
\norm{\A\S\S^\top \A^\top - \A \A^\top}_2 \leq \varepsilon.
\end{align}
By Markov's inequality, we have
\begin{align*}
\Pr\left[\norm{\A\S\S^\top \A^\top - \A \A^\top}_2 \geq \varepsilon\right] 
&\leq 
\Pr\left[\norm{\A\S\S^\top \A^\top - \A \A^\top}_F \geq \varepsilon\right]  \\
&\leq
\varepsilon^{-2} \E\left[\norm{\A\S \S^\top \A^\top - \A \A^\top}_F^2\right] \\
&\leq 
\frac{11}{\varepsilon^{2} m} \norm{\A}_F^4  \\
%=  \frac{11 s_{\lambda}^2}{\varepsilon^{2} m}
&= 
\frac{11}{\varepsilon^{2} m} \left[\tr(\A^\top \A)\right]^2 \\
&= 
\frac{11}{\varepsilon^{2} m} \left[\tr\left((\Z\Z^\top + \I)^{-1} \Z\Z^\top\right)\right]^2 \\
&\leq 
\frac{11}{\varepsilon^{2} m} \left( \frac{\tr(\X\X^\top \odot \Y\Y^\top)}{\tr(\X\X^\top \odot \Y\Y^\top)/n + \lambda}\right)^2
\end{align*}
where the inequality in third line holds from \cref{lmm:cs_var_bound} with $q=2$ and the last inequality follows from \cref{lmm:sdim_upperbound}. Taking $m \geq \frac{11}{\varepsilon^2 \delta} \left( \frac{\tr(\X\X^\top \odot \Y\Y^\top)}{\tr(\X\X^\top \odot \Y\Y^\top)/n + \lambda} \right)^2$, \cref{eq:cs_spectral_error} holds with probability at least $1-\delta$. 
This completes the proof of \cref{lmm:cs_spectral}. \qed

% \begin{lemma} \label{lmm:hadamard}
% Let $\A, \B, \C \in \R^{n \times n}$ be positive semidefinite matrices and assume that $\A \preceq \B$. Then, $\A \odot \C \preceq \B \odot \C$ where $\odot$ is the Hadamard product of matrices.
% \end{lemma}

% {\it Proof of \cref{lmm:hadamard}.}
% It is enough to show that $(\B - \A) \odot \C \succeq 0$. Since $\B-\A$ and $\C$ are both positive semidefinite matrices, this holds from the Schur product theorem. This completes the proof of \cref{lmm:hadamard}.
% \qed

\begin{lemma} \label{lmm:sdim_upperbound}
Let $\A, \B$ be positive semidefinite matrices and let $s_{\lambda}$ be the statistical dimension of $\A \odot \B$ for any $\lambda>0$, i.e., $s_\lambda := \tr\left({(\A \odot \B)(\A \odot \B + \lambda \I)^{-1}}\right)$. Then, it holds that $s_{\lambda} \leq \frac{\tr (\A \odot \B)}{\tr(\A \odot \B) / n + \lambda}$.
\end{lemma}

{\it Proof of \cref{lmm:sdim_upperbound}.}
By the Schur product theorem, $\A \odot \B$ is a positive semidefinite matrix. Let $\lambda_1 \geq \lambda_2 \geq \dots \geq \lambda_n \geq 0$ be the eigenvalues of $\A \odot \B$. By the definition of statistical dimension, it holds that
\begin{align}
s_\lambda &= \tr \left( (\A\odot \B)\left( \A\odot \B + \lambda \I\right)^{-1}\right)= \sum_{i=1}^{n} \frac{\lambda_i}{\lambda_i + \lambda} \\
&\leq n \frac{ \left(\sum_{i} \lambda_i\right) / n}{\left(\sum_{i} \lambda_i\right) / n + \lambda}
= \frac{\sum_i \lambda_i}{(\sum_i \lambda_i)/n + \lambda} 
= \frac{\tr(\A\odot \B)}{\tr(\A\odot \B)/n + \lambda}
\end{align}
where the inequality holds from the Jensen's inequality. This completes the proof of \cref{lmm:sdim_upperbound}.
\qed

\section{Mathematical Details on Deep Active Learning}
\label{sec:math-active}

In this section, we give additional details on how NTK random features can be used to accelerate the greedy selection algorithm in~\citep{shoham2020experimental}. In actuality, the improvement is not specific to NTK features, but works for every kernel for which have a low-rank factorization. It applies to NTK random features by virtue of the low-rank factorization the approximate kernel induces. 

For a matrix $\A$ and index sets ${\cal S}$ and ${\cal T}$, let $\A_{{\cal S}, {\cal T}}$ denote the matrix obtained by restricting to the rows whose
index is in ${\cal S}$ and the columns whose index is in ${\cal T}$.  Using $:$ as the index set denotes the entire relevant index set. Consider kernel ridge regression, and assume the kernel matrix is $\K$. One important variant of the criteria developed in \citep{shoham2020experimental} is the minimization of 
$$
J({\cal S}) = \tr(-\K_{:, {\cal S}} (\K_{{\cal S}, {\cal S}}+\lambda \I_{|{\cal S}|})^{-1} \K^\top_{:, {\cal S}})
$$
In order to perform greedy minimization of $J({\cal S})$ we need to be able to evaluate $J({\cal S})$ quickly for a given ${\cal S}$.

Assume that $\K = \BPhi \BPhi^\top$ where $\BPhi$ has $m$ columns. We now show how after preprocessing $J({\cal S})$ can be computed in $\bigo(|{\cal S}|^2 m)$ time as long as $|{\cal S}| \leq m$. First, notice that
\begin{eqnarray*}
J({\cal S}) & = & \tr(-\K_{:, {\cal S}} (\K_{{\cal S}, {\cal S}}+\lambda \I_{|{\cal S}|})^{-1} \K^\top_{:, {\cal S}})\\
            & = & \tr(-(\K_{{\cal S}, {\cal S}}+\lambda \I_{|{\cal S}|})^{-1} \K^\top_{:, {\cal S}}\K_{:, {\cal S}}) \\
            & = & \tr(-\underbrace{(\K_{{\cal S}, {\cal S}}+\lambda \I_{|{\cal S}|})^{-1} \BPhi_{{\cal S}, :}}_{\A({\cal S})}\cdot \underbrace{\BPhi^\top \BPhi \BPhi^\top_{{\cal S}, :}}_{\B({\cal S})})
\end{eqnarray*}
Now, notice that $\A({\cal S})$ can be computed in $\bigo(|{\cal S}|^2 m)$ time if we assume that $|{\cal S}| \leq m$.  As for $\B({\cal S})$, this matrix consists exactly of the columns in ${\cal S}$ of $\BPhi^\top \BPhi \BPhi^\top$. To take advantage of that we precompute $\BPhi^\top \BPhi \BPhi^\top$ in $\bigo(nm^2)$. Finally, note that we are only interested in the trace of $\A({\cal S}) \cdot \B({\cal S})$. There is no need to compute the entire product; we can compute only the diagonal elements. We see that after the $\bigo(nm^2)$ preprocessing, we can compute $J({\cal S})$ in $\bigo(|{\cal S}|^2 m)$ time.

In order to greedily minimize $J({\cal S})$, we start with ${\cal S}=\emptyset$, and and add at each iteration the index that will minimize $J({\cal S})$. To do so, we scan the entire index set, evaluating $J({\cal T})$ for each candidate ${\cal T}$ that consists of the current ${\cal S}$ and the addition index. Since there are $n$ data-points, the scan takes $\bigo(n |{\cal S}|^2 m)$ which is the cost per iteration. Including preprocessing time, the cost of finding a design of size $k$ is $\bigo(nm(k^3 + m))$.

\section{Additional Experiments on Image Classification}

%%%%%%%%%%%%%%%%%%%%%%%%%%%%%%%%%%%%%%%%%%%%%%%%%%%%%%%%%%%%%%%%%%%%%%%%%%%%%%%%%%%%%%%%%%%%%%%%%%%%%%%%%%%%%%
\subsection{Classification on Image Datasets} \label{sec:fewshot}

% \paragraph{Training details and datasets} 
% We first perform classification tasks using $8$ fine-grained image datasets. 

We additionally conduct experiments on image classification using $8$ fine-grained datasets: $\mathtt{CIFAR}$-10/100~\cite{krizhevsky2009learning}, \VOC~\cite{everingham2010pascal}, \CALTECH~\cite{fei2004learning}, \CUB~\cite{WelinderEtal2010}, \DOG~\cite{khosla2011novel}, \FLOWER~\cite{nilsback2008automated}, \FOOD{}~\cite{bossard2014food}.
% where their statistics are listed in \cref{tab:dataset_spec}.
% \addtolength{\tabcolsep}{4pt}
% \begin{table*}[ht]
% \centering
% \caption{Statistics of fine-grained image datasets}\label{tab:dataset_spec} 
% \vspace{-0.1in}
% \scalebox{0.8}{
% \begin{tabular}{@{}lcccccccc@{}}
% \toprule
% & \CIFARTEN & \CIFARHUN & \VOC & \CALTECH & \CUB & \DOG & \FLOWER & \FOOD \\
% \midrule
% \# of Classes & 10 & 100 & 20 & 101 & 200 & 120 & 102 & 101\\
% \# of Test Instances & 10,000 & 10,000 & 4,952 & 3,018 & 3,033 & 8,580 & 6,149 & 25,250\\
% \bottomrule
% \end{tabular}
% }
% \end{table*}
% \addtolength{\tabcolsep}{-4pt}
In particular, we follow the transfer learning mechanism~\cite{goyal2019scaling} where we extract image features from the penultimate layer of the pretrained ResNet18~\cite{he2016deep} with dimension $d=512$.  These features are then leveraged as inputs to be transformed to random features. In particular, we follow low-shot setting used in~\cite{arora2019harnessing}; we randomly choose 5 image data from each class of training set, and use the whole test set for evaluation. We repeat sampling training images $50$ times for \CIFARTEN~and \VOC, $10$ times for other datasets. This is because these datasets have relatively small classes, i.e., 10 and 20, respectively. 
% The average test accuracy with 95\% confidence interval are reported in \cref{tab:low-data-transfer-learning}.

We run \cref{alg:ntk_random_features} with and without Gibbs sampling (GS) (i.e., \cref{alg:gibbs}). The number of Gibbs iteration is set to $T=1$. The output dimension $m$ is fixed to $4{,}000$ and {\sc CountSketch} dimension is considered $m_\mathtt{cs}$ as a hyperparameter. We set $m_0 = m_1=m-m_\mathtt{cs}$ and choose $m_\mathtt{cs} \in \{ 0, \frac{m}{10}, \dots, \frac{9m}{10},m\}$ for the best validation accuracy.  We normalize the output of \cref{alg:ntk_random_features} so that the corresponding features lie in $\mathbb{S}^{d-1}$.  We also benchmark the Random Fourier Features (RFF) with the same dimension $m=4{,}000$. Once features are generated, we train a linear classifier with SGD optimizer for $500$ epochs where inputs can be the pretrained features or random features. We perform grid search for finding the best learning rate in $\{0.01, 0.02, 0.1, 0.2\}$ and momentum in $\{0, 0.9, 0.99\}$ and report the best test accuracy among 12 different setups. We also execute various classifiers including AdaBoost, random forest, $k$-nearest neighbors and support vector classifier (SVC). For methods running with SVC, we search the cost value $C$ in $\{2^{-19},\dots, 2^{-3}, 10^{-7},  \dots, 10^3\}$ and choose the best one that achieves the globally maximum accuracy. Hence, a single cost value is globally used for all iterations. The number of network depth for NTK is chosen by $L=2$. For RBF kernel (i.e., $K(\x, \y) = \exp(-\gamma \norm{\x- \y}_2^2)$), we choose the best $\gamma$ in $\{ \frac{1}{d}, \frac{1}{d \sigma}\}$ that achieves the globally maximum accuracy where $\sigma$ is the variance of training data.  For $k$-nearest neighbor classifier, we search the best $k$ in $\{1,2,\dots, \min(20, n/2)\}$ where $n$ is the number of training data. This allows the number of training instance per class to be roughly larger than $2$.  For AdaBoost and Random Forest, we search the number of ensembles in $\{50, 100, 200, 500\}$.

% We apply kernel SVC to NTK, Radial Basis Function (RBF) and polynomial kernels. 

In \cref{tab:low-data-transfer-learning}, the average test accuracy with 95\% confidence interval is reported. Observe that NTK and the corresponding random features show better performance than other competitors for all datasets. These observations match with the previous result~\cite{arora2019harnessing} that the NTK can outperform on small-scale datasets. We additionally verify that NTK random features can perform similar or even better than NTK for most datasets, especially with Gibbs sampling. Such performance gaps are also observed between the Random Fourier Features and RBF kernel. This is likely due to the fact that random features has an implicit regularization effect which can lead to better generalization ability.

% then search for the best  {\sc CountSketch} dimension $m_\mathtt{cs} \in [0, m]$. Features are trained with a linear with the optimal learning rate and momentum.
%  In particular, for SVC, the optimal cost-value $\C$ was chosen from $2^{[-19,-3]} \cup 10^{[-7,3]}$ that achieves the globally maximum accuracy across iterations
% hyper-parameter search for other benchmark algorithms 
% (We apply linear support vector classifier(SVC) to kernel methods, e.g. Radial Basis Function (RBF) and NTK. For SVC, the optimal cost-value $\C$ was chosen from $2^{[-19,-3]} \cup 10^{[-7,3]}$ with the globally maximum average accuracy across iterations. For linear classifiers, learning rate in $\{0.01, 0.02, 0.1, 0.2\}$ and momentum in $\{0, 0.9, 0.99\}$ is optimized by grid-search)

\addtolength{\tabcolsep}{-4pt}   
\begin{table*}[t]
% \vspace{-0.2in}
\caption{Results of average test accuracy for image classification 
using features from the pretrained ResNet-18. We measure the 95\% confidence interval across $50$ iterations for \CIFARTEN~and \VOC,~ and $10$ iterations for the rest. {\bf Bold entries} indicate the best mean accuracy for each dataset. 
% The mean and 95\% confidence interval of accuracy is computed from 50 iterations for \CIFARTEN { }and \VOC, and 10 iterations for the rest.
% We train each classifiers with training set of size 5 $\times$ \#classes, with the entire test set used for evaluation. The \textbf{bold} indicates the best mean accuracy.
} \label{tab:low-data-transfer-learning}
\vspace{-0.05in}
\centering
\scalebox{0.7}{
\begin{tabular}{lcccccccc}
\toprule
Method               & \CIFARTEN & \CIFARHUN & \VOC & \CALTECH & \CUB            & \DOG      & \FLOWER           & \FOOD \\
\midrule
Linear Classifier & 59.97 $\pm$ 0.61 & 37.61 $\pm$ 0.42 & 59.91 $\pm$ 0.75 & 82.41 $\pm$ 0.39 & 37.92 $\pm$ 0.46 & 66.09 $\pm$ 0.51 & 73.56 $\pm$ 0.30 & 29.04 $\pm$ 0.33\\
$k$-Nearest Neighbors            & 47.64 $\pm$ 0.77 & 26.09 $\pm$ 0.50 & 50.32 $\pm$ 0.89 & 73.37 $\pm$ 0.46 & 24.78 $\pm$ 0.48 & 54.33 $\pm$ 0.59 & 56.36 $\pm$ 0.49 & 18.59 $\pm$ 0.43 \\
AdaBoost          & 32.09 $\pm$ 0.90 &  9.56 $\pm$ 0.57 & 31.56 $\pm$ 0.99 & 26.41 $\pm$ 2.34 & 7.01 $\pm$ 0.87  & 22.72 $\pm$ 1.95 & 20.87 $\pm$ 1.81 & 6.20 $\pm$ 0.45\\
Random Forest     & 56.79 $\pm$ 0.63 & 30.09 $\pm$ 0.43 & 56.16 $\pm$ 0.75 & 75.70 $\pm$ 0.74 & 29.71 $\pm$ 0.42 & 61.31 $\pm$ 0.45 & 64.49 $\pm$ 0.37 & 22.47 $\pm$ 0.39 \\
Linear SVM        & 58.59 $\pm$ 0.63 & 35.74 $\pm$ 0.48 & 60.63 $\pm$ 0.79 & 81.39 $\pm$ 0.55 & 36.68 $\pm$ 0.33 & 66.32 $\pm$ 0.49 & 70.94 $\pm$ 0.25 & 28.44 $\pm$ 0.37 \\
RBF Kernel SVM    & 59.16 $\pm$ 0.63 & 36.42 $\pm$ 0.54 & 60.79 $\pm$ 0.76 & 82.14 $\pm$ 0.45 & 36.67 $\pm$ 0.44 & 66.49 $\pm$ 0.57 & 71.41 $\pm$ 0.14 & 29.03 $\pm$ 0.34 \\      
Random Fourier Features & 59.63 $\pm$ 0.68 & 37.74 $\pm$ 0.50 & 60.95 $\pm$ 0.73 & 82.17 $\pm$ 0.51 & 37.50 $\pm$ 0.48 & 67.38 $\pm$ 0.43 & 72.83 $\pm$ 0.20 & 30.02 $\pm$ 0.36\\
\midrule
NTK SVM           & 60.48 $\pm$ 0.60 & 37.53 $\pm$ 0.55 & 61.19 $\pm$ 0.70 &  {\bf82.83 $\pm$ 0.41} & 37.95 $\pm$ 0.40 & 67.72 $\pm$ 0.49 & 72.32 $\pm$ 0.23 & 29.63 $\pm$ 0.37 \\
\rowcolor{Gray}{NTK Random Features} & 60.63 $\pm$ 0.62 & {\bf38.53 $\pm$ 0.50} & 61.44 $\pm$ 0.67 & 82.65 $\pm$ 0.53 & 38.11 $\pm$ 0.56 &  68.06 $\pm$ 0.54 & 73.62 $\pm$ 0.27 &  30.49 $\pm$ 0.34\\
\rowcolor{Gray}{NTK Random Features with GS} & {\bf60.66 $\pm$ 0.60} & 38.49 $\pm$ 0.51 & {\bf61.48 $\pm$ 0.68} & 82.67 $\pm$ 0.51 & {\bf 38.22 $\pm$ 0.51} &  {\bf68.07 $\pm$ 0.55} & {\bf 73.67 $\pm$ 0.28} & {\bf30.50 $\pm$ 0.32} \\
\bottomrule
\end{tabular}
}
\vspace{-0.1in}
\end{table*}
\addtolength{\tabcolsep}{4pt}

%%%%%%%%%%%%%%%%%%%%%%%%%%%%%%%%%%%%%%%%%%%%%%%%%%%%%%%%%%%%%%%%%%%%%%%%%%%%%%%%%%%%%%%%%%%%%%%%%%%%%%%%%%%%%%

\iffalse

We additionally conduct the same experiment as described in \cref{sec:fewshot} but features are now extracted from the penultimate layer of ResNet-50 instead of ResNet-18. Table \ref{tab:resnet50} summarized the results with ResNet-50. Note that the dimension of the pre-trained features is $d=1{,}024$. We observe that the NTK random features still shows the most promising results under $6$ datasets. Unlike the results with ResNet-18, linear classifier trained by the pre-trained features shows better performances under \CUB{} and \FLOWER{} than the NTK methods. This is because features extracted from ResNet-50 itself produces a good enough representation.

\input{tables/small-scale-image-resnet50}

\fi

% {\bf Performance versus feature dimension.}
We also investigate the effect of feature dimension $m$ to image classification performance. In \cref{fig:feat_dim_versus_accuracy}, we plot the test accuracy of the proposed NTK random features with Gibbs sampling (blue, triangle), Random Fourier Features (red, cross) and features from the pretrained ResNet-18 (green, circle) when feature dimension $m$ changes from $500$ to $8{,}000$. Note that the pretrained features has a fixed dimension $d=512$. The hyperparameters are chosen by the same approach as described above.
We observe that a larger $m$ can lead to higher test accuracy for both ours and RFF. It suffices to set $m=4{,}000$ for the NTK features to achieve higher accuracy than the pretrained features. 
We also verify that the proposed NTK random features shows better performance than RFF for the same dimension $m$. This justifies that ours is more effective for fine-grained image classifications.

\begin{figure}[t]
% \vspace{-0.05in}
\begin{center}
\begin{subfigure}{0.32\textwidth}
\centering
\includegraphics[width=\textwidth]{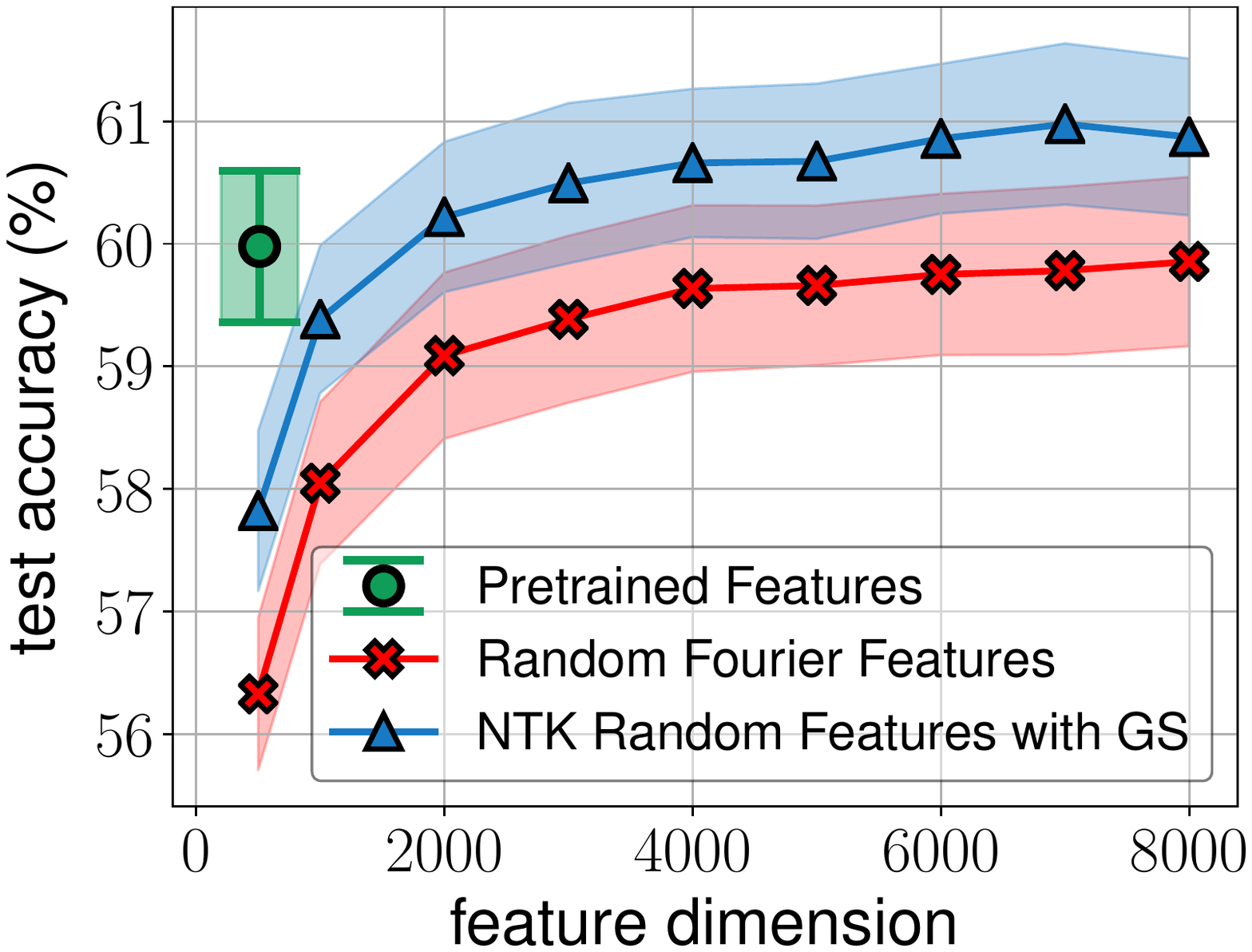}
\caption{\CIFARTEN}
\end{subfigure}
\begin{subfigure}{0.32\textwidth}
\centering
\includegraphics[width=\textwidth]{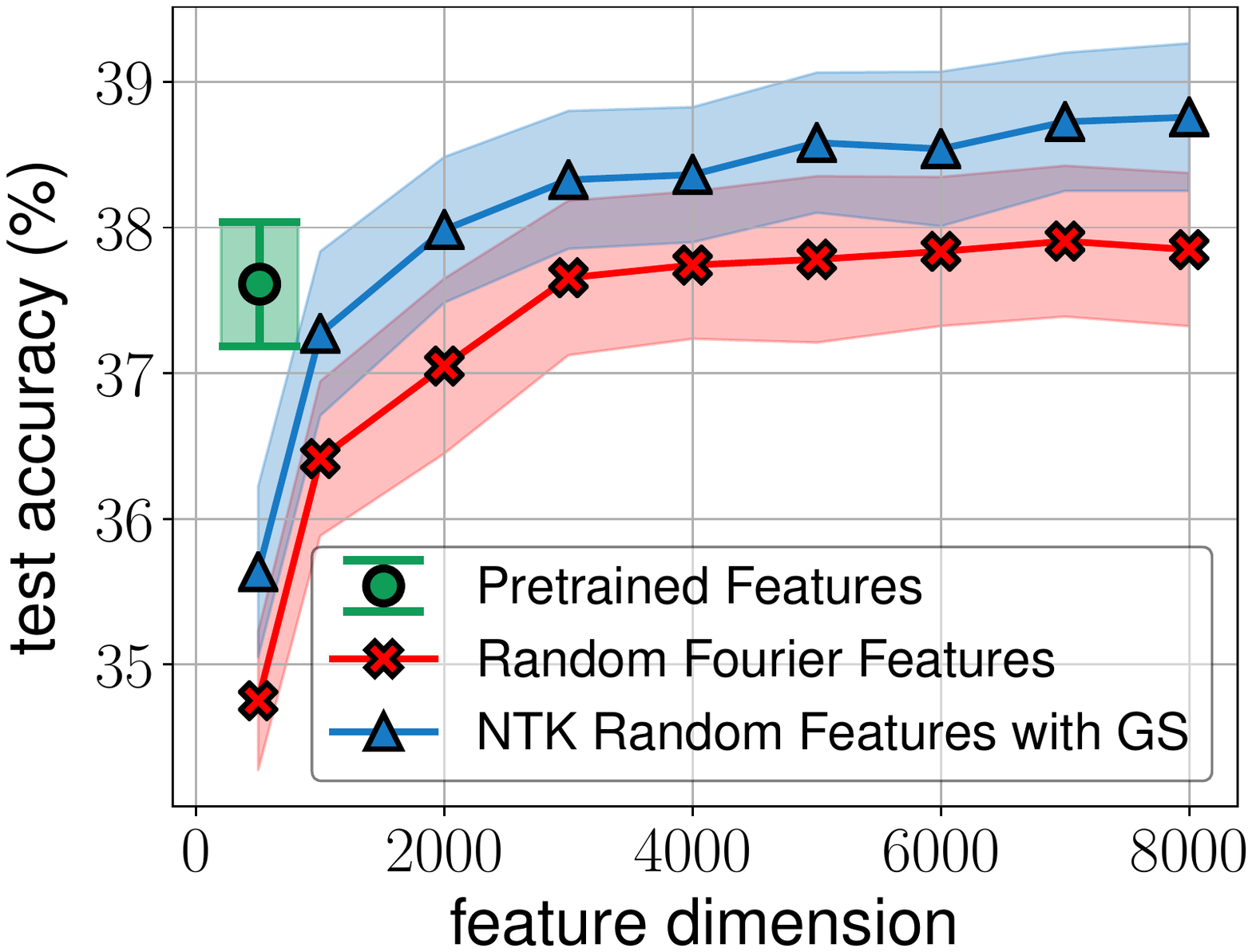}
\caption{\CIFARHUN}
\end{subfigure}
\begin{subfigure}{0.32\textwidth}
\centering
\includegraphics[width=\textwidth]{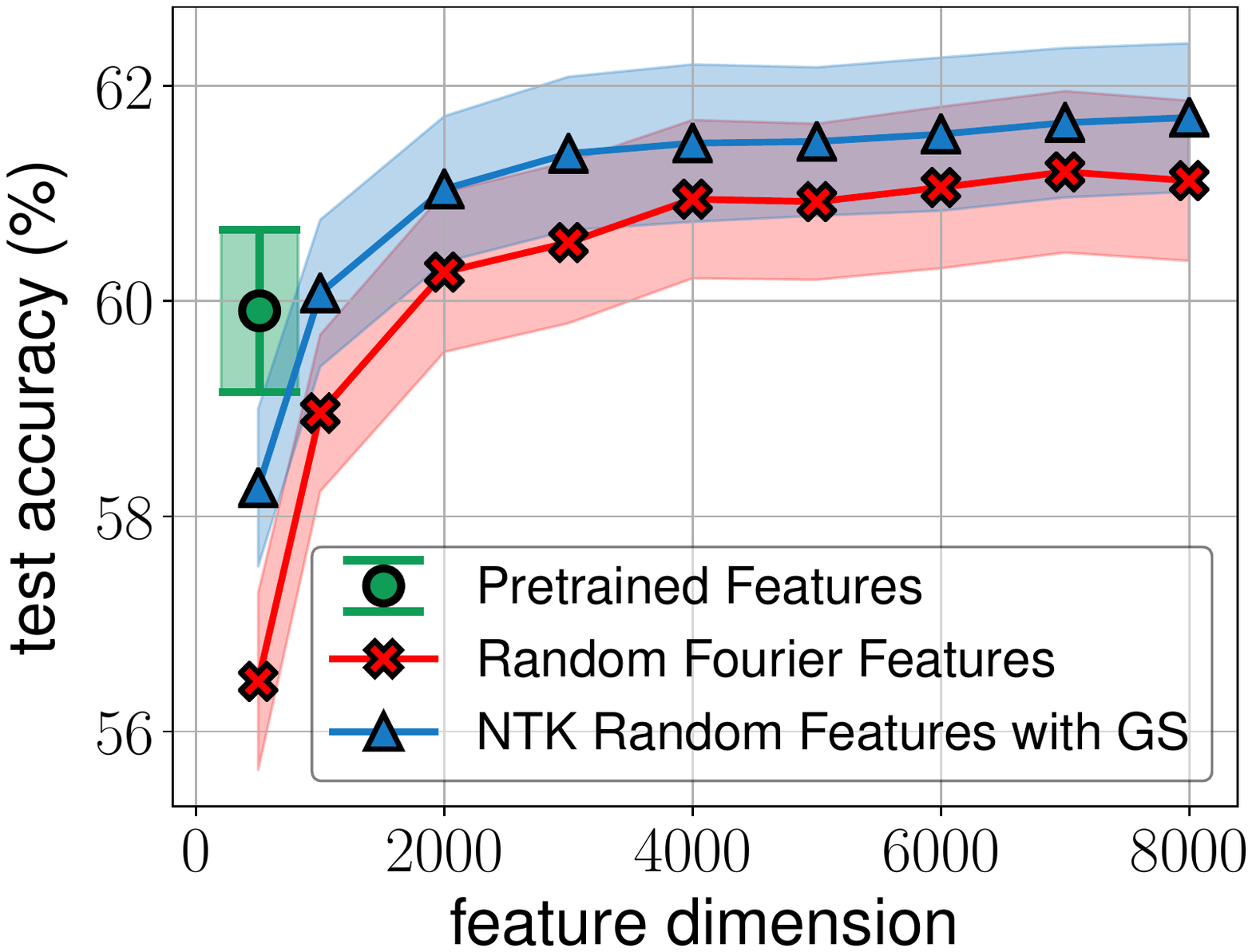}
\caption{\VOC}
\end{subfigure}
\begin{subfigure}{0.32\textwidth}
\centering
\includegraphics[width=\textwidth]{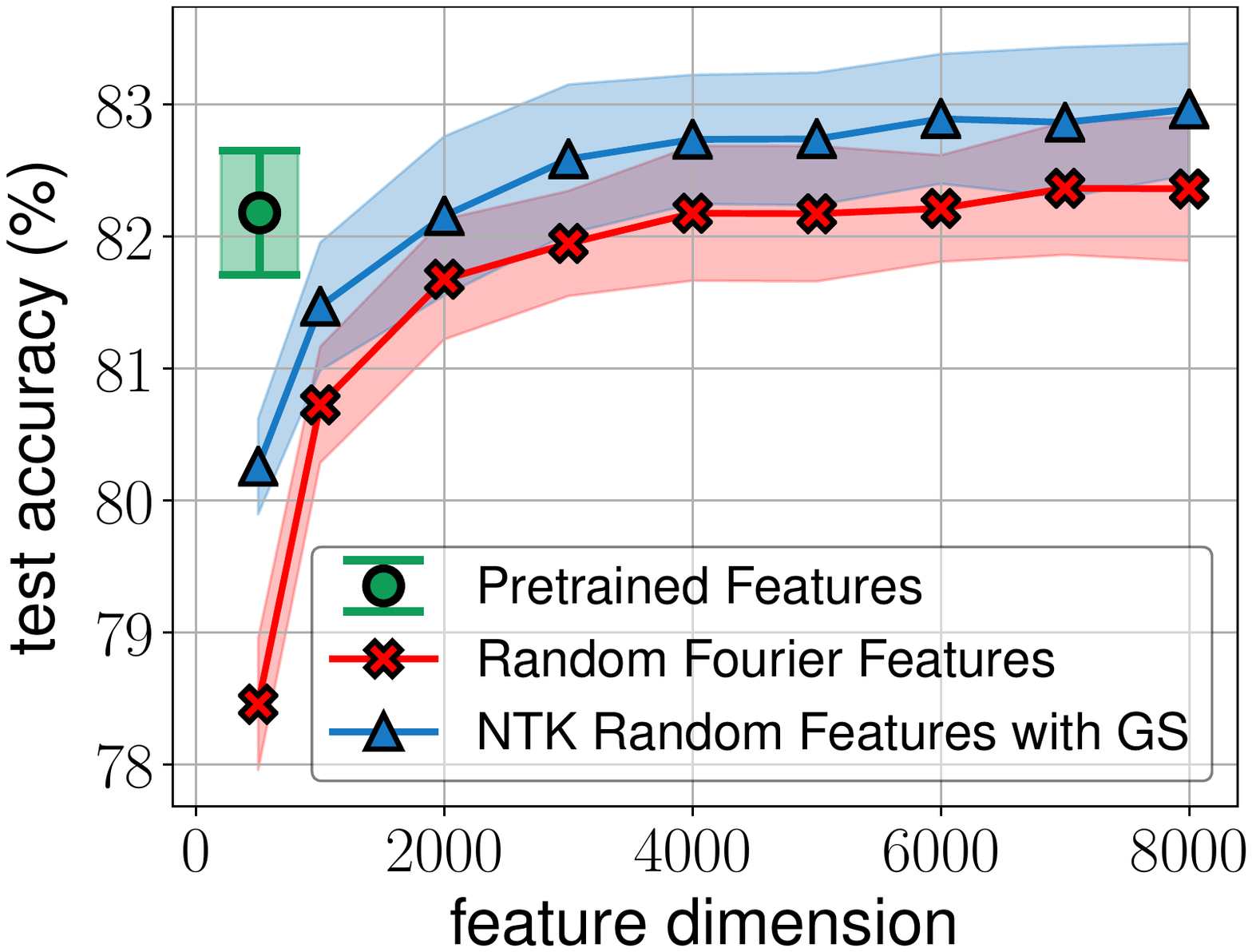}
\caption{\CALTECH}
\end{subfigure}
\begin{subfigure}{0.32\textwidth}
\centering
\includegraphics[width=\textwidth]{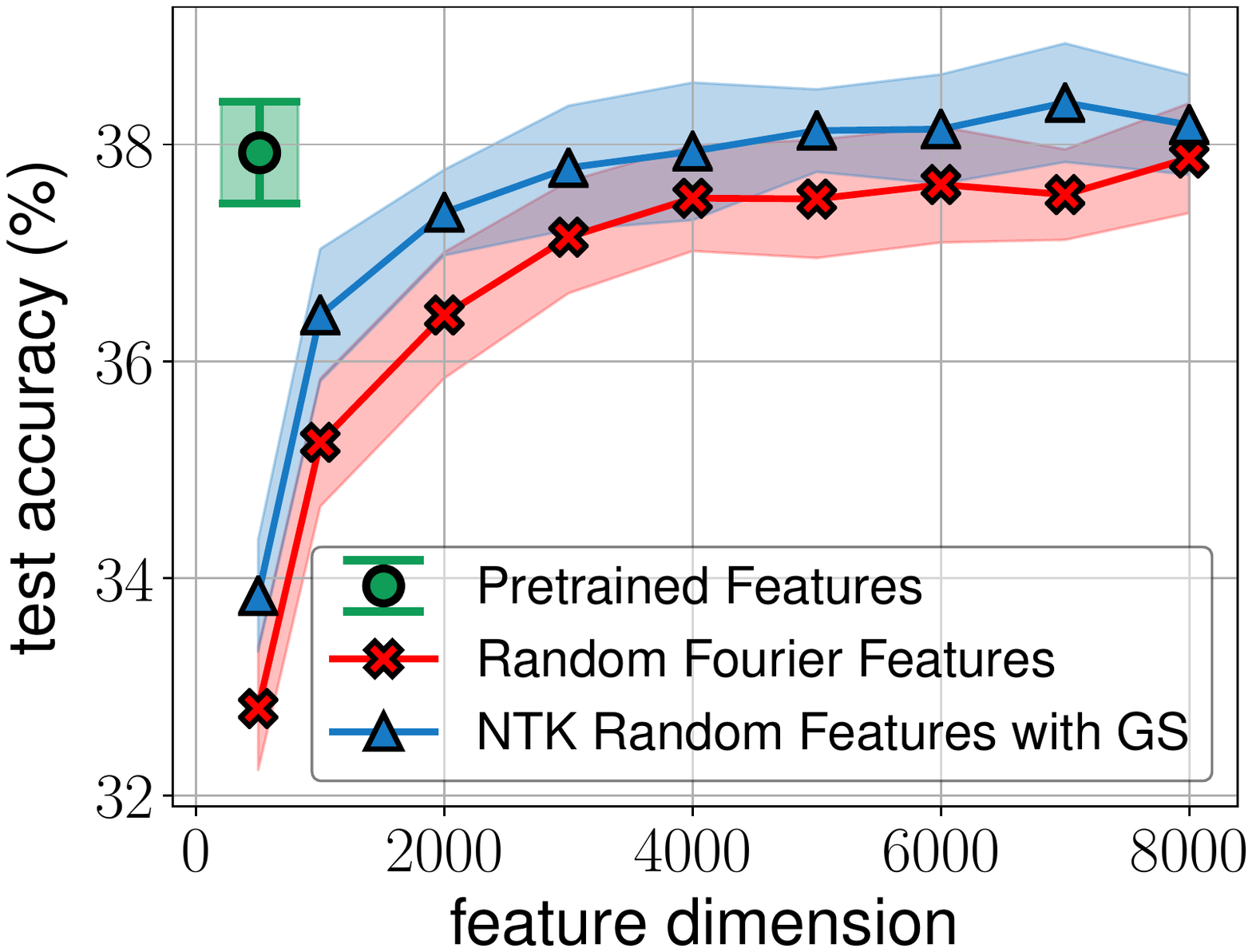}
\caption{\CUB}
\end{subfigure}
\begin{subfigure}{0.32\textwidth}
\centering
\includegraphics[width=\textwidth]{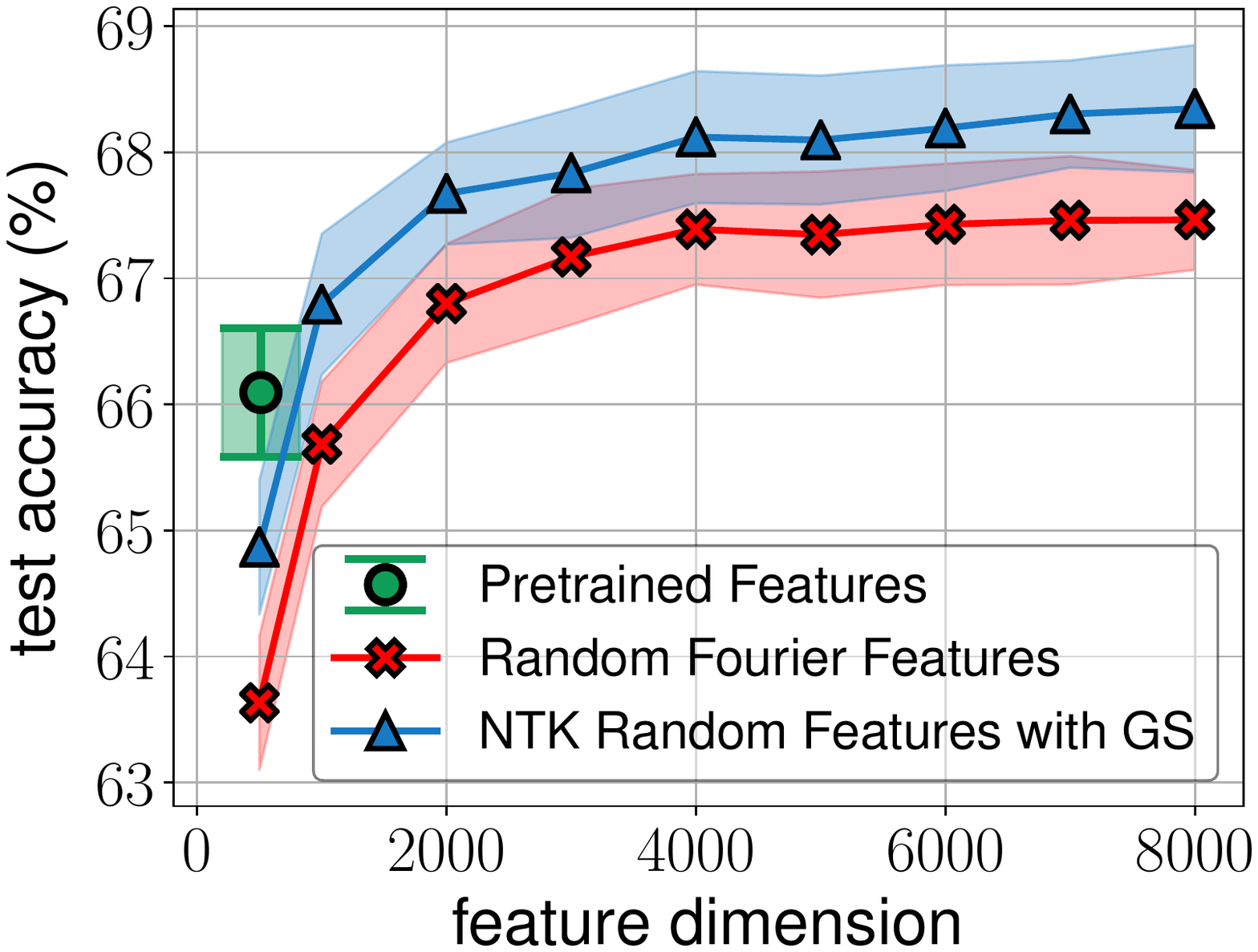}
\caption{\DOG}
\end{subfigure}
\begin{subfigure}{0.32\textwidth}
\centering
\includegraphics[width=\textwidth]{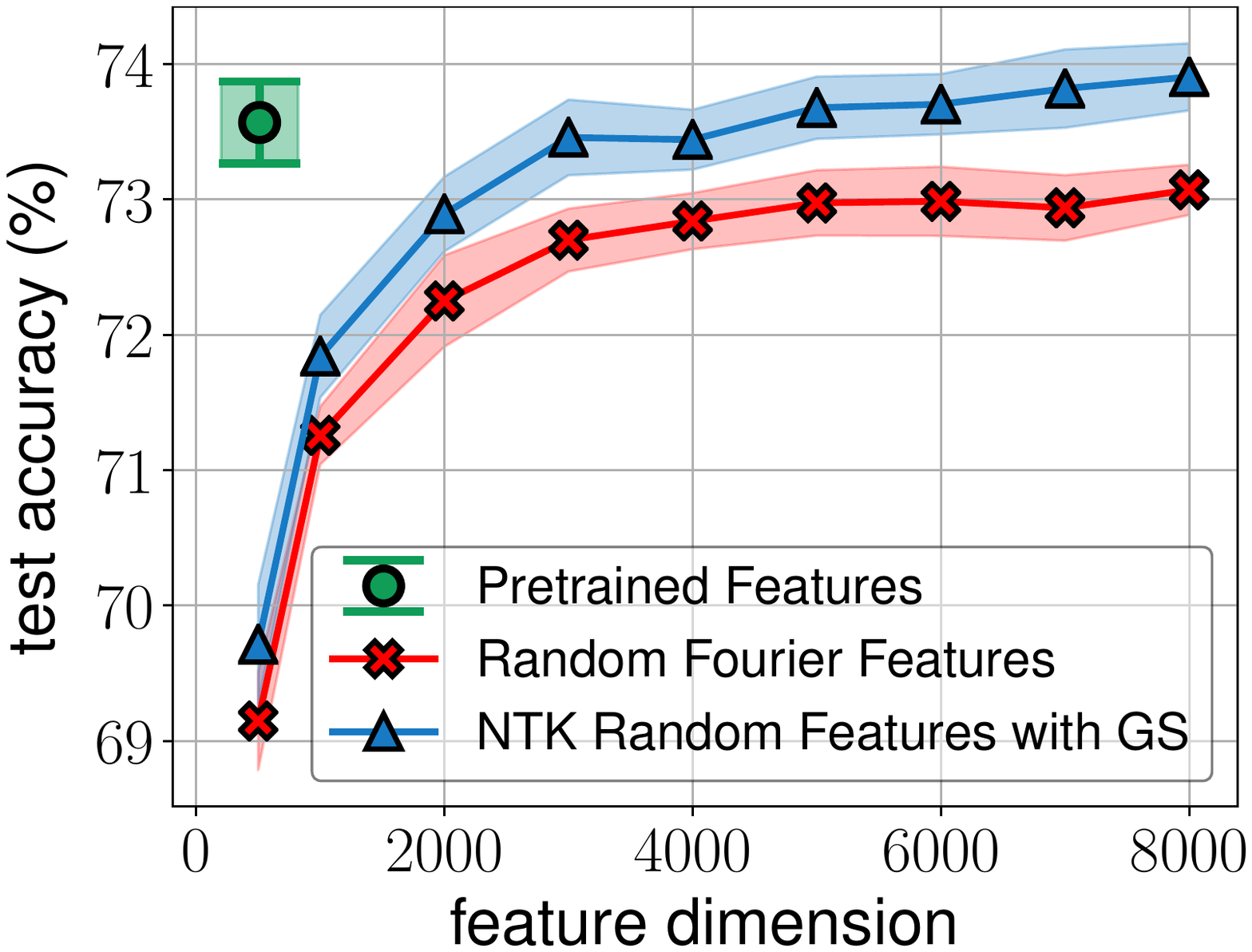}
\caption{\FLOWER}
\end{subfigure}
\begin{subfigure}{0.32\textwidth}
\centering
\includegraphics[width=\textwidth]{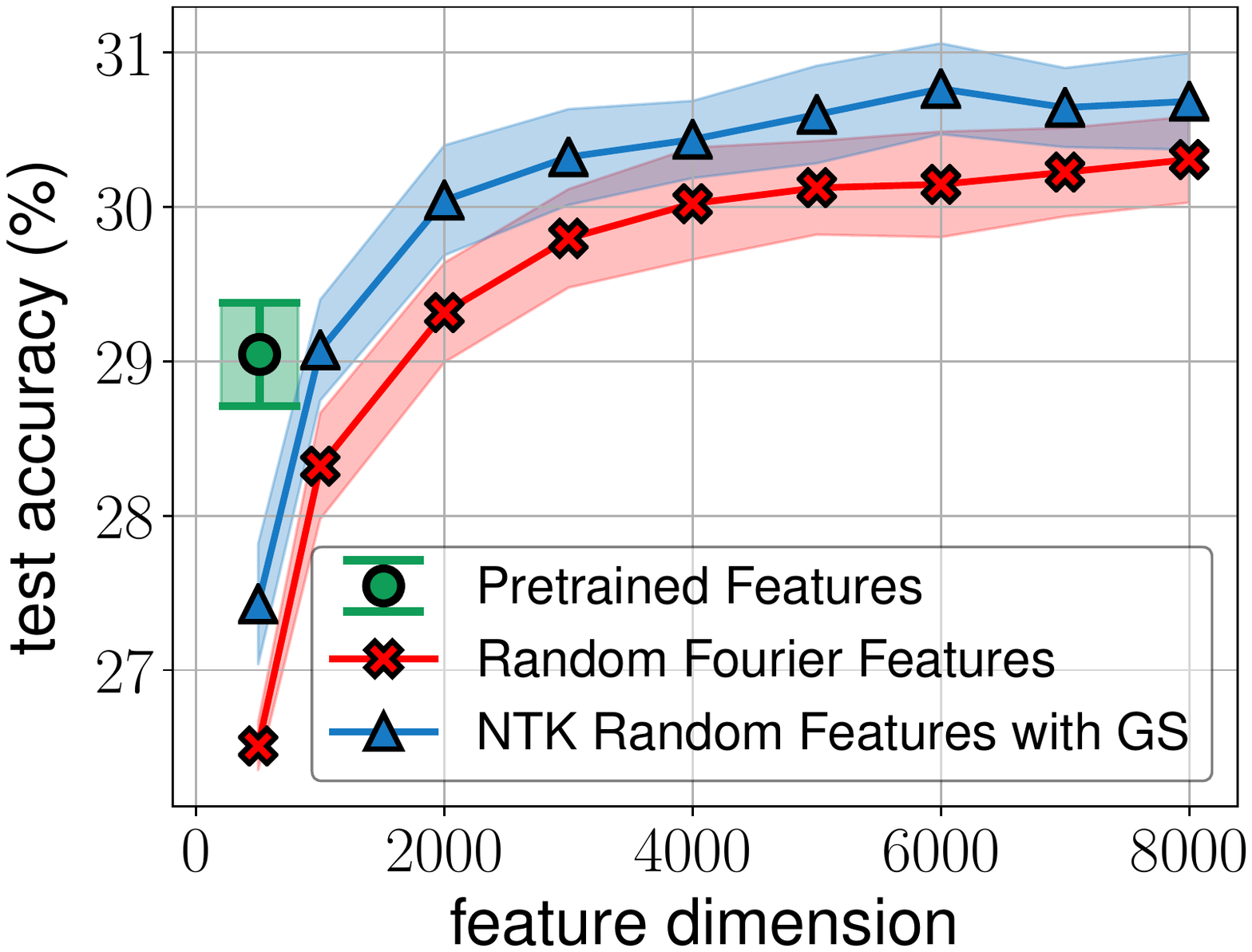}
\caption{\FOOD}
\end{subfigure}
\vspace{-0.05in}
\caption{Test accuracy of the proposed NTK random features with Gibbs sampling and Random Fourier Features when the number of feature dimension $m$ changes from $500$ to $8{,}000$.} \label{fig:feat_dim_versus_accuracy}
\end{center}
\end{figure}

\end{document}